\documentclass[twoside,11pt]{article}

\usepackage{jmlr2e}
\usepackage{amsfonts,amsmath,amssymb,stmaryrd,bbm}   
\usepackage{nicefrac}       
\usepackage{microtype}      
\usepackage{thmtools}
\usepackage{mathrsfs}
\usepackage{mathtools}
\usepackage{thm-restate}
\usepackage{MnSymbol,wasysym}
\usepackage{dsfont}
\usepackage[all]{xy}
\usepackage{enumitem}
\usepackage{multicol}
\usepackage[dvipsnames]{xcolor}
\definecolor{dgreen}{rgb}{0.00,0.49,0.00}
\definecolor{dblue}{rgb}{0,0.08,0.75}
\usepackage{hyperref}       
\hypersetup{
    colorlinks = false,
    citecolor=dgreen,
    urlcolor=dblue,
    linkcolor=dblue
}
\hypersetup{breaklinks=true}
\usepackage{booktabs} 
\usepackage[retainorgcmds]{IEEEtrantools}
\usepackage{stackengine}
\usepackage{scalerel}
\usepackage[most]{tcolorbox}
\usepackage{tikz-cd}
\usepackage{wrapfig}
\usepackage{tabularx}
\usepackage{float}
\usepackage[hang,flushmargin]{footmisc}
\usepackage{soul}
\newcommand{\change}[1]{{\color{black}#1}}
\newcommand{\changebis}[1]{{\color{black}#1}}

\newcommand{\R}{\mathbb{R}}
\newcommand{\N}{\mathbb{N}}
\newcommand{\Z}{\mathbb{Z}}
\newcommand{\Rd}{\mathbb{R}^{d}}

\newcommand{\cY}{\mathcal{Y}}

\newcommand{\cH}{\mathcal{H}}

\newcommand{\cI}{\mathcal{I}}

\newcommand{\cG}{\mathcal{G}}

\newcommand{\Lp}{L^{p}(\R)}

\renewcommand{\P}{\mathbb{P}}
\newcommand{\E}{\mathbb{E}}

\renewcommand{\emph}[1]{\textbf{\textit{#1}}}

\newtheorem{defn}{Definition}
\newtheorem{prop}{Proposition}
\newtheorem{theo}{Theorem}
\newtheorem{lma}{Lemma}
\newtheorem{cor}{Corollary}
\newtheorem{rem}{Remark}

\DeclareMathOperator*{\argmin}{arg\,min}

\newcommand{\al}{\alpha}
\newcommand{\be}{\beta}

\newcommand{\la}{\lambda}

\usepackage{lastpage}
\jmlrheading{25}{2024}{1-\pageref{LastPage}}{12/23; Revised 5/24}{6/24}{23-1663}{Zhu Li; Dimitri Meunier; Mattes Mollenhauer; Arthur Gretton}

\ShortHeadings{Optimal Vector-Valued Kernel Regression}{Li,Meunier,Mollenhauer,Gretton}
\firstpageno{1}

\begin{document}

\title{\LARGE\bf Towards Optimal Sobolev Norm Rates for the Vector-Valued Regularized Least-Squares Algorithm\vspace{1em}}

\author{\name Zhu Li\thanks{*Equal Contribution} \email zhu.li@ucl.ac.uk \\
       \addr Gatsby Computational Neuroscience Unit\\
       University College London\\
       London, W1T 4JG, UK
       \AND
       \name  Dimitri Meunier\footnotemark[1] \email dimitri.meunier.21@ucl.ac.uk \\
       \addr Gatsby Computational Neuroscience Unit\\
       University College London\\
       London, W1T 4JG, UK
       \AND
       \name Mattes Mollenhauer \email mattes.mollenhauer@merantix-momentum.com \\
       \addr Merantix Momentum\\
       Merantix AI Campus\\
        Max--Urich--Straße 3, 13355 Berlin, Germany
       \AND
       \name Arthur Gretton \email arthur.gretton@gmail.com \\
       \addr Gatsby Computational Neuroscience Unit\\
       University College London\\
       London, W1T 4JG, UK}

\editor{Daniel Hsu}

\maketitle

\begin{abstract}
We present the first optimal rates for infinite-dimensional vector-valued ridge regression on a continuous scale of norms that interpolate between $L_2$ and the hypothesis space, which we consider as a vector-valued reproducing kernel Hilbert space. These rates allow to treat the misspecified case in which the true regression function is not contained in the hypothesis space. We combine standard assumptions on the capacity of the hypothesis space with a novel tensor product construction of vector-valued interpolation spaces in order to characterize the smoothness of the regression function. Our upper bound not only attains the same rate as real-valued kernel ridge regression, but also removes the assumption that the target regression function is bounded. For the lower bound, we reduce the problem to the scalar setting using a projection argument. We show that these rates are optimal in most cases and independent of the dimension of the output space. We illustrate our results for the special case of vector-valued Sobolev spaces.
\end{abstract}


\section{Introduction}

Optimal learning rates for least-squares regression with scalar outputs have been studied  extensively in the context of reproducing kernel Hilbert spaces (RKHS) over the last two decades. While some analyses considered vector-valued outputs, the optimality of vector-valued kernel-based algorithms remained an open question in important settings which include model misspecification or infinite-dimensional response variables.  We consider a data set $D=\left\{\left(x_i, y_i\right)\right\}_{i=1}^n$ of observations independently sampled from a joint unknown distribution $P$ on $E_X \times \cY$, where $\cY$ is a potentially infinite-dimensional output space
and $E_X$ is the covariate space. Let $(X,Y)$ be a random variable taking values in $E_X \times \cY$ distributed according to $P$. The objective is to estimate the \textit{regression function} or \textit{conditional mean function} $F_*: E_X \rightarrow \cY$ given by $F_*(x):=\mathbb{E}[Y \mid X=x]$.
Our focus in this work is to approximate $F_*$ with kernel-based regularized least-squares algorithms, where we pay special attention to the case when $\cY$ is of high or infinite dimension---a setting including important applications in multitask learning, functional data analysis and inference with kernel mean embeddings. We consider the estimate $\hat F_{\lambda}: E_X \to \cY$, obtained by solving the convex optimization problem
\begin{equation}\label{eq:hatflambda}
    \hat F_{\lambda}=\underset{F \in \cG}{\operatorname{argmin}}\left\{\frac{1}{n} \sum_{i=1}^n\left\|y_i-F(x_i)\right\|^2_{\cY} + \lambda\|F\|_{\cG}^2\right\},
\end{equation}
where a vector-valued reproducing kernel Hilbert space (vRKHS) $\cG$ over $E_X$ serves as the hypothesis space and $\lambda>0$ is a regularization parameter. A vRKHS, as detailed in Section~\ref{sec:bg}, is a generalization of the standard RKHS, allowing us to model functions that take values in a Hilbert space. This algorithm is commonly referred to as (vector-valued) \textit{ridge regression} or
simply the \textit{regularized least squares} (RLS) algorithm---even though it can be understood as a special instance of \textit{Tikhonov regularization} in the more general context of regularization theory.
A central theoretical challenge in this context is to establish learning rates, either in expectation or in probability with respect to the distribution of $D$, for the error
\begin{equation}\label{eq:gen_error}
    \|\hat F_{\la} - F_*\|
\end{equation}
in some relevant norm. In this paper, we investigate the behavior of Eq.~(\ref{eq:gen_error}) with respect to the norms of a continuum of suitable Hilbert spaces $[\cG]^{\gamma}$ with $\cG \subseteq [\cG]^{\gamma} \subseteq L_2$; see Definition~\ref{def:inter_ope_norm} for an exact definition. We focus on the class of vector-valued RKHSs induced by the vector-valued kernel 
\begin{equation}\label{eq:vkernel}
K(x,x') := k_{X}(x,x')\change{T},
\end{equation}
where $\change{T}:\cY \to \cY$ is a \change{bounded}
positive-semidefinite
self-adjoint operator and $k_X: E_X \times E_X \to \R$
is a scalar-valued kernel.
This choice of kernel is the de-facto standard for 
infinite-dimensional learning problems, as it allows 
to practically compute $\hat F_{\lambda}$ defined
by Eq.~(\ref{eq:hatflambda}) conveniently in a variety of practical applications. 

\paragraph{\change{Relevant applications.}} Notable examples
for such an infinite-dimensional learning setting are the estimation of dynamical systems \citep{song2009hilbert,kostic2022learning,kostic2023koopman}, functional response regression \citep{kadri2016functional}, structured prediction \citep{ciliberto2016consistent, ciliberto2020general}, 
the estimation of linear operators \citep{mollenhauer2020nonparametric,mollenhauer2022learning}, the conditional mean embedding
\citep{grunewalder2012conditional,grunewalder2012modelling, park2020measure}, causal effect estimation under observed covariates \citep{10.1093/biomet/asad042} and kernel regression with instrumental  \citep{singh2019kernel} and proximal \citep{mastouri2021proximal} variables.
\change{We emphasize that in all of the above applications,
vector-valued kernels of the form \eqref{eq:vkernel} are used,
the identity operator $T = \operatorname{Id}_{\mathcal{Y}} $ being the most popular choice.
An important reason for this choice of kernel is that, even in the infinite-dimensional case,
it allows a convenient numerical evaluation of $\hat F_{\lambda}$ in terms
of a vector-valued representer theorem (see e.g.\ \citealp{grunewalder2012conditional} and \citealp{kadri2016functional}).
}

\paragraph{\change{Related work.}}
The RLS algorithm has been extensively studied in literature \citep[see e.g.,][and references therein]{caponnetto2007optimal,smale2007learning,steinwart2009optimal,blanchard2018optimal,dicker2017kernel,lin2020optimal,fischer2020sobolev}. However, existing results concerning the optimal rates for RLS often cover the real-valued output space only. One exception is the work by \cite{caponnetto2007optimal}, 
where the output space $\cal Y$ is potentially infinite-dimensional. Their analysis does not generally hold for
the kernel $K$ defined by Eq.~(\ref{eq:vkernel}), however, in the setting that $\cY$ is infinite-dimensional, as it relies 
on a trace condition which is violated when $\change{T}$ is not
trace class---a restriction which rules out the choice $\change{T} = \operatorname{Id}_\cY$, used for example in the analysis of conditional mean embedding \citep{grunewalder2012modelling,lietal2022optimal}. This
issue has been noted by multiple authors in the context of specific applications \citep{grunewalder2012conditional,grunewalder2012modelling,kadri2016functional,mollenhauer2020nonparametric, park2020measure}. In addition, \citet{caponnetto2007optimal} assume that $\cal Y$ is finite dimensional in order to obtain the matching lower bound and only consider the well-specified case. \change{We provide a comparison of our results with existing results in Table~\ref{tab:my_label}. For a detailed discussion, please see Section~\ref{sec:comp}.}

\paragraph{\change{Contributions of this work.}}
This manuscript extends the results from the earlier work of \cite{lietal2022optimal} in multiple ways: 
\begin{itemize}
    \item \change{\emph{Vector-valued interpolation spaces.}} While \cite{lietal2022optimal} consider a  specific instance of infinite-dimensional RLS in terms of the conditional mean embedding, by assuming that the response variable $Y$ takes values in a RKHS, we offer an updated analysis of the RLS algorithm which applies to more general infinite-dimensional spaces $\cal Y$. Our study covers both the hard learning scenario (misspecified setting) when $F_* \notin \cG$ and the easy learning scenario (well-specified setting) when $F_* \in \cG$ (see Theorem~\ref{theo:upper_rate_bis}). In order to cover the misspecified case, we construct novel interpolation spaces of vector-valued functions in terms of an isomorphism which allows to represent functions and linear operators in terms of a tensor product \citep{mollenhauer2020nonparametric, mollenhauer2022learning}. In both cases, when $\cal Y$ is real-valued, we recover the same rate as in \cite{fischer2020sobolev}, the current best known rate for real-valued kernel ridge regression in the literature. We provide a thorough comparison to previous works after stating our results. 
    
    \item \change{\emph{Proof technique.}} 
    \change{Building upon our previous work, our definition of the
    vector-valued interpolation spaces allows to 
    modify the integral operator technique 
    while avoiding the aforementioned trace condition for $T$ required by \citet{caponnetto2007optimal}.
    We bypass the trace condition by making use of tensor product arguments. This allows to reduce
    the infinite-dimensional learning scenario
    to known real-valued arguments in multiple instances
    in our proofs.}
    
    \item \change{\emph{Finite dimensions.}} 
    \change{The results of this work naturally cover dimension-free rates for misspecified multitask learning in finite-dimensional spaces. To the
    best of our knowledge, such a setting has not been
    investigated before in the literature.}
    
    \item \change{\emph{Lower rates.}} 
    \change{With the so-called reduction technique, we obtain lower rates for the the general vector-valued learning setting even when $\cal Y$ is infinite-dimensional (see Theorem~\ref{theo:lower_bound}).
    These rates match our upper rates in many cases.}

    \item \change{\emph{Unbounded regression function.}} 
    The available analysis of scalar-valued kernel ridge regression requires boundedness of the regression function ($\sup_{x \in E_X} |F_*(x)| < \infty$), see \citet{fischer2020sobolev}. \changebis{Recently, \cite{zhang2023optimality} proved that same rates can be obtained by replacing the boundedness condition with the weaker assumption that $|F_*|^q$ is integrable for some $q \geq 2$. Later, \cite{zhang2023optimalityspectral} demonstrated that this integrability assumption is automatically satisfied by scalar valued interpolation spaces (defined in Section~\ref{sec:bg}). This so-called \textit{$L_q$-embedding property} allows to completely remove the assumption that the target function is bounded or that its higher moments are integrable. Generalizing the ideas of \citet{zhang2023optimality, zhang2023optimalityspectral}, we derive learning rates in the vector-valued setting without requiring boundedness or integration of the higher moments of $F_*$ (see Theorem~\ref{theo:upper_rate_bis} and Remark~\ref{rem:drop_bond}). Key to this generalisation is Theorem~\ref{th:Lq_embedding}, which states that the $L_q-$embedding property also holds for vector-valued interpolation spaces.}
    

    \item \change{\emph{Vector-valued Sobolev spaces.}} 
    As a final contribution, we study RLS learning in the setting of vector-valued Sobolev spaces (see Definition~\ref{def:vector_sobolev}) which was not covered in \cite{lietal2022optimal}. We obtain, for the first time, the minimax optimal learning rate for RLS learning in vector-valued Sobolev spaces (see Corollary~\ref{cor:upper_sobolev} and~\ref{cor:lower_sobolev}). 
    \change{This contribution shows
    that our definition of the vector-valued interpolation spaces
    admits a natural interpretation in practical applications, rather than just being ``the appropriate technical tool'' to prove rates for misspecified
    vector-valued learning problems.}
    
\end{itemize}

\begin{table}[ht]
    \centering
\begin{tabularx}{\textwidth}{| X | X | X | X | X | X | X |} \hline
& Kernel& Output space & Misspecified case & smoothness & Algorithm & Norm \\ \hline
\citet{blanchard2018optimal} & scalar $k_X(x,x')$& Real-valued & no & Hölder source condition & general& $\gamma$-norm\\ \hline
\citet{fischer2020sobolev} & scalar $k_X(x,x')$ & Real-valued & yes & interpolation space & ridge regression & $\gamma$-norm\\ \hline
\citet{caponnetto2007optimal} & $K(x,x')$ trace class for all $x,x' \in E_X$ & Vector-valued & no & Hölder source condition & ridge regression  & $L_2$-norm \\ \hline
This work & $k_{X}(x,x')\change{T}$ with $\change{T}$ psd.& Vector-valued & yes & vector-valued interpolation space & ridge regression& $\gamma$-norm \\ \hline
\end{tabularx}
    \caption{An overview of articles providing lower rates and corresponding optimal upper rates for kernel-based least squares algorithms based on the integral operator technique in various scenarios under comparable assumptions on the underlying distributions. Note that the $\gamma$-norm is defined in Eq.~\eqref{eq:gamma_norm}.}
    \label{tab:my_label}
\end{table}

\paragraph{\change{Organisation of this paper.}}
In Section~\ref{sec:bg}, we introduce the concept of a vRKHS and the required mathematical preliminaries. Section~\ref{sec:interpolation} discusses the formal construction of vector-valued interpolation spaces, which will be central to our analysis.
We provide upper rates for the vector-valued learning
problem in Section~\ref{sec:upper_rates}, while
a corresponding lower bound on the rates is presented
in Section~\ref{sec:lower}. Section~\ref{sec:sobolev}
sets our results in line with known rates for the scalar
learning setting in the context of Sobolev spaces,
and Section~\ref{sec:comp} compares our result with other contributions. In order to improve the readability,
we defer proofs and technical auxiliary results to the appendices. 

\change{Readers familiar with
the commonly used integral operator framework for kernel ridge regression 
(e.g.\ \citealp{caponnetto2007optimal})
who are mostly interested in the structure of our upper and lower rates may directly refer to
Section~\ref{sec:upper_rates} and Section~\ref{sec:lower}. The precise mathematical definition of the vector-valued interpolation space 
in Section~\ref{sec:interpolation} may be consulted afterwards---technically, it is used as a direct replacement
for other source conditions found in the literature. The aforementioned three sections contain the fundamental additions to the known framework, while the technical 
setup and mathematical background from Section~\ref{sec:bg} are standard.
For convenience, we provide a summary of our most important notation in Table~\ref{table:notations} 
in Section~\ref{sec:interpolation}}.

\section{Mathematical Preliminaries}\label{sec:bg}

Throughout the paper, we consider a random variable $X$ (the covariate) defined on a second countable locally compact Hausdorff space\footnote{Under additional technical assumptions, the results in this paper can also be formulated when $E_X$ is a more general topological space (for example when $E_X$ is Polish). However, some properties of kernels defined on $E_X$ such as the so-called \textit{$c_0$-universality} \citep{carmeli2010vector} simplify the exposition when $E_X$ is a second countable locally compact Hausdorff space; see Remark~\ref{rem:universality}.} $E_X$ endowed with its Borel $\sigma$-field $\mathcal{F}_{E_X}$, and the random variable $Y$ (the output) defined on a potentially infinite dimensional separable real Hilbert space $(\mathcal{Y}, \langle \cdot, \cdot \rangle_{\mathcal{Y}})$ endowed with its Borel $\sigma$-field $\mathcal{F}_{\mathcal{Y}}$. We let $(\Omega, \mathcal{F},\mathbb{P})$ be the underlying probability space with expectation operator $\mathbb{E}$. Let $P$ be the pushforward of $\mathbb{P}$ under $(X,Y)$ and $\pi$ and $\nu$ be the marginal distributions on $E_X$ and $\cY$, respectively; i.e., $X \sim \pi$ and $Y \sim \nu$. We use the Markov kernel $p: E_X \times \mathcal{F}_{\mathcal{Y}} \rightarrow \mathbb{R}_+$ to express
the distribution of $Y$ conditioned on $X$ as 
\[
\mathbb{P}[Y \in A|X = x] = \int_{A} p(x,dy),
\] for all $x \in E_X$ and events $A \in \mathcal{F}_{\mathcal{Y}}$,
see e.g.\ \citet{dudley2002}.


We now introduce some notation related
to linear operators and vector-valued integration. For more details, the reader may consult
\citet{weidmann80linear} and \citet{diestel77vector}, respectively.
We denote the space of real-valued Lebesgue square integrable functions on $(E_X,\mathcal{F}_{E_X})$ with respect to $\pi$ as $L_2(E_X,\mathcal{F}_{E_X},\pi)$ abbreviated $L_2(\pi)$ and similarly for $\nu$ we use $L_2(\mathcal{Y},\mathcal{F}_{\mathcal{Y}},\nu)$ abbreviated $L_2(\nu)$. Let $H$ be a separable real Hilbert space with inner product $\langle \cdot, \cdot \rangle_{H}$. We write $\mathcal{L}(\change{H,H'})$ as the Banach space of bounded linear operators from $\change{H}$ to another Hilbert space $\change{H'}$, equipped with the operator norm $\|\cdot\|_{\change{H \rightarrow H'}}$. When $\change{H=H'}$, we simply write $\mathcal{L}(\change{H})$ instead. We also let $L_p(E_X,\mathcal{F}_{E_X},\pi; \change{H})$, abbreviated $L_p(\pi; \change{H})$, the space of strongly $\mathcal{F}_{E_X}-\mathcal{F}_{\change{H}}$ measurable and Bochner $p$-integrable functions from $E_X$ to $\change{H}$ for $1 \leq p \leq \infty$ with the norms
\begin{equation}
\| f \|_{L_p(\pi;\change{H})}^p 
= \int_{E_X} \|f  \|_{\change{H}}^p \, \mathrm{d}\pi, \quad 1\leq p<\infty, \qquad \| f \|_{L_{\infty}(\pi;\change{H})} = \inf\left\{C \geq 0: \pi\{\|f\|_{\change{H}}> C\}=0\right\}.
\end{equation}
We denote the $p$-Schatten class $S_p(H,H')$ to be the space of all compact operators $C$ from $H$ to $H'$ such that $\|C\|_{S_p(H,H')} := \left\|\left(\sigma_i(C)\right)_{i\in J}\right\|_{\ell_p}$ is finite. Here $\|\left(\sigma_i(C)\right)_{i\in J}\|_{\ell_p}$ is the $\ell_p$ sequence space norm of the sequence of the strictly positive singular values of $C$ indexed by the countable set $J$. For $p = 2$, $S_2(H,H')$ is the Hilbert space of Hilbert-Schmidt operators from $H$ to $H'$. Finally, for two \change{Hilbert} spaces \change{$H,H'$}, we say that \change{$H$} is (continuously) embedded in \change{$H'$} and denote it as 
$\change{H} \hookrightarrow \change{H'}$ if $\change{H}$ can be interpreted as a vector subspace of \change{$H'$} and the inclusion operator $i: \change{H} \to \change{H'}$ 
performing the change of norms with $i x = x$ for $x \in \change{H}$ is continuous; and we say that \change{$H$} is isometrically isomorphic to \change{$H'$} and denote it as $\change{H} \simeq \change{H'}$ if there is a linear isomorphism $\Psi:\change{H} \to \change{H'}$ which is an isometry. We write $H \cong H'$, if the sets coincide with equivalent norms.

{\textbf{Tensor Product of Hilbert Spaces} (\citealp{aubin2000applied}, Section 12)\textbf{:}} 
Denote $H\otimes H'$  the tensor product of Hilbert spaces $H$, $H'$. The Hilbert space $H\otimes H'$ is the completion of the algebraic tensor product with respect to the norm induced by the inner product $\langle x_1\otimes x_1', x_2\otimes x_2'\rangle_{H\otimes H'} = \langle x_1,x_2 \rangle_H \langle x_1', x_2'\rangle_{H'}$ for $x_1,x_2 \in H$ and $x_1', x_2' \in H'$ defined on the elementary tensors
of $H\otimes H'$. This definition extends to
$\operatorname{span}\{x\otimes x'| x\in H, x'\in H'\}$ and finally to
its completion. The space $H\otimes H'$ is separable whenever both $H$ and $H'$ are separable. The element $x\otimes x' \in H \otimes H'$ is treated as the linear rank-one operator $x\otimes x': H' \rightarrow H$ defined by $y' \rightarrow \langle y',x' \rangle_{H'}x $ for $y' \in H'$. 
Based on this identification, the tensor product space $H\otimes H'$ is isometrically isomorphic to the space of Hilbert-Schmidt operators from $H'$ to $H$, i.e., $H\otimes H' \simeq S_2(H',H)$. We will hereafter not make the distinction between those two spaces and see them as identical. If $\{e_i\}_{i \in I}$ and $\{e'_j\}_{j \in J}$ are orthonormal basis in $H$ and $H'$, $\{e_i\otimes e'_j\}_{i\in I, j\in J}$ is an orthonormal basis in $H \otimes H'$. 

\begin{rem}[\citealp{aubin2000applied}, Theorem 12.6.1]\label{rem:tensor_product}
Consider the Bochner space $L_2(\pi;H)$ where $H$ is a separable Hilbert space. One can show that $L_2(\pi;H)$ is isometrically identified with the tensor product space $H \otimes L_2(\pi)$.
\end{rem}

{\bf Reproducing Kernel Hilbert Spaces, Covariance Operators:} 
We let $k_{X}: E_X \times E_X \rightarrow \mathbb{R}$ be a symmetric and positive definite kernel function and $\mathcal{H}_{X}$ be a vector space of functions from $E_X$ to $\mathbb{R}$, endowed with a Hilbert space structure via an inner product $\langle\cdot, \cdot\rangle_{\mathcal{H}_{X}}$. We say $k_{X}$ is a reproducing kernel of $\mathcal{H}_{X}$ if and only if for all $ x \in E_X$ we have $k_{X}(\cdot,x) \in \mathcal{H}_{X}$ and for all $x \in E_X$ and $f \in \mathcal{H}_{X}$, we have $f(x)=\left\langle f, k_{X}(x,\cdot)\right\rangle_{\mathcal{H}_{X}}$. A space $\mathcal{H}_{X}$ which possesses a reproducing kernel is called a reproducing kernel Hilbert space (RKHS; \citealp{berlinet2011reproducing}). We denote the canonical feature map of $\mathcal{H}_{X}$ as $\phi_{X}(x) = k_{X}(\cdot,x)$. 

We require some technical assumptions on the previously defined RKHS and kernel:
\begin{enumerate}[itemsep=0em]
    \item \label{assump:separable}
    $\mathcal{H}_{X}$ is separable, this is satisfied
    if $k_{X}$ is continuous, given that $E_X$ is separable\footnote{This follows from \citet[Lemma 4.33]{steinwart2008support}. Note that the lemma requires separability of $E_X$, which is satisfied since we assume that $E_X$ is a second countable locally compact Hausdorff space.}; 
    
    \item \label{assump:measurable}
    $k_{X}(\cdot,x)$ is measurable for $\pi$-almost all $x \in E_X$;
    
    \item \label{assump:bounded}
    $k_X(x,x) \leq \kappa_X^2$ for $\pi$-almost all $x \in E_X$.
\end{enumerate}

Note that the above assumptions are not restrictive in practice, as well-known kernels such as the Gaussian, Laplacian and Mat{\'e}rn kernels satisfy all of the above assumptions on $E_X \subseteq \mathbb{R}^d$ \citep{sriperumbudur2011universality}. We now introduce some facts about the interplay between $\mathcal{H}_{X}$ and $L_{2}(\pi),$ which has been extensively studied by \cite{smale2004shannon,smale2005shannon}, \cite{de2006discretization} and \cite{steinwart2012mercer}. We first define the (not necessarily injective) embedding $I_{\pi}: \mathcal{H}_{X} \rightarrow L_{2}(\pi)$, mapping a function $f \in \mathcal{H}_{X}$ to its $\pi$-equivalence class $[f]$. The embedding is a well-defined compact operator as long as its Hilbert-Schmidt norm is finite. In fact,
this requirement is satisfied since its Hilbert-Schmidt norm can be computed as \citep[][Lemma 2.2 \& 2.3]{steinwart2012mercer}
$$
\left\|I_{\pi}\right\|_{S_{2}\left(\mathcal{H}_{X}, L_{2}(\pi)\right)}=\|k_{X}\|_{L_{2}(\pi)}:=\left(\int_{E_X} k_{X}(x, x) \mathrm{d} \pi(x)\right)^{1 / 2}<\infty.
$$
The adjoint operator $S_{\pi}:=I_{\pi}^{*}: L_{2}(\pi) \rightarrow \mathcal{H}_{X}$ is an integral operator with respect to the kernel $k_{X}$, i.e. for $f \in L_{2}(\pi)$ and $x \in E_X$ we have \citep[][Theorem 4.27]{steinwart2008support}
$$
\left(S_{\pi} f\right)(x)=\int_{E_X} k_{X}\left(x, x^{\prime}\right) f\left(x^{\prime}\right) \mathrm{d} \pi\left(x^{\prime}\right).
$$
Next, we define the self-adjoint and positive semi-definite integral operators
$$
L_{X}:=I_{\pi} S_{\pi}: L_{2}(\pi) \rightarrow L_{2}(\pi) \quad \text { and } \quad C_{XX}:=S_{\pi} I_{\pi}: \mathcal{H}_{X} \rightarrow \mathcal{H}_{X}.
$$
These operators are trace class and their trace norms satisfy
$$\left\|L_{X}\right\|_{S_{1}\left(L_{2}(\pi)\right)}=\left\| C_{XX}\right\|_{S_{1}(\mathcal{H}_{X})}=\left\|I_{\pi}\right\|_{S_2(\mathcal{H}_{X}, L_{2}(\pi))}^{2}=\left\|S_{\pi}\right\|_{S_2(L_{2}(\pi),\mathcal{H}_{X})}^{2}. $$


{\bf Vector-valued RKHS:} We give a brief overview of the vector-valued reproducing kernel Hilbert space. We refer the reader to \cite{carmeli2006vector} and \cite{carmeli2010vector} for more details.

\begin{defn} Let $K: E_X \times E_X \rightarrow \mathcal{L}(\cY)$ be an operator valued positive-semidefinite (psd) kernel such that $K(x,x') = K(x',x)^*$ for all $x,x' \in E_X$, and for all $x_1,\dots,x_n\in E_X$ and $h_i,h_j \in \cY$,
\[\sum_{i,j =1}^n \langle h_i,K(x_i,x_j)h_j \rangle_{\cY} \geq 0.\]
\end{defn}

Fix $K$, $x \in E_X$, and $h \in \cY$, then $\left[K_{x} h\right](\cdot):=K(\cdot, x) h$
defines a function from $E_X$ to $\cY$. We now consider
$$
\mathcal{G}_{\text {pre }}:=\operatorname{span}\left\{K_{x} h \mid x \in E_X, h \in \cY\right\}
$$
with inner product on $\mathcal{G}_{\text {pre }}$ by linearly extending the expression
\begin{IEEEeqnarray}{rCl}
\left\langle K_{x} h, K_{x^{\prime}} h^{\prime}\right\rangle_{\mathcal{G}}:=\left\langle h, K\left(x, x^{\prime}\right) h^{\prime}\right\rangle_{\cY} . \label{eqn:vrkhs_inp}
\end{IEEEeqnarray}

Let $\mathcal{G}$ be the completion of $\mathcal{G}_{\text{pre}}$ with respect to this inner product. We call $\mathcal{G}$ the vRKHS induced by the kernel $K$. The space $\mathcal{G}$ is a Hilbert space consisting of functions from $E_X$ to $\cY$ with the reproducing property
\begin{IEEEeqnarray}{rCl}
\langle F(x), h\rangle_{\cY}=\left\langle F, K_{x} h\right\rangle_{\mathcal{G}}, \label{eqn:vrkhs_repro}
\end{IEEEeqnarray}
for all $F \in \mathcal{G}, h \in \cY$ and $x \in E_X$. For all $F \in \mathcal{G}$ we obtain
$$
\|F(x)\|_{\cY} \leq\|K(x, x)\|^{1 / 2}_{\cY \rightarrow \cY}\|F\|_{\mathcal{G}}, \quad x \in E_X.
$$
The inner product given by Eq.~(\ref{eqn:vrkhs_inp}) implies that $K_{x}$ is a bounded operator for all $x \in E_X$. For all $F \in \mathcal{G}$ and $x \in E_X$, Eq.~(\ref{eqn:vrkhs_repro}) can be written as $F(x)=K_{x}^{*} F$. The linear operators $K_{x}: \cY \rightarrow \mathcal{G}$ and $K_{x}^{*}: \mathcal{G} \rightarrow \cY$ are bounded with
$$
\left\|K_{x}\right\|_{\cY \rightarrow \mathcal{G}}=\left\|K_{x}^{*}\right\|_{\mathcal{G} \rightarrow \cY}=\|K(x, x)\|_{\cY \rightarrow \cY}^{1 / 2}
$$
and we have $K_{x}^{*} K_{x^{\prime}}=K\left(x, x^{\prime}\right), x, x^{\prime} \in E_X$. In the following, we will denote $\mathcal{G}$ as the vRKHS induced by the kernel $K: E_X \times E_X \rightarrow \mathcal{L}(\mathcal{Y})$ with 
\[K(x,x') := k_{X}(x,x')\operatorname{Id}_{\mathcal{Y}}, \quad x,x' \in E_X.\]

\begin{rem}[\change{General multiplicative kernel}]
\label{rem:general_kernel}
    Without loss of generality, we provide our results for the vRKHS
    $\mathcal{G}$ induced by the operator-valued kernel given by
    $K(x,x') = k_{X}(x,x')\operatorname{Id}_{\mathcal{Y}}$.
    However, with suitably adjusted constants in the assumptions, 
    our results transfer directly to the more general
    vRKHS $\widetilde{\mathcal{G}}$ induced by the more general operator-valued kernel 
    \[
    \widetilde K(x,x') := k_{X}(x,x') \change{T}
    \]  
    where $\change{T}: \mathcal{Y} \to \mathcal{Y}$
    is any \change{bounded} positive-semidefinite self-adjoint operator.
    In fact, this setting is covered by straightforwardly
    replacing the response variable $Y$ with the modified response
    $\widetilde{Y} := \change{T}^{1/2}Y$ in our learning problem. 
    Equivalently, the space $\widetilde{\mathcal{G}}$ is 
    obtained as the vRKHS induced by the kernel $K(x,x') = k_{X}
    (x,x')\operatorname{Id}_{\mathcal{Y}}$ \change{by introducing 
    $\langle y, y' \rangle_{\change{\tilde{\mathcal{Y}}}} 
    := \langle y , \change{T} y' \rangle_\mathcal{Y}$ 
    for all $y,y' \in \mathcal{Y}$, which defines an 
    inner product on the quotient space%
    \footnote{\change{When $T$ is not stricly
    positive definite, then elements in the nullspace
    $\ker(T)$ are simply interpreted as the element $0$ in $\tilde{\mathcal{Y}}$,
    ensuring that $\langle \cdot, \cdot \rangle_{\tilde{\mathcal{Y}}}$
    is a well-defined inner product.}} 
    $\tilde{\mathcal{Y}} := \mathcal{Y} / \ker(T)$.}
    This can readily be seen by the construction of $\widetilde{\mathcal{G}}$.
    By \eqref{eqn:vrkhs_inp}, we have
    \begin{equation}
        \label{eq:general_kernel}
    \left\langle 
        \widetilde{K}_{x} y, \widetilde{K}_{x^{\prime}} y^{\prime}
    \right\rangle_{\widetilde{\mathcal{G}}}
    =
    \left\langle 
        y, k_X\left(x, x^{\prime}\right) \change{T} y^{\prime}
    \right\rangle_{\mathcal{Y}}
        =
    \left\langle 
        y, k_X\left(x, x^{\prime}\right)\operatorname{Id}_{\mathcal{Y}} y^{\prime}
    \right\rangle_{\change{\tilde{\mathcal{Y}}}} 
    \end{equation}
    for all $x,x' \in E_X$ and $y, y' \in \mathcal{Y}$. 
    \change{In fact, 
    we have $\tilde{\mathcal{G}} \simeq \mathcal{H} \otimes \tilde{\mathcal{Y}}$,
    see \citet[Example 3.2]{carmeli2010vector}. In Section~\ref{sec:general_kernel}, we give
    the adjusted constants appearing in our 
    learning rates when $K$ is replaced with $\tilde{K}$.}
\end{rem}

An important property of $\mathcal{G}$ is that elements in $\mathcal{G}$ are isometrically
isomorphic to the space of Hilbert-Schmidt operators between $\mathcal{H}_{X}$ and $\mathcal{Y}$.

\begin{theo}[Example 5 (i) in \citealp{carmeli2010vector}]\label{theo:isometric} For $g \in \mathcal{Y}$ and $f \in \mathcal{H}_{X}$, define the map $\bar{\Psi}$ on the elementary tensors as 
\[\left[\bar{\Psi}\left(g \otimes f\right)\right](x):= f(x)g = \left(g \otimes f\right)\phi_{X}(x).\] 
Then $\bar{\Psi}$ defines an isometric isomorphism between $S_2(\mathcal{H}_{X}, \mathcal{Y})$ and $\mathcal{G}$ through linearity and completion.
\end{theo}

More details regarding Theorem~\ref{theo:isometric} (for the special case where $\mathcal{Y}$ is a RKHS) can be found in \citet[Theorem 4.4]{mollenhauer2020nonparametric}. The isometric isomorphism $\bar{\Psi}$ induces the operator reproducing property stated below.
\begin{cor}\label{theo:operep}
For every function $F\in \mathcal{G}$ there exists an operator $C := \bar{\Psi}^{-1}(F) \in S_2(\mathcal{H}_{X}, \mathcal{Y})$ such that \[F(x) = C\phi_{X}(x) \in \mathcal{Y},\] for all $x \in E_X$ with $\|C\|_{S_2(\mathcal{H}_{X}, \mathcal{Y})} = \|F\|_{\mathcal{G}}$ and vice versa. Conversely, for any pair $F \in \mathcal{G}$ and $C \in S_2(\mathcal{H}_{X}, \mathcal{Y})$, we have $C = \bar{\Psi}^{-1}(F)$ as long as $F(x) = C \phi_{X}(x)$ for all $x \in E_X$.
\end{cor}
The proof of Corollary~\ref{theo:operep} is a simple extension of Lemma $15$ in \cite{ciliberto2016consistent} and Corollary $4.5$ in \cite{mollenhauer2020nonparametric}. Corollary~\ref{theo:operep} shows that the vRKHS $\mathcal{G}$ is generated via the space of Hilbert-Schmidt operators $S_2(\mathcal{H}_{X}, \mathcal{Y})$ \[\mathcal{G}= \left\{F: E_X \rightarrow \mathcal{Y} \mid F = C\phi_{X}(\cdot), \quad C \in S_2(\mathcal{H}_{X}, \mathcal{Y})\right\}.\]

{\bf Vector-valued regression:} We briefly recall the basic setup of regularized least squares regression with Hilbert space-valued random variables. The risk for vector-valued regression is 
\begin{IEEEeqnarray}{rCl}
\mathcal{E}(F):= \mathbb{E}\left[\|Y - F(X)\|^2_{\mathcal{Y}} \right] = \int_{E_X \times \mathcal{Y}} \|y - F(x)\|^2_{\mathcal{Y}}p(x,dy)\pi(dx) , \nonumber
\end{IEEEeqnarray}
for measurable functions $F: E_X \rightarrow \mathcal{Y}$. The analytical minimiser of the risk over all those measurable functions is the \textit{regression function} or the \textit{conditional mean function} $F_\star \in L_2(\mathbb{\pi}; \cY)$ given by
\begin{IEEEeqnarray}{rCl}
F_*(x):= \mathbb{E}[Y \mid X = x] = \int_{\mathcal{Y}} y \, p(x,dy), \quad x \in E_X. \nonumber
\end{IEEEeqnarray}
This fact can for example be proven via a classical decomposition of the risk, see e.g.\ \citet[Theorem A.1]{mollenhauer2020nonparametric}. Throughout the paper, we assume that $\E[\Vert Y \Vert_\cY^2] < +\infty$, i.e., the random variable $Y$ is square integrable. Note that this ensures that we have $F_\star \in L_2(\pi; \cY)$.



We pick $\mathcal{G}$ as an hypothesis space of functions to estimate $F_*$. Given a data set $D = \{(x_i, y_i)\}_{i=1}^n$ independently and identically sampled from the joint distribution of $X$ and $Y$, a regularized estimate of $F_*$ is the solution of the following optimization problem:
\begin{IEEEeqnarray}{rCl}
\hat{F}_{\lambda}:= \argmin_{F \in \mathcal{G}} \frac{1}{n}\sum_{i = 1}^n \left\|y_i - F(x_i)\right\|^2_{\mathcal{Y}} + \lambda \|F\|_{\mathcal{G}}^2, \nonumber 
\end{IEEEeqnarray}
where $\lambda > 0$ is the regularization parameter. According to Corollary~\ref{theo:operep}, $\hat{F}_{\lambda}(\cdot) := \bar{\Psi}\left(\hat{C}_{\lambda}\right)(\cdot)  = \hat{C}_{\lambda}\phi_X(\cdot)$ where
\begin{IEEEeqnarray}{rCl}
\hat{C}_{\lambda}:= \argmin_{C \in S_2(\mathcal{H}_{X}, \mathcal{Y})} \frac{1}{n}\sum_{i = 1}^n \left\|y_i - C\phi_X(x_i)\right\|^2_{\mathcal{Y}} + \lambda \|C\|_{S_2(\mathcal{H}_{X}, \mathcal{Y})}^2, \label{eqn:vkrr_hs}
\end{IEEEeqnarray}
An explicit solution is given by 
$$
\hat{F}_{\lambda}(x) = \hat{C}_{\lambda}\phi_{X}(x)=\sum_{i=1}^{n}y_{i}\beta_{i}(x), \qquad \beta(x) := \left[\mathbf{K}_{XX}+n\lambda \operatorname{Id}\right]^{-1}\mathbf{k}_{Xx} \in \mathbb{R}^n
$$
$$
\begin{aligned}
(\mathbf{K}_{XX})_{ij} & =k(x_{i},x_{j}) \qquad &i,j \in [n] \\
\left(\mathbf{k}_{Xx}\right)_{i} & = k(x_{i},x) \qquad &i \in [n]
\end{aligned}
$$
The above model is well-specified
if the equivalence class $F_\star \in L_2(\pi; \cY)$ admits 
a representative which is contained in $\cG$. In what follows, we will simply write this scenario as $F_\star \in \cG$ by abuse of notation.
We note that universal consistency of this approach
at least requires that $\cG$ is dense in $L_2(\pi; \cY)$ such that for every possible 
$F_\star$, we can achieve
$\hat{F}_{\lambda} \to F_\star$ 
in the norm of $L_2(\pi; \cY)$
either in expectation or with high probability
with respect to the distribution of the
samples $D$ for some admissible regularization scheme $\lambda = \lambda_n \to 0$ whenever $n \to \infty$. This denseness is well-investigated and generally satisfied; see Remark~\ref{rem:vector_inter}. 

{\bf Real-valued Interpolation Space:}
We now introduce the background required in order to characterize the Hilbert spaces used to deal with the misspecified setting $F_* \notin \mathcal{G}$.
We  review the results of 
\citet{steinwart2012mercer} and \citet{fischer2020sobolev}
that set out the eigendecompositions of $L_{X}$ and $C_{XX}$, and apply these in constructing the interpolation spaces  used for the misspecified setting. By the spectral theorem for self-adjoint compact operators, there exists an at most countable index set $I$, a non-increasing sequence $(\mu_i)_{i\in I} > 0$, and a family $(e_i)_{i \in I} \in \mathcal{H}_{X}$, such that $\left([e_i]\right)_{i \in I}$ is an orthonormal basis (ONB) of $\overline{\text{ran}~I_{\pi}} \subseteq L_2(\pi)$ and $(\mu_i^{1/2}e_i)_{i\in I}$ is an ONB of $\left(\operatorname{ker} I_{\pi}\right)^{\perp} \subseteq \mathcal{H}_{X}$, and we have 
\begin{equation} \label{eq:SVD}
    L_{X} = \sum_{i\in I} \mu_i \langle\cdot, [e_i] \rangle_{L_2(\pi)}[e_i], \qquad C_{XX} = \sum_{i \in I} \mu_i \langle\cdot, \mu_i^{\frac{1}{2}}e_i \rangle_{\mathcal{H}_{X}} \mu_i^{\frac{1}{2}}e_i
\end{equation}

For $\alpha \geq 0$, we define the $\alpha$-interpolation space \cite{steinwart2012mercer} by
\begin{equation} \label{eq:interpolation_space}
[\mathcal{H}]_{X}^{\alpha}:=\left\{\sum_{i \in I} a_{i} \mu_{i}^{\alpha / 2}\left[e_{i}\right]:\left(a_{i}\right)_{i \in I} \in \ell_{2}(I)\right\} \subseteq L_{2}(\pi),
\end{equation}
equipped with the $\alpha$-power norm
\begin{equation} \label{eq:gamma_norm}
\left\|\sum_{i \in I} a_{i} \mu_{i}^{\alpha / 2}\left[e_{i}\right]\right\|_{[\mathcal{H}]_{X}^{\alpha}}:=\left\|\left(a_{i}\right)_{i \in I}\right\|_{\ell_{2}(I)}=\left(\sum_{i \in I} a_{i}^{2}\right)^{1 / 2}.
\end{equation}
For $\left(a_{i}\right)_{i\in I} \in \ell_{2}(I)$, the $\alpha$-interpolation space becomes a Hilbert space with inner product defined as \[\left\langle \sum_{i \in I}a_i(\mu_i^{\alpha/2}[e_i]), \sum_{i \in I}b_i(\mu_i^{\alpha/2}[e_i]) \right\rangle_{[\mathcal{H}]_{X}^{\alpha}} = \sum_{i \in I} a_i b_i.\] Moreover, $\left(\mu_{i}^{\alpha / 2}\left[e_{i}\right]\right)_{i \in I}$ forms an ONB of $[\mathcal{H}]_{X}^{\alpha}$ and consequently $[\mathcal{H}]_{X}^{\alpha}$ is a separable Hilbert space. In the following, we use the abbreviation $\|\cdot\|_{\alpha}:=\|\cdot\|_{[\mathcal{H}]_{X}^{\alpha}}$. For $\alpha=0$ we have $[\mathcal{H}]_{X}^{0}=\overline{\operatorname{ran} I_{\pi}} \subseteq L_{2}(\pi)$ with $\|\cdot\|_{0}=\|\cdot\|_{L_{2}(\pi)}$. Moreover, for $\alpha=1$ we have $[\mathcal{H}]_{X}^{1}=\operatorname{ran} I_{\pi}$ and $[\mathcal{H}]_{X}^{1}$ is isometrically isomorphic to the closed subspace $\left(\operatorname{ker} I_{\pi}\right)^{\perp}$ of $\mathcal{H}_{X}$ via $I_{\pi}$, i.e. $\left\|[f]\right\|_{1}=\|f\|_{\mathcal{H}_{X}}$ for $f \in\left(\operatorname{ker} I_{\pi}\right)^{\perp}$. For $0<\beta<\alpha$, we have
\begin{IEEEeqnarray}{rCl} \label{eqn:inter_inclusion}
[\mathcal{H}]_{X}^{\alpha} \hookrightarrow [\mathcal{H}]_{X}^{\beta}  \hookrightarrow [\mathcal{H}]_{X}^{0} \subseteq L_{2}(\pi). 
\end{IEEEeqnarray}
For $\alpha > 0$, the $\alpha$-interpolation space is given by the image of the fractional integral operator, namely
$$
[\cH]_X^\alpha=\operatorname{ran} L_X^{\alpha / 2} \quad \text { and } \quad\left\|L_X^{\alpha / 2} f\right\|_\alpha=\|f\|_{L_2(\pi)}
$$
for $f \in \overline{\operatorname{ran} I_\pi}$.

\begin{rem}[\change{Universality}]
\label{rem:universality}
    Under assumptions \ref{assump:separable} to \ref{assump:bounded} and $E_X$ being a second-countable locally compact Hausdorff space, if $k_X(\cdot, x)$ is continuous and vanishing at infinity, then $[\mathcal{H}]_{X}^{0} = L_{2}(\pi)$ if and only if $\cH_X$ is dense in the space of continuous functions vanishing at infinity equipped with the uniform norm \citep{carmeli2010vector}. Such RKHS are called $c_0-$universal. As a special case of \citet[Proposition 5.6]{carmeli2010vector}, one can show that on $\Rd$, Gaussian, Laplacian, inverse multiquadrics and Matérn kernels are $c_0$-universal.
\end{rem}

\begin{rem}[Interpolation space] \label{rem:interpolation_scalar} The name $\alpha-$interpolation space comes from the fact that for $0<\alpha<1$, the $\alpha$-interpolation space can also be characterized in terms of 
classical interpolation theory of real vector spaces (\citealp[see e.g.,][]{triebel1995interpolation}). In particular, \citet[][Theorem 4.6]{steinwart2012mercer} proved the fact
$$
[\cH]_X^\alpha \cong \left[L_2(\pi),[\cH]_X^1\right]_{\alpha, 2},
$$
where the notation \smash{$\left[L_2(\pi),[\cH]_X^1\right]_{\alpha, 2}$}
denotes the classical 
Hilbert space interpolation of $L_2(\pi)$ and $[\cH]_X^1$ of degree $\alpha$ (``interpolation of the real method''). As the precise
construction of this space is fairly technical and merely used
as an alternative interpretation of $[\cH]_X^\alpha$ here, 
we refer to Appendix~\ref{sec:interpolation_vector} for more details.
\end{rem}

\begin{figure}
\centering
\begin{tikzcd} 
S_2(L_2(\pi), \textcolor{Blue}{\mathcal{Y}}) \arrow[r, "\Psi"']  &  L_2(\pi;\textcolor{Blue}{\mathcal{Y}}) \\
S_2(\mathcal{H}_X, \textcolor{Blue}{\mathcal{Y}}) \arrow[r, "\bar{\Psi}"'] \arrow[u, bend left, dotted, "\mathcal{I}_{\pi}", red, xshift=-4ex] & \mathcal{G}
\end{tikzcd}
\caption{$\Psi$ and $\bar{\Psi}$ are the isometric isomorphisms between each pair of spaces. $\mathcal{I}_{\pi}$ denotes the canonical embedding between the two Hilbert-Schmidt spaces.}
\label{fig:spaces}
\end{figure}

\section{Approximation of \texorpdfstring{$F_*$}{F*} with Vector-valued Interpolation Space}
\label{sec:interpolation}

In this section, we deal with the misspecified setting where $F_* \notin \mathcal{G}$. To do this, we first define the \textit{vector-valued interpolation space} via the tensor product space. We recall from Remark~\ref{rem:tensor_product} that $L_2(\pi;\mathcal{Y})$ is isometrically isomorphic to $S_2\left(L_2(\pi), \mathcal{Y}\right)$ and we denote by $\Psi$ the isometric isomorphism between the two spaces. Similarly, we have $\mathcal{G} \simeq S_2(\mathcal{H}_X, 
\mathcal{Y})$ and we denote by $\bar\Psi$ the isometric isomorphism between both spaces in accordance with Theorem~\ref{theo:isometric}. This is summarized in Figure \ref{fig:spaces}. The second chain of spaces is not isometric to the first but can be naturally embedded into the first as follows. Recall that we denote by $I_{\pi}: \mathcal{H}_X \rightarrow L_2(\pi)$ the embedding that maps each function to its equivalence class, $I_{\pi}(f) = [f]$. We therefore naturally define the embedding $\mathcal{I}_{\pi}: S_2(\mathcal{H}_X, \mathcal{Y}) \rightarrow S_2(L_2(\pi), \mathcal{Y})$ through $\mathcal{I}_{\pi}(g \otimes f) = g \otimes I_{\pi}(f) = g \otimes [f]$ for all $f \in \mathcal{H}_X$, $g \in \mathcal{Y}$, and obtain the extension to the whole space by linearity and continuity.\footnote{$\mathcal{I}_{\pi}$ is formally the tensor product of the operator $I_{\pi}$ with the operator $\operatorname{Id}_{\mathcal{Y}}$, see \citet[Definition 12.4.1.]{aubin2000applied}} Therefore, for $F \in \mathcal{G}$ we define $[F] :=  \Psi \circ \mathcal{I}_{\pi} \circ \bar{\Psi}^{-1}(F)$. In the rest of the paper, every embedding will be denoted using the notation $[~\cdot~]$. A stricter requirement would be to write $[~\cdot~]_{\pi}$ due to dependence on the measure $\pi$, but we omit the subscript for ease of notation.


\begin{table}[htbp] 
    \centering
    \caption{Notation for spaces and operators}
    \change{
    \setlength\tabcolsep{2pt}
    \begin{tabular}{|c|p{0.4\textwidth}|c|p{0.4\textwidth}|} 
        \hline
        \textbf{Symbol} & \textbf{Description} & \textbf{Symbol} & \textbf{Description}\\
        \hline
        $E_X$ & Covariate space & $L_X$ & $L_2$-integral operator \\
        $\cY$ & Output space & $C_{XX}$ & $\cH_X$-covariance operator \\
        $\cH_X$ & Scalar-valued RKHS & $\cG$ & Vector-valued RKHS \\
        $[\cH]_X^{\alpha} $& Scalar-valued $\alpha$-interpolation space & $[\cG]^{\alpha}$ & Vector-valued $\alpha$-interpolation space \\
        $k_X$ & Scalar-valued kernel & $K$ & Vector-valued kernel  \\
        $I_\pi$ & Scalar $L_2$-embedding operator & $\mathcal{I}_{\pi}$ & 
        Vector-valued $L_2$-embedding operator\\
        \hline
    \end{tabular}
    } \label{table:notations}
\end{table}

\begin{defn}[\change{Vector-valued interpolation space}]
\label{def:inter_ope_norm} 
Let $k_X$ be a real-valued kernel with associated RKHS $\mathcal{H}_X$ and let $[\mathcal{H}]_X^{\alpha}$ be the real-valued interpolation space associated to $\mathcal{H}_X$ with some $\alpha \geq 0$. Since $[\mathcal{H}]_X^{\alpha} \subseteq L_2(\pi)$, it is natural to define the vector-valued interpolation space $[\mathcal{G}]^{\alpha}$ as
\[[\mathcal{G}]^{\alpha} 
:= \Psi\left(S_2([\mathcal{H}]_X^{\alpha},\mathcal{Y}) \right) = \{F \mid F = \Psi(C), ~C \in S_2([\mathcal{H}]_X^{\alpha},\mathcal{Y})\}.\]
$[\mathcal{G}]^{\alpha}$ is a Hilbert space equipped with the norm $$\|F\|_{\alpha} := \|C\|_{S_2([\mathcal{H}]_X^{\alpha}, \mathcal{Y})} \qquad (F \in [\mathcal{G}]^{\alpha}),$$ where $C = \Psi^{-1}(F)$. For $\alpha = 0$, we retrieve, $$\|F\|_{0} = \|C\|_{S_2(L_2(\pi), \mathcal{Y})}.$$
\end{defn}

\begin{rem}[\change{Interpolation space inclusions}]
\label{rem:vector_inter} 
The vector-valued interpolation space $[\mathcal{G}]^{\alpha}$ allows us to study the approximation of $F_*$ in the misspecified case. Note that we have $F_* \in L_2(\pi;\mathcal{Y})$ since $Y \in L_2(\P; \cY)$ by assumption. In light of Eq.~(\ref{eqn:inter_inclusion}), for $0< \beta <\alpha$ we have \[[\mathcal{G}]^{\alpha} \hookrightarrow [\mathcal{G}]^{\beta}  \hookrightarrow [\mathcal{G}]^{0}\subseteq L_2(\pi;\mathcal{Y}).\] 
While the well-specified case corresponds to $F_* \in \mathcal{G}$, the misspecified case corresponds to $F_* \in [\mathcal{G}]^{\beta}$ for some $0 \leq \beta <1$. One can see from Remark~\ref{rem:universality} that under assumptions \ref{assump:separable} to \ref{assump:bounded} and $E_X$ being a second-countable locally compact Hausdorff space, $[\mathcal{G}]^{0} = L_2(\pi;\mathcal{Y})$ if and only if $k_X$ is $c_0-$universal.
\end{rem}

Akin to Remark~\ref{rem:interpolation_scalar} for scalar-valued interpolation spaces, our next theorem shows that for $0<\alpha<1$, $[\cG]^{\alpha}$ can be characterized in terms of interpolation spaces.

\begin{theo}[Vector-valued interpolation norm equivalence] \label{th:interpolation_vector}
    For $0<\alpha<1$, 
    $$
    [\cG]^\alpha \cong \left[L_2(\pi; \cY),[\cG]^1\right]_{\alpha, 2}.
    $$
\end{theo}
This result allows us to obtain learning rates for vector-valued Sobolev spaces, see Section~\ref{sec:sobolev} for details.

\section{Upper Learning Rates}
\label{sec:upper_rates}

In this section, we derive the learning rate for the difference between $[\hat{F}_{\lambda}]$ and $F_*$ in the interpolation norm. As our assumptions match those of \cite{fischer2020sobolev}, we include their corresponding labels 
for ease of reference. Recall that $(\mu_i)_{i\in I}$ are the positive eigenvalues of the integral operator. We now list our assumptions:
\begin{itemize}
    \item[$5$.] For some constants $c_1 >0$ and $p \in (0,1]$ and for all $i \in I$,
    \begin{equation}\label{asst:evd}
        \mu_i \leq c_1i^{-1/p}.\tag{EVD}
    \end{equation}
    
    \item[$6$.] For $\alpha \in [p, 1]$, the inclusion map $I^{\alpha, \infty}_{\pi}: [\mathcal{H}]_{X}^{\alpha} \hookrightarrow L_{\infty}(\pi)$ is continuous, there is a constant $A > 0$ such that
    \begin{equation}\label{asst:emb}
        \|I^{\alpha, \infty}_{\pi}\|_{[\mathcal{H}]_{X}^{\alpha} \rightarrow L_{\infty}(\pi)} \leq A \tag{EMB}
    \end{equation}
    
    \item[$7$.] There exists $0 < \beta$ and a constant $B \geq 0$ such that $F_* \in [\mathcal{G}]^{\beta}$
    \begin{equation}\label{asst:src}
        \|F_*\|_{\beta} \leq B. \tag{SRC}
    \end{equation}
    
    We let $C_{*} := \Psi^{-1}(F_*) \in S_2([\mathcal{H}]_X^{\beta}, \mathcal{Y})$.
    
    \item[$8$.]We assume that there are constants $\sigma, R > 0$ such that
    \begin{equation}
        \label{asst:mom}
    \int_{\mathcal{Y}}\|y- F_{*}(x)\|_{\mathcal{Y}}^q p(x,dy) \leq \frac{1}{2}q!\sigma^2R^{q-2}, 
        \tag{MOM}
    \end{equation}
    is satisfied for $\pi$-almost all $x \in E_X$ and all $q \geq 2$.
\end{itemize}

\eqref{asst:evd} is a standard assumption on the eigenvalue decay of the integral operator (see more details in \citealp{caponnetto2007optimal,fischer2020sobolev}). 
Property \eqref{asst:emb} is referred to as the \textit{embedding property} in \cite{fischer2020sobolev}.
It can be shown that it holds if and only if there exists a constant $A \geq 0$ with $\sum_{i \in I} \mu_i^{\alpha} e_i^2(x) \leq A^2$ for $\pi$-almost all $x \in E_X$ \citep[Theorem 9]{fischer2020sobolev}. Since we assume $k_X$ to be bounded, the embedding property always hold true when $\alpha = 1$. Furthermore, \eqref{asst:emb} implies a polynomial eigenvalue decay of order $1/\alpha$, which is why we take $\alpha \geq p$. \eqref{asst:src} is justified by Remark~\ref{rem:vector_inter} and is often referred to as the source condition in literature \citep{caponnetto2007optimal,fischer2020sobolev,lin2018optimal,lin2020optimal}. It measures the smoothness of the regression function $F_*$. In particular, when $\beta \geq 1$, the source condition implies that $F_*$ has a representative from $\mathcal{G}$, indicating the well-specified scenario. However, once we let $\beta < 1$, we are in the misspecified learning setting, which is the main interest in this manuscript. Finally, the \eqref{asst:mom} condition on the Markov kernel $p(x,dy)$ is a Bernstein moment condition used to control the noise of the observations (see \citealp{caponnetto2007optimal,fischer2020sobolev} for more details). If $Y$ is almost surely bounded, for example $\|Y\|_{\cY} \leq Y_{\infty}$ almost surely, then \eqref{asst:mom} is satisfied with $\sigma=R=2Y_{\infty}$. It is possible to prove that the Bernstein condition is equivalent to sub-exponentiality, 
see \citet[Remark 4.9]{mollenhauer2022learning}.




\begin{theo}[\change{Upper learning rates}]\label{theo:upper_rate_bis}
Let $\mathcal{H}_X$ be a RKHS on $E_X$ with respect to a kernel $k_X$ such that assumptions \ref{assump:separable} to \ref{assump:bounded} hold. Furthermore, let the conditions \eqref{asst:evd}, \eqref{asst:emb},  \eqref{asst:mom} be satisfied for some $0 < p \leq \alpha \leq 1$. For $0 \leq \gamma \leq 1$, if \eqref{asst:src} is satisfied with $\gamma < \beta \leq 2$, and then 
\begin{enumerate}
    \item in the case $\beta + p \leq \alpha$, let $\lambda_n = \Theta \left(\left(n/\log^{\theta}(n)\right)^{-\frac{1}{\alpha}}\right)$ for some $\theta > 1$, for all $\tau > \log(5)$ and sufficiently large $n \geq 1$, there is a constant $J > 0$ independent of $n$ and $\tau$ such that \[\left\|[\hat{F}_{\lambda}] - F_*\right\|^2_{\gamma} \leq \tau^2 J\left(\frac{n}{\log ^{\theta} n}\right)^{-\frac{\beta-\gamma}{\alpha}}\] is satisfied with $P^n$-probability not less than $1-5e^{-\tau}$. 
    \item in the case $\beta + p > \alpha$, let $\lambda_n = \Theta \left(n^{-\frac{1}{\beta + p}}\right)$, for all $\tau > \log(5)$ and sufficiently large $n \geq 1$, there is a constant $J > 0$ independent of $n$ and $\tau$ such that \[\left\|[\hat{F}_{\lambda_n}] - F_*\right\|^2_{\gamma} \leq \tau^2 J n^{-\frac{\beta-\gamma}{\beta + p}}\] is satisfied with $P^n$-probability not less than $1-5e^{-\tau}$.
\end{enumerate}
\end{theo}

\begin{rem}[\change{Constants}]
\label{rem:q_mark} 
The index bound hidden in the phrase ``sufficiently large $n$'' depends on the parameters and constants from \eqref{asst:evd} and \eqref{asst:emb}, on $\tau$, on a lower bound $0 < c \leq 1$ for the operator norm $c \leq \|C_{XX}\|$, on $\|F_*\|_{L_{q_{\alpha,\beta}}(\pi; \cY)}$ (see Remark~\ref{rem:drop_bond} and Theorem~\ref{th:Lq_embedding} below) and on the regularization parameter sequence $(\la_n)_{n \geq 1}$. We can see that the constant $J$ in Theorem~\ref{theo:upper_rate_bis} does not depend on $\|F_*\|_{\infty}$, and only depends on the parameters and constants from \eqref{asst:evd}, \eqref{asst:emb},  \eqref{asst:mom}, \eqref{asst:src}, on $\|F_*\|_{L_{q_{\alpha,\beta}}(\pi; \cY)}$, and on the regularization parameter sequence $(\la_n)_{n \geq 1}$.
\end{rem}
Theorem~\ref{theo:upper_rate_bis} states that the learning rate for $[\hat{F}_{\lambda}]$ is governed by the interplay between $p$, $\alpha$, and $\beta$. 
To simplify the discussion, we focus on the $L_2(\pi; \mathcal{Y})$ learning rate, corresponding to $\gamma = 0$. The exponent $\beta / \max\{\alpha, \beta+p\}$ explicitly provides the learning rate. For example, if we have $\alpha \leq \beta$, we obtain a learning rate of $\beta/ (\beta + p)$. In particular, for a Gaussian kernel on a bounded convex set $E_X$ with $\pi$ uniform on $E_X$, $p$ and $\alpha$ are arbitrarily close to $0$ \cite[see e.g.,][Example 2]{meunier2023nonlinear}, and our learning rate can achieve 
$O(\log(n)/ n)$ rate simply by taking $\lambda_n = \Theta\left(\left(\log(n)/n\right)^{1/\beta}\right)$. We address the case of kernels with slower eigenvalue decay such as the Mat{\'e}rn kernel in Section~\ref{sec:sobolev}. 

\begin{rem}[\change{Saturation effect}]
    \change{We note that for $\beta > 2$, the upper learning rate is still valid but saturates, i.e. the best upper bound of the generalization error (in $L_2-$norm) is $n^{-\frac{2}{2 + p}}$. This phenomenon is commonly called the Tikhonov saturation effect and well known in classical
    regularisation theory \citep{engl1996regularization}
    and kernel learning literature (see e.g. \citealp{caponnetto2007optimal, fischer2020sobolev, rudi2015less}). In the real-valued setting, under the additional assumption that the kernel $k_X$ is Hölder continuous and that the conditional variance of the noise is bounded away from $0$, \cite{li2022saturation} recently demonstrated that for any regularization parameter $\lambda_n$, the generalization error for kernel ridge regression is lower bounded with high probability by $n^{-\frac{2}{2 + p}}$. This proves that the saturation effect is unavoidable when the algorithm employs Tikhonov regularization. To benefit from smoothness of the regression function beyond that saturation point at $\beta=2$, one can employ different spectral regularization algorithms as explored by \cite{blanchard2018optimal}. Proving the saturation effect for vector-valued regression with Tikhonov regularization, as well as exploring alternative spectral regularization algorithms, are very important research directions that we leave open for future works.}
\end{rem}

\begin{rem}[\change{Boundedness condition}]
\label{rem:drop_bond}  When analyzing the RLS algorithm for both scalar and vector-valued outputs, it is standard to assume that the regression function $F_\star$ is bounded, as discussed in prior studies \citep[see for example][]{caponnetto2007optimal,fischer2020sobolev}. This boundedness is inherently met when $\beta \geq \alpha$, which falls in line with the assumption~\eqref{asst:emb}. However, when $\beta < \alpha$, this condition must be explicitly assumed. Our previous work, \cite{lietal2022optimal}, did not provide a way to relax this assumption (see Appendix~\ref{sec:learning_rate}). \changebis{Nonetheless, when $\cal Y=\R$, recent insights by \citet{zhang2023optimality} suggest that this boundedness criterion can be substituted with the requirement that $F_*$ belongs to $L_q(\pi;\R)$ for some $q \geq 2$. Furthermore, \cite{zhang2023optimalityspectral} demonstrated that $F_* \in L_q(\pi;\R)$ is automatically satisfied when $F_* \in [\cH]_X^{\beta}$ leading to the $L_q-$embedding property of scalar valued interpolation spaces.} 
\changebis{Adopting a similar methodology, our Theorem~\ref{theo:upper_rate_bis} extends the scope of \cite{zhang2023optimality, zhang2023optimalityspectral} to the case where $\cY$ can be any Hilbert space, not just a subset of $\mathbb{R}$. To achieve this goal we first show that the boundedness requirement on $F_*$ can be weakened to the assumption that $F_* \in L_q(\pi;\cY)$ for some $q \geq 2$, and then we remove this assumption by showing that we have a continuous embedding $[\cG]^{\beta} \hookrightarrow L_q(\pi;\cY).$ 
This $L_q-$embedding property of vector-valued interpolation spaces is given in the next theorem. As will be demonstrated in subsequent sections, these enhancements, building upon the findings in \cite{lietal2022optimal}, are pivotal in achieving minimax rates for the RLS algorithm within many vector-valued RKHS, including vector-valued Sobolev spaces. 
}
\end{rem}

\changebis{
\begin{theo}[$L_q$-embedding property]\label{th:Lq_embedding}
    Let Assumption~(\ref{asst:emb}) be satisfied with parameter $\alpha \in (0,1]$. For any $\beta \in (0,\alpha]$, the inclusion map 
    \begin{equation*}
        I_{\pi}^{q_{\alpha,\beta}}:[\mathcal{G}]^\beta \hookrightarrow L_{q_{\alpha,\beta}}(\pi ; \mathcal{Y})
    \end{equation*}
    is continuous, where $q_{\alpha,\beta} := \frac{2\alpha}{\alpha-\beta}$.
\end{theo}
Notice that when we let $\beta \to \alpha$, we have $q_{\alpha,\beta} \to +\infty$ and we retrieve the property that $[\mathcal{G}]^\alpha \hookrightarrow L_{\infty}(\pi ; \mathcal{Y})$. On the other hand, when $\beta \to 0$, we find $q_{\alpha,\beta} \to 2$ and we retrieve the property that $[\mathcal{G}]^{\beta} \hookrightarrow L_{2}(\pi ; \mathcal{Y})$ for all $\beta \geq 0$. The $L_q-$embedding property allows to characterise the integrability of elements of $[\cG]^{\beta}$ in the intermediate situations where $0 < \beta < \alpha$. 
}

\change{

\subsection{Rates for the general multiplicative kernel}
\label{sec:general_kernel}
    \label{rem:general_kernel2}
As previously discussed in Remark~\ref{rem:general_kernel}, we 
show that the rates from
Theorem~\ref{theo:upper_rate_bis} also hold for the kernel 
\begin{equation*}
    \widetilde K(x,x') := k_{X}(x,x') \change{T}
\end{equation*}
where $\change{T}: \mathcal{Y} \to \mathcal{Y}$ is a bounded positive-semidefinite self-adjoint operator. 
Let $\tilde{\mathcal{G}}$ be the vRKHS induced by the kernel $\tilde{K}$.
By Remark~\ref{rem:general_kernel} 
(see also \citealp[Example 3.2]{carmeli2010vector}),
we obtain upper rates for learning
the conditional mean function $F_\star \in L_2(\pi; \mathcal{Y})$ with 
the general kernel $\tilde{K}$ by applying the transformation 
$y \mapsto T^{1/2}y$ for all $y \in \mathcal{Y}$
and simply invoking Theorem~\ref{theo:upper_rate_bis}.
That is, we are learning $\tilde{F}_\star \in L_2(\pi; \mathcal{Y})$ given by
\begin{equation*}
    \tilde{F_\star}(\cdot) := \E[ T^{1/2} Y \mid X = \cdot ] 
\end{equation*}
with the kernel $K(x,x') = k_{X}(x,x')\operatorname{Id}_{\mathcal{Y}}$.
We first notice that the conditions (\ref{asst:evd}) and (\ref{asst:emb})
do not depend on the choice of $T$ (or equivalently, the choice of norm
on $\mathcal{Y}$). Therefore, it remains to investigate
the constants for which (\ref{asst:mom}) and (\ref{asst:src}) hold for
$\tilde{F}_\star$ with respect to the kernel $K$, under the assumption that $F_\star$ satisfies (\ref{asst:mom}) and 
(\ref{asst:src}) with respect to the kernel $\tilde{K}$---
this allows to apply Theorem~\ref{theo:upper_rate_bis} and directly
extends the upper rates to the general case with adjusted constants.

We first verify (\ref{asst:mom}) for $\tilde{F_\star}$
under the assumption that $F_\star$ satisfies (\ref{asst:mom})
for some $\sigma, R >  0$. We have 
\begin{align*}
    \int_{\mathcal{Y}}\| T^{1/2}y- \tilde{F}_{*}(x) \|_{\mathcal{Y}}^q \, p(x,dy) 
    &= 
    \int_{\mathcal{Y}}\| T^{1/2} (y- F_{*}(x)) \|_{\mathcal{Y}}^q \, p(x,dy) \\
    &\leq 
    \| T \|_{\mathcal{Y} \to \mathcal{Y}}^{q/2}
    \int_{\mathcal{Y}}\| y- F_{*}(x) \|_{\mathcal{Y}}^q \, p(x,dy) 
    \leq 
    \frac{1}{2}q! \tilde{\sigma}^2 \tilde{R}^{q-2}     
\end{align*} 
with $\tilde{\sigma}:= \| T \|_{\mathcal{Y} \to \mathcal{Y}}^{1/2} \sigma $ 
and $\tilde{R} := \| T \|_{\mathcal{Y} \to \mathcal{Y}}^{1/2} R$.

We now assume $F_\star$ satisfies (\ref{asst:src})
with respect to interpolation space $[\mathcal{\tilde{G}}]^\beta$
induced by the kernel $\tilde{K}$.
That is, we have $\| F_\star \|_{[\mathcal{\tilde{G}}]^\beta} < B$ 
for some $B \geq 0$ and $\beta > 0$.
We recall 
$\tilde{\mathcal{G}} \simeq \mathcal{H} \otimes \tilde{\mathcal{Y}}$,
where
$\tilde{\mathcal{Y}} = \mathcal{Y} / \ker(T)$
equipped with the inner product 
$\langle y, y' \rangle_{\change{\tilde{\mathcal{Y}}}} 
= \langle y , \change{T} y' \rangle_\mathcal{Y}$.
Analogously to the interpolation space
$[\mathcal{G}]^{\beta}$, we obtain
$[\tilde{\mathcal{G}}]^{\beta} \simeq 
S_2([\mathcal{H}]_X^{\beta},\mathcal{\tilde Y})$.
Hence, there exists an orthogonal sequence $\{h_i\}_{i \in I}$ 
in $[\mathcal{H}]^{\beta}$
and some sequence 
$\{y_i\}_{i \in I}$ in $\mathcal{Y}$, such that isometrically, we have
$F_\star \simeq \sum_{i \in I} y_i \otimes h_i$. By orthogonality of the $\{h_i\}_{i \in I}$, we have
\begin{equation*}
    \| F_\star \|^2_{[\mathcal{\tilde{G}}]^\beta} 
    = \sum_{i \in I} 
    \| h_i \|^2_{[\mathcal{H}]^{\beta}} 
    \| T^{1/2} y_i \|^2_{\mathcal{Y}}
    = 
    \| \tilde{F}_\star \|^2_{[\mathcal{G}]^\beta}, 
\end{equation*}
confirming (\ref{asst:src}) for $\tilde{F}_\star$ with
respect to the interpolation space $[\mathcal{G}]^\beta$ 
without adjusting the constant $B$.
}

\section{Lower Bound}
\label{sec:lower}

Our final theorem provides a lower bound for the convergence rates, which eventually allows us to confirm the optimality of  the learning rates given in the preceding section. In deriving the lower bound, we need the following extra assumption.
\begin{itemize}
    \item[$8$.] For some constants $c_1,c_2 >0$ and $p \in (0,1]$ and for all $i\in I$,
    \begin{equation}\label{asst:evd+}
        c_2 i^{-1/p} \leq \mu_i \leq c_1i^{-1/p}\tag{EVD+}
    \end{equation}
\end{itemize}


\begin{theo}[\change{Lower learning rates}]
\label{theo:lower_bound} Let $k_X$ be a kernel on $E_X$ such that assumptions \ref{assump:separable} to \ref{assump:bounded}  hold and $\pi$ be a probability distribution on $E_X$ such that \eqref{asst:evd+} holds with $0< p \leq 1$. Then for all parameters $0 < \beta \leq 2$, $0 \leq \gamma \leq 1$ with $\gamma < \beta$ and all constants $\sigma, R, B$, there exist constants $J_0, J, \theta > 0$ such that for all learning methods $D\rightarrow \hat{F}_{D}$ ($D:=\{(x_i,y_i)\}_{i=1}^n$), all $\tau > 0$, and all sufficiently large $n \geq 1$ there is a distribution $P$ defined on $E_X \times \mathcal{Y}$ used to sample $D$, with marginal distribution $\pi$ on $E_X$, such that \eqref{asst:src} with respect to $B,\beta$ and \eqref{asst:mom} with respect to $\sigma, R$ are satisfied, and with $P^n$-probability not less than $1-J_0\tau^{1/\theta}$, 
\[\|[\hat{F}_D] - F_*\|^2_{\gamma} \geq \tau^2 J n^{-\frac{\beta-\gamma}{\beta+p}}.\]
\end{theo}


Theorem~\ref{theo:lower_bound} states that under the assumptions of Theorem~\ref{theo:upper_rate_bis} and Assumption~\eqref{asst:evd+} no learning method can achieve a learning rate faster than 
\begin{equation} \label{eq:lb}
    n^{-\frac{\beta - \gamma}{\beta+p}}.
\end{equation}

To our knowledge, this is the first analysis that demonstrates the lower rate for vector-valued regression in infinite dimension. In the context of regularized regression, \citet{caponnetto2007optimal},
\citet{steinwart2009optimal} and \citet{blanchard2018optimal} provide lower bounds on the learning rate under comparable assumptions. However, one key difference in our analysis is that the output of the regression learning now takes values in a potentially infinite dimensional Hilbert space $\mathcal{Y}$, rather than in $\mathbb{R}$ or $\Rd$. 


\changebis{Our analysis reveals that for $\beta \geq \alpha-p$, the RLS estimator leads to the minimax optimal rate (by combining Theorem~\ref{theo:lower_bound} and case 2 in Theorem~\ref{theo:upper_rate_bis}), namely 
$O(n^{-(\beta-\gamma)/(\beta + p)})$. This scenario is particularly relevant for vector-valued Sobolev RKHSs where $p=\alpha$, a topic we will explore in the following section. We point out that finding the optimal rate for $\beta < \alpha-p$ remains a longstanding challenge, even when the output is in $\mathbb{R}$.}

\section{Example: Vector-valued Sobolev Space} \label{sec:sobolev}
In this section we illustrate our main results in the case of vector-valued Sobolev RKHSs. To this end, we assume that $E_X \subseteq \Rd$ is a bounded domain with smooth
boundary equipped with the Lebesgue measure $\mu$. $L_2(E_X; \mathcal{Y}) := L_2(E_X, \mu; \mathcal{Y})$ denotes the corresponding Bochner space. We start by introducing vector-valued Sobolev spaces. 

\begin{defn}[vSobolev space]\label{def:vector_sobolev}
    For $m \in \mathbb{N}$, the vector-valued Sobolev space $W^{m, 2}(E_X ; \cY)$ is the Hilbert space of all $f \in L_2(E_X ; \cY)$ whose weak derivatives of all orders\footnote{$r := (r_1, \ldots, r_d) \in \N^d$ is a multi-index and $|r|$ denotes the sum of its values.} $|r| \leqslant m$ exist and belong to $L_2(E_X ; \cY)$, endowed with the norm
    $$
    \|f\|_{W^{m, 2}(E_X ; \cY)}^2:=\sum_{|r| \leqslant m}\left\|\partial^r f\right\|_{L_2(E_X ; \cY)}^2.
    $$
    For $m=0$, $W^{0, 2}(E_X ; \cY) := L_2(E_X; \mathcal{Y})$.
\end{defn}
For the definition of weak derivatives of functions in $L_2(E_X ; \cY)$ see \citet[Section 12.7]{aubin2000applied}. The following theorem allows us to connect vector-valued Sobolev spaces to our framework.  

\begin{theo}[\citealt{aubin2000applied}, Theorem 12.7.1] \label{th:vec_sobolev}
For $m \in \N$, the vSobolev space $W^{m, 2}(E_X ; \cY)$ is isometric to the Hilbert tensor product $\cY \otimes W^{m, 2}(E_X)$, where $W^{m, 2}(E_X):= W^{m, 2}(E_X ; \R)$ is the standard scalar-valued Sobolev space.
\end{theo}

When $k_X$ is a translation invariant kernel on $\Rd$ whose Fourier transform behaves as $\left(1+\|\cdot\|_2^2\right)^{-m}$ with $m > d/2$, such as the Matérn kernel (see Definition~\ref{def:sobolev}), the induced RKHS $\cH_X$ restricted to $E_X$ coincides with $W^{m, 2}(E_X)$, and their norms are equivalent, see \citet[Corollary 10.13 and Theorem 10.46]{wendland2004scattered}. Therefore, if we choose such a kernel $k_X$ to construct $K = k_X\operatorname{Id}_{\cY}$
and $\cG$, we obtain by Theorem~\ref{theo:isometric} and Theorem~\ref{th:vec_sobolev},
$$
\cG \simeq \cY \otimes \cH_X \simeq \cY \otimes W^{m, 2}(E_X) \simeq W^{m, 2}(E_X ; \cY).
$$
The induced vector-valued RKHS therefore corresponds to a vector-valued Sobolev space. Furthermore, the interpolation spaces $[\cG]^{\alpha}$, $\alpha \geq 0$, can be characterized as a vector-valued fractional Sobolev space (\citealp[see e.g., Section 5.6][for more details]{hytonen2016analysis}). 

\begin{defn}[Vector-Valued Fractional Sobolev Space] \label{def:frac_sobolev} Fix $r > 0$, and let $m:=\min \{s \in \mathbb{N}: s>r\}$. The vector-valued fractional Sobolev space $W^{r,2}(E_X;\cY)$ is defined by means of the real interpolation method, namely 
$$
W^{r,2}(E_X; \cY):=\left[L_2(E_X;\cY), W^{m,2}(E_X; \cY)\right]_{r / m, 2}.
$$

\end{defn}

\begin{prop}\label{cor:inter_sobolev}
    For all $r \geq 0$, $W^{r,2}(E_X; \cY) \cong \cY \otimes W^{r,2}(E_X)$. Furthermore, if $k_X$ is a kernel on $E_X$ such that $\cG \simeq W^{m, 2}(E_X ; \cY)$ with $m > d/2$, then for all $r  \geq 0$, 
    $$
    [\cG]^{r/m} \cong W^{r,2}(E_X; \cY).
    $$
    i.e., $[\cG]^{r/m} = W^{r,2}(E_X; \cY)$ with equivalent norms.
\end{prop}

\begin{defn}[Matérn kernel]\label{def:sobolev}
For $m \in \N$ with $m>d/2$, the Matérn kernel of order $m$ is defined as for all $x,x' \in \Rd$
\begin{IEEEeqnarray*}{rCl}
k_X\left(x^{\prime}, x\right)= \frac{1}{2^{m-d/2-1}\Gamma(m-d/2)}\left(\sqrt{2(m-d/2)}\left\|x^{\prime}-x\right\|\right)^{m-d / 2} \mathcal{K}_{m-d/2}\left(\sqrt{2(m-d/2)}\left\|x^{\prime}-x\right\|\right)
\end{IEEEeqnarray*}
where $\mathcal{K}_{m-d/2}$ is the modified Bessel function of the second kind of order $m-d/2$ and $\Gamma$ is the Gamma function (see e.g., \citealp{kanagawa2018gaussian}, Examples 2.2 and 2.6). 
\end{defn}

We now specify Theorem~\ref{theo:upper_rate_bis} and Theorem~\ref{theo:lower_bound} in the setting of vector-valued Sobolev spaces. We make the assumption that the marginal $\pi$ is equivalent to the Lebesgue measure so that $L_2(E_X; \mathcal{Y}) \simeq L_2(E_X, \pi; \mathcal{Y})$.

\begin{cor}[vSobolev Upper Rates] \label{cor:upper_sobolev}
Let $k_X$ be a kernel on $E_X$ such that assumptions \ref{assump:separable} to \ref{assump:bounded}  hold and such that $\cG \simeq W^{m, 2}(E_X ; \cY)$ with $m > d/2$, and let $P$ be a probability distribution on $E_X \times \mathcal{Y}$ such that $\pi:=P_{E_X}$ (the marginal distribution on $E_X$) is equivalent to the Lebesgue measure $\mu$ on $E_X$. Furthermore, let $B> 0$ be a constant such that $\|F_*\|_{W^{s, 2}(E_X ; \cY)} \leq B$ for some $0 < s \leq 2m$, and \eqref{asst:mom} be satisfied. Then, for $0 \leq t < s$ and a choice $\lambda_n = \Theta \left(n^{-\frac{m}{s + d/2}}\right)$, for all $\tau > \log(5)$ and sufficiently large $n \geq 1$, there is a constant $J > 0$ independent of $n$ and $\tau$ such that \[\left\|[\hat{F}_{\lambda_n}] - F_*\right\|^2_{W^{t, 2}(E_X ; \cY)} \leq \tau^2 J n^{-\frac{s-t}{s + d/2}}\] is satisfied with $P^n$-probability not less than $1-5e^{-\tau}$.
\end{cor}

\begin{cor}[vSobolev Lower Rates]\label{cor:lower_sobolev}
Let $k_X$ be a kernel on $E_X$ such that assumptions \ref{assump:separable} to \ref{assump:bounded} hold and such that $\cG \simeq W^{m, 2}(E_X ; \cY)$ with $m > d/2$, $P$ be a probability distribution on $E_X \times \mathcal{Y}$ such that $\pi:=P_{E_X}$ (the marginal distribution on $E_X$) is equivalent to the Lebesgue measure $\mu$ on $E_X$. Then for all parameters $0 \leq t < s \leq 2m$, and all constants $\sigma, R,B > 0$ there exist constants $J_0, J, \theta > 0$ such that for all learning methods $D\rightarrow \hat{F}_{D}$ ($D:=\{(x_i,y_i)\}_{i=1}^n$), all $\tau > 0$, and all sufficiently large $n \geq 1$ there is a distribution $P$ defined on $E_X \times \mathcal{Y}$ used to sample $D$, with marginal distribution $\pi$ on $E_X$, such that $\|F_*\|_{W^{s, 2}(E_X ; \cY)} \leq B$ and \eqref{asst:mom} with respect to $\sigma, R$, are satisfied, and with $P^n$-probability not less than $1-J_0\tau^{1/\theta}$, \[\|[\hat{F}_D] - F_*\|^2_{W^{t,2}(E_X; \mathcal{Y})} \geq \tau^2 J n^{-\frac{s-t}{s+d/2}}.\]   
\end{cor}

Corollary~\ref{cor:upper_sobolev} and~\ref{cor:lower_sobolev} are proved by inserting the values for $p, \al$ and $\beta$ from \eqref{asst:evd}, \eqref{asst:emb} and \eqref{asst:src} into Theorem~\ref{theo:upper_rate_bis} and~\ref{theo:lower_bound}. For $\cG \simeq  W^{m, 2}(E_X ; \cY)$ and $\pi$ equivalent to the Lebesgue measure, we show in the appendix that $p = \frac{d}{2m}$, $\alpha = p + \epsilon$ for all $\epsilon>0$ (Proposition~\ref{prop:sobolev_alpha_p} in the appendix), and by Proposition~\ref{cor:inter_sobolev}, $W^{s, 2}(E_X ; \cY) \cong [\cG]^{s/m}$, which implies $\beta = s/m$. Since $\alpha - p$ is 
arbitrarily close to zero, we are in the regime $\beta + p > \alpha$ and achieve the rate $n^{-\frac{\beta-\gamma}{\beta+p}} = n^{-\frac{s-t}{s+d/2}}$ by Theorem~\ref{theo:upper_rate_bis}. 

Our results show that the RLS estimator in Eq.~(\ref{eq:hatflambda}) leads to minimax optimal rates for any $\beta \in (0, 2]$ when $\cG$ is a vector-valued Sobolev RKHS, since the rates obtained in Corollary~\ref{cor:upper_sobolev} match the lower bound in Corollary~\ref{cor:lower_sobolev}. This aligns with the recent findings obtained in \cite{zhang2023optimality} for scalar-valued Sobolev RKHS.

\section{Related Work} \label{sec:comp}
In this section, we compare our results with learning rates obtained in the literature.
Due to the large amount of available types of rates
for the scalar learning setting, we primarily focus on optimal rates
derived under comparable assumptions on the underlying distributions.
For the much less investigated vector-valued learning case, we provide a more general overview of recent results.

As discussed previously, the closest work to our results is \cite{caponnetto2007optimal}. By assuming that 
$F_* \in \cal G$, \cite{caponnetto2007optimal} provide the first analysis of RLS with Tikhonov regularization for when $\cal Y$ is infinite dimensional. In this work, the smoothness of the
Bayes function is naturally expressed as an element of the range
of a power iterate of the corresponding covariance operator
(this is known as a \textit{Hölder source condition}
in regularization theory, see e.g.\ \citealp{blanchard2018optimal}).
They show the $L_2$ learning rate $n^{-\beta/(\beta+p)}$, when $K(x,x')$ is trace class for all $x,x' \in E_X$---this condition is violated for the standard choice of kernel $K(x,x')=k_{X}(x,x')\operatorname{Id}_\cY$ whenever $\cal Y$ is infinite dimensional. For finite dimensional $\cal Y$, \citet{caponnetto2007optimal} obtain the matching lower bound. In contrast, we study the RLS algorithm beyond the well-specified setting. Our analysis covers both the well-specified case and the hard learning scenario with $F_* \notin \cal G$, without assuming the trace class condition for the kernel $K$. When $F_* \in \cal G$, 
the construction of our vector-valued interpolation space
can be interpreted as a generalisation of the Hölder source condition
(this is seen by interpreting the covariance operator in terms
of the tensor product structure of $\cG$, see \citealp{mollenhauer2022learning}).
Hence, in the well-specified case, our rates recover the same rate as \cite{caponnetto2007optimal}. Moreover, instead of $L_2$-rate, we derive a general $\gamma$-learning rate, such that the $L_2$ learning rate is recovered when $\gamma =0$. Finally, we obtain the dimension-free matching lower rate without requiring finite dimensional $\cal Y$. 

In the real-valued RLS setting, \cite{blanchard2018optimal} and \cite{fischer2020sobolev} provide the $\gamma$-learning rate with general regularization schemes for the well-specified case under the Hölder source condition \citep{blanchard2018optimal} and Tikhonov regularization for the misspecified case based on
real-valued interpolation spaces \citep{fischer2020sobolev} respectively. For Tikhonov RLS in the well-specified regime, they obtain the same $L_2$ learning rate as in \cite{caponnetto2007optimal} 
given by $n^{-\beta/(\beta+p)}$. They both provide the matching lower bound when $F_* \in \cal G$. A key difference between the two is that \cite{fischer2020sobolev} extend the learning rate analysis to the hard learning scenario by employing the embedding property. We use similar techniques to \cite{fischer2020sobolev}, and generalize the study to the vector-valued RLS setting through our construction of
vector-valued interpolation spaces. Thus, when $\cal Y$ is real-valued, our results recover the known kernel ridge regression rate of \citet{fischer2020sobolev}.

In addition to the previously mentioned work, there are some comparable results
for infinite-dimensional RLS which do not explicitly contain optimal upper rates and/or do not 
provide corresponding lower bounds.
To our knowledge, \citet{mollenhauer2022learning} derive the first upper learning rates for the infinite-dimensional RLS algorithm for the case of general regularisation schemes which are not exclusively based on vector-valued RKHSs. These rates hold for the $\gamma$-norm 
and cover our setting with the kernel $K(x,x')=k_{X}(x,x')\operatorname{Id}_\cY$ as a special case.
Technically, their approach is similar to the real-valued analysis by \citet{blanchard2018optimal}---thus, they only cover the well-specified setting
under Hölder source conditions. As a major difference compared to our results, 
\citet{mollenhauer2022learning} only consider rates up to the order $O(n^{-1/2})$
without additional assumptions about the marginal of $X$ (which are needed for faster rates). \cite{singh2019kernel} study the vector-valued RLS problem in a setting which is similar to ours. They obtain a suboptimal $O(n^{-1/4})$ upper rate in the well-specified setting, however, due to the use of a less sharp concentration bound. Finally, there are extensive studies concerning real-valued RLS \cite[see e.g.,][and references therein]{bauer2007regularization,smale2007learning,dicker2017kernel,lin2020optimal,lin2018optimal,steinwart2008support,steinwart2009optimal}. In particular, \cite{lin2018optimal} derive the $\gamma$-learning rate of $n^{- (\beta-\gamma) /\max\{\beta+p,1\}}$ using the integral operator technique, while \cite{steinwart2009optimal} obtain an $L_2$ rate of $n^{- \beta /\max\{\beta+p,1\}}$ using an empirical process technique.

\section*{Acknowledgements}
MM is partly supported by the Deutsche Forschungsgemeinschaft (DFG) through grant EXC 2046 \textit{MATH+}, Project EF1-19:  \textit{Machine Learning Enhanced Filtering Methods for Inverse Problems}. ZL, DM and AG are supported by the Gatsby Charitable Foundation. The authors wish to thank Haobo Zhang, Zikai Shen and Ingo Steinwart for helpful discussions.

\bibliography{ref}

\newpage
\appendix
\section*{Appendices}

We now report proofs which were omitted in the main text.
As discussed in Remark~\ref{rem:drop_bond}, the proof
for the upper rates given in Theorem~\ref{theo:upper_rate_bis} builds on the results from conditional mean embedding learning in \cite{lietal2022optimal}. In particular, we proceed in two steps to prove Theorem~\ref{theo:upper_rate_bis}. Firstly, in Section~\ref{sec:learning_rate}, we extend the analysis in \cite{lietal2022optimal} to general vector-valued regression setting with a bounded regression function. Secondly, in Section~\ref{sec:pro_2}, we replace this boundedness assumption by a weaker integrability condition. \changebis{Finally, using Theorem~\ref{th:Lq_embedding}, we show that this integrability condition can be removed, leading to Theorem~\ref{theo:upper_rate_bis}. The proof of Theorem~\ref{th:Lq_embedding} is provided at the end of Section~\ref{sec:pro_2}.}
In Section~\ref{sec:proof_lower_bound},
we prove the lower bound on the rates given in 
Theorem~\ref{theo:lower_bound}. Section~\ref{sec:interpolation_vector} contains the proof Theorem~\ref{th:interpolation_vector}. Section~\ref{sec:proofs_sobolev} contains the proofs for the results related to Sobolev spaces presented in Section~\ref{sec:sobolev}.
Finally, in Section~\ref{sec:auxiliary}, we collect some technical supporting results.


\section{Learning rates for bounded regression function} 
\label{sec:learning_rate}

\begin{theo}\label{theo:upper_rate}
Let $\mathcal{H}_X$ be a RKHS on $E_X$ with respect to a kernel $k_X$ such that assumptions \ref{assump:separable} to \ref{assump:bounded} hold. Let $P$ be a probability distribution on $E_X \times \mathcal{Y}$ with $\pi:=P_{E_X}$ (the marginal distribution on $E_X$). Furthermore, let the conditions \eqref{asst:evd}, \eqref{asst:emb},  \eqref{asst:mom} be satisfied for some $0 < p \leq \alpha \leq 1$ and let $B_{\infty}> 0$ be a constant with $\|F_*\|_{L_{\infty}(\pi;\cY)} \leq B_{\infty}$. Then for $0 \leq \gamma \leq 1$, if \eqref{asst:src} is satisfied with $\gamma < \beta \leq 2$,
\begin{enumerate}
    \item in the case $\beta + p \leq \alpha$ and $\lambda_n = \Theta \left(\left(n/\log^{\theta}(n)\right)^{-\frac{1}{\alpha}}\right)$ for some $\theta > 1$, for all $\tau > \log(4)$ and sufficiently large $n \geq 1$, there is a constant $J > 0$ independent of $n$ and $\tau$ such that \[\left\|[\hat{F}_{\lambda_n}] - F_*\right\|^2_{\gamma} \leq \tau^2 J\left(\frac{n}{\log ^{\theta} n}\right)^{-\frac{\beta-\gamma}{\alpha}}\] is satisfied with $P^n$-probability not less than $1-4e^{-\tau}$. 
    
    \item in the case $\beta + p > \alpha$ and $\lambda_n = \Theta \left(n^{-\frac{1}{\beta + p}}\right)$, for all $\tau > \log(4)$ and sufficiently large $n \geq 1$, there is a constant $J > 0$ independent of $n$ and $\tau$ such that \[\left\|[\hat{F}_{\lambda_n}] - F_*\right\|^2_{\gamma} \leq \tau^2 J n^{-\frac{\beta-\gamma}{\beta + p}}\] is satisfied with $P^n$-probability not less than $1-4e^{-\tau}$.
\end{enumerate}
\end{theo}

\begin{rem}
The proof of Theorem~\ref{theo:upper_rate} reveals that the index bound hidden in the phrase ``sufficiently large $n$'' just depends on the parameters and constants from \eqref{asst:evd} and \eqref{asst:emb}, on $\tau$, on a lower bound $0 < c \leq 1$ for the operator norm $c \leq \|C_{XX}\|$, and on the regularization parameter sequence $(\la_n)_{n \geq 1}$. Moreover, the constant $J$ only depends on the parameters and constants from \eqref{asst:evd}, \eqref{asst:emb},  \eqref{asst:mom}, \eqref{asst:src}, on $B_{\infty}$, and on the regularization parameter sequence $(\la_n)_{n \geq 1}$. 
\end{rem}


\paragraph{Structure of the proof.} 
Recall that $\hat{F}_{\lambda} \in \mathcal{G}$ is defined as $\hat{F}_{\lambda} := \bar{\Psi}\left(\hat{C}_{\lambda}\right)$ where $\hat{C}_{\lambda}$ is solution of Eq.~(\ref{eqn:vkrr_hs}). We introduce its population counterpart, the solution of the following problem:
\begin{IEEEeqnarray}{rCl}
C_{\lambda}:= \argmin_{C \in S_2(\mathcal{H}_{X}, \mathcal{Y})} \mathbb{E}_{P} \left\|Y -C \phi_X(X)\right\|^2_{\mathcal{Y}} + \lambda \|C\|_{S_2(\mathcal{H}_{X}, \mathcal{Y})}^2, \qquad F_{\lambda} := \bar{\Psi}\left(C_{\lambda}\right) \in \cG \nonumber.
\end{IEEEeqnarray}
It can be readily shown (see for example \citealp[Appendix D.1][]{grunewalder2012modelling} and \citealp[Corollary 7.4][]{mollenhauer2020nonparametric}) that 
\begin{equation*}
\begin{aligned}
    C_{\lambda} &= C_{YX}\left(C_{XX} + \lambda Id_{\mathcal{H}_X}\right)^{-1}, \\
    \hat{C}_{\lambda} &= \hat{C}_{YX}\left(\hat{C}_{XX} + \lambda Id_{\mathcal{H}_X} \right)^{-1},
\end{aligned}
\end{equation*}
where $Id_{\mathcal{H}_X}$ is the identity operator in $\cH_X$ and
\begin{equation*}
\begin{aligned}
    C_{XX} &= \mathbb{E}[\phi_{X}(X)\otimes \phi_{X}(X)] \qquad~~~~~ C_{YX} = \mathbb{E}[Y\otimes \phi_{X}(X)]& \\
    \hat{C}_{XX} &= \frac{1}{n}\sum_{i=1}^n\phi_{X}(x_i) \otimes \phi_{X}(x_i) \qquad\hat{C}_{YX} = \frac{1}{n}\sum_{i=1}^n y_i \otimes \phi_{X}(x_i).&
\end{aligned}
\end{equation*}
Finally, recall that $F_{*} \in L_2(\pi; \mathcal{Y})$ and $C_{*} := \Psi^{-1}\left(F_{*}\right)$ is in $S_2(L_2(\pi), \cY)$. From the definition of the vector-valued interpolation norm we introduce the following decomposition,
\begin{IEEEeqnarray}{rCl}
\left\|[\hat{F}_{\lambda}] - F_* \right\|_{\gamma} &\leq& \left\|\left[\hat{F}_{\lambda} - F_{\lambda}\right]\right\|_{\gamma} + \left\|[F_{\lambda}] -F_*\right\|_{\gamma} \nonumber \\ &=& \left\|\left[\hat{C}_{\lambda} - C_{\lambda}\right]\right\|_{S_2\left([\mathcal{H}]_X^{\gamma},\mathcal{Y}\right)} + \left\|[C_{\lambda}] -C_{*} \right\|_{S_2\left([\mathcal{H}]_X^{\gamma},\mathcal{Y}\right)} \label{eqn:risk_decom}
\end{IEEEeqnarray}

We can see that the error for the first term is mainly due to the sample approximation. We therefore refer to the first term as the \textit{Variance}.  We refer to the second term as the \textit{Bias}. Our proof of convergence of the bias adapts the proof by \citet{fischer2020sobolev}, and utilizes the fact that $C_{*}$ is Hilbert-Schmidt to obtain a sharp rate.

\subsection{Bounding the Bias}\label{sec:bias}
In this section, we establish the bound on the bias. The key insight is that due to \citet[Theorem 12.6.1]{aubin2000applied} the conditional mean function can be expressed as a Hilbert-Schmidt operator. The proof generalizes \citet[Lemma 14]{fischer2020sobolev}, which addresses the scalar case; and \citet[Theorem 6]{singh2019kernel}.

\begin{restatable}{lma}{newbias}\label{lma:new_bias}
If $F_{*} \in [\mathcal{G}]^{\beta}$ is satisfied for some $0 \leq \beta \leq 2$, then the following bound is satisfied, for all $\lambda > 0$ and $0 \leq \gamma \leq \beta$:
\begin{equation}
    \left\|[F_{\lambda}] - F_{*}\right\|_{\gamma}^{2} \leq\left\|F_{*}\right\|_{\beta}^{2} \lambda^{\beta-\gamma} \label{eq:bias_bound}
\end{equation}   
\end{restatable}

\begin{proof}
We first recall that since $F_{*} \in [\mathcal{G}]^{\beta}$, $F_{*} = \Psi\left(C_{*} \right)$ with $C_{*} \in S_2([\mathcal{H}]_X^{\beta}, \mathcal{Y})$, furthermore $F_{\lambda} = \bar{\Psi}\left(C_{\lambda}\right)$ with $C_{\lambda} \in S_2(\mathcal{H}_X, \mathcal{Y})$. Hence, $\left\|[F_{\lambda}] - F_{*}\right\|_{\gamma} = \left\|[C_{\lambda}] - C_{*}\right\|_{S_2\left([\mathcal{H}]_X^{\gamma},\mathcal{Y}\right)}$ and $\left\|F_{*}\right\|_{\beta} = \left\|C_{*}\right\|_{S_2\left([\mathcal{H}]_X^{\beta},\mathcal{Y}\right)}$. We first decompose $[C_{\lambda}]- C_{*}$, and follow this by establishing an upper bound on the bias. Since $C_{*} \in S_2([\mathcal{H}]_X^{\beta}, \mathcal{Y}) \subseteq  S_2(\overline{\text{ran}~I_{\pi}}, \mathcal{Y})$, it admits the decomposition
\begin{equation*}
    C_{*}=\sum_{i \in I} \sum_{j \in J} \check{a}_{ij} d_j \otimes [e_{i}].
\end{equation*}
where $(d_j)_{j \in J}$ is any countable basis of $\mathcal{Y}$ and $\sum_{i \in I} \sum_{j \in J} \check{a}_{ij}^2 < +\infty$ with $\check{a}_{ij}=\left\langle C_*,d_j \otimes [e_{i}] \right\rangle_{S_2(L_2(\pi), \mathcal{Y})}  = \left\langle C_* [e_i],d_j\right\rangle_{\mathcal{Y}}$ for all $i \in I, j \in J$ (see e.g.\ \citet{gretton2013introduction}, Lecture on ``testing statistical dependence''). On the other hand, $C_{\lambda} = C_{YX}\left(C_{XX} + \lambda Id_{\mathcal{H}_X} \right)^{-1}$. Since $\left(\mu_{i}^{1 / 2} e_{i}\right)_{i \in I}$ is an ONB of $\left(\operatorname{ker} I_{\pi}\right)^{\perp}$, we can complete it with an at most countable basis $\left(\bar{e}_i\right)_{i \in I'}$ (with $I \cap I' = \varnothing$) of $\operatorname{ker} I_{\pi}$ such that the union of the family forms a basis of $\mathcal{H}_X$. We get a basis of $S_2(\mathcal{H}_X, \mathcal{Y})$ through $\left(d_j \otimes f_i\right)_{i \in I \cup I', j \in J}$ where  $f_i = \mu_{i}^{1 / 2} e_{i}$ if $i \in I$ and $f_i = \bar{e}_i$ if $i \in I'$. By the spectral decomposition of $C_{XX}$ Eq.~(\ref{eq:SVD}), for $a > 0$ we then have
\begin{equation*}
    \left(C_{XX}+\lambda Id_{\mathcal{H}_X}\right)^{-a}=\sum_{i \in I}\left(\mu_{i}+\lambda\right)^{-a}\left\langle\mu_{i}^{1 / 2} e_{i}, \cdot\right\rangle_{\mathcal{H}_X} \mu_{i}^{1 / 2} e_{i}+\lambda^{-a} \sum_{i \in I'}\left\langle\bar{e}_{i}, \cdot\right\rangle_{\mathcal{H}_X} \bar{e}_{i}.
\end{equation*}

Furthermore, 
\begin{IEEEeqnarray*}{rCl}
C_{YX} &=& \mathbb{E}_{P}\left[Y \otimes \phi_{X}(X)\right] \nonumber\\
&=& \mathbb{E}_{X}\left[ \mathbb{E}_{Y|X}\left[Y\right] \otimes \phi_{X}(X)\right] \nonumber\\
&=& \mathbb{E}_{X}\left[F_{*}(X) \otimes \phi_{X}(X)\right] \nonumber\\
&=& \mathbb{E}_{X}\left[\Psi\left(C_{*}\right)(X) \otimes \phi_{X}(X)\right] \nonumber\\
&=& \sum_{i \in I} \sum_{j \in J} \check{a}_{ij} \mathbb{E}_{X}\left[\Psi\left(d_j \otimes [e_{i}]\right)(X) \otimes \phi_{X}(X)\right] \nonumber\\
&=& \sum_{i \in I} \sum_{j \in J} \check{a}_{ij} \mathbb{E}_{X}\left[[e_{i}](X) d_j \otimes \phi_{X}(X)\right].
\end{IEEEeqnarray*}
In the last step we used the explicit form of the isomorphism between $L_2(\pi;\mathcal{Y})$ and $S_2(L_2(\pi), \mathcal{Y})$ mentioned in Remark~\ref{rem:tensor_product}: $\Psi$ is characterized by $\Psi\left(g \otimes f\right) = \left(x \mapsto gf(x) \right)$, for all $g \in \mathcal{Y}, f \in L_2(\pi)$. Then, using that $\left([e_{i}]\right)_{i \in I}$ is an ONS in $L_2(\pi),$

\begin{equation*}
    [C_{\lambda}] = \sum_{i \in I} \sum_{j \in J} \check{a}_{ij}\frac{\mu_i}{\lambda + \mu_i} d_j \otimes [e_{i}],
\end{equation*}
and hence
\begin{equation*} 
    [C_{\lambda}] - C_{*} = -\sum_{i \in I} \sum_{j \in J} \check{a}_{ij}\frac{\lambda}{\lambda + \mu_i} d_j \otimes [e_{i}].
\end{equation*}
We are now ready to compute the upper bound. Parseval's identity w.r.t. the ONB $\left(d_j \otimes \mu_{i}^{\gamma / 2}\left[e_{i}\right]\right)_{i \in I, j \in J}$ of $S_2\left([\mathcal{H}]_{X}^{\gamma}, \mathcal{Y}\right)$ yields
\begin{equation*}
    \begin{aligned}
    \left\|[C_{\lambda}] - C_{*}\right\|_{S_2\left([\mathcal{H}]_X^{\gamma},\mathcal{Y}\right)}^{2} &= \left\|\sum_{i \in I} \sum_{j \in J} \check{a}_{ij}\frac{\lambda}{\lambda + \mu_i} d_j \otimes [e_{i}]\right\|_{S_2\left([\mathcal{H}]_X^{\gamma},\mathcal{Y}\right)}^{2}\\
    &= \sum_{i \in I} \sum_{j \in J} \check{a}_{ij}^2\left(\frac{\lambda}{\lambda + \mu_i}\right)^2\mu_i^{-\gamma}.
    \end{aligned}
\end{equation*}
Next we notice that,
\begin{equation*}
    \begin{aligned}
        \left( \frac{\lambda }{\mu_i + \lambda}\right)^2\mu_i^{-\gamma} &= \left( \frac{\lambda }{\mu_i + \lambda}\right)^2 \mu_i^{-\gamma} \left( \frac{\lambda}{\lambda} \frac{\mu_i+ \lambda}{\mu_i + \lambda}\right)^{\beta-\gamma}\\
        &= \lambda^{\beta-\gamma}\mu_i^{-\beta} \left( \frac{\lambda }{\mu_i + \lambda}\right)^2 \left(\frac{\mu_i}{\mu_i+ \lambda} \right)^{\beta-\gamma} \left( \frac{\mu_i+ \lambda}{ \lambda}\right)^{\beta-\gamma}\\
        &= \lambda^{\beta-\gamma}\mu_i^{-\beta} \left(\frac{\mu_i}{\mu_i+ \lambda} \right)^{\beta-\gamma} \left( \frac{ \lambda}{ \lambda+ \mu_i}\right)^{2-\beta+\gamma}\\
        &\leq \lambda^{\beta-\gamma}\mu_i^{-\beta},
    \end{aligned}
\end{equation*}
where we used $\beta - \gamma \geq 0$ and $2 - \beta + \gamma \geq 0$. Hence,
\begin{equation*}
    \begin{aligned}
    \left\|[C_{\lambda}] - C_{*}\right\|_{S_2\left([\mathcal{H}]_X^{\gamma},\mathcal{Y}\right)}^{2} &\leq \lambda^{\beta - \gamma}\sum_{i \in I} \sum_{j \in J} \check{a}_{ij}^2\mu_i^{-\beta} \\ 
    &= \lambda^{\beta-\gamma}\left\|C_{*}\right\|_{S_2\left([\mathcal{H}]_X^{\beta},\mathcal{Y}\right)}^{2} 
    \end{aligned}
\end{equation*}
\end{proof}

\subsection{Bounding the Variance}
The proof will require several lemmas in its construction, which we now present. We start with a  lemma that allows to go from the $\gamma$-norm of embedded vector-valued maps to their norm in the original Hilbert-Schmidt space. 

\begin{restatable}{lma}{NormTrasfer}\label{theo:gamma_norm_transfer}
For $0 \leq \gamma \leq 1$ and $F \in \mathcal{G}$ the inequality
\begin{IEEEeqnarray}{rCl}
\left\|[F]\right\|_{\gamma} \leq\left\|CC_{XX}^{\frac{1-\gamma}{2}} \right\|_{S_2(\mathcal{H}_X, \mathcal{Y})} \label{lma:bias_upper}
\end{IEEEeqnarray}
holds, where $C = \bar{\Psi}^{-1}(F) \in S_2(\mathcal{H}_X, \mathcal{Y})$. If, in addition, $\gamma<1$ or $C \perp \mathcal{Y} \otimes \operatorname{ker} I_{\pi}$ is satisfied, then the result is an equality.
\end{restatable}
\begin{proof}
Let us fix $F \in \mathcal{G}$, and define $C := \bar{\Psi}^{-1}(F) \in S_2(\mathcal{H}_X, \mathcal{Y})$. Since $\left(\mu_{i}^{1 / 2} e_{i}\right)_{i \in I}$ is an ONB of $\left(\operatorname{ker} I_{\pi}\right)^{\perp}$, we can complete it with a basis $\left(\bar{e}_i\right)_{i \in I'}$ (with $I \cap I' = \varnothing$) of $\operatorname{ker} I_{\pi}$ such that the union of the family forms a basis of $\mathcal{H}_X$. Let $\left(d_j\right)_{j \in J}$ be a basis of $\mathcal{Y}$, we get a basis of $S_2(\mathcal{H}_X, \mathcal{Y})$ through $\left(d_j \otimes f_i\right)_{i \in I \cup I', j \in J}$ where  $f_i = \mu_{i}^{1 / 2} e_{i}$ if $i \in I$ and $f_i = \bar{e}_i$ if $i \in I'$. Then $C$ admits the decomposition $$C=\sum_{i \in I} \sum_{j \in J} a_{ij} d_j \otimes \mu_{i}^{1 / 2} e_{i} + \sum_{i \in I'} \sum_{j \in J} a_{ij} d_j \otimes \bar{e}_{i},$$ where $a_{ij}=\left\langle C,d_j \otimes f_i \right\rangle_{S_2(\mathcal{H}_X, \mathcal{Y})}  = \left\langle C f_i,d_j\right\rangle_{\mathcal{Y}}$ for all $i \in I \cup I', j \in J$. Since $$[C]=\sum_{i \in I} \sum_{j \in J} a_{ij} d_j \otimes \mu_{i}^{1 / 2} [e_{i}],$$ with Parseval's identity w.r.t. the ONB $\left(d_j \otimes \mu_{i}^{\gamma / 2}\left[e_{i}\right]\right)_{i \in I, j \in J}$ of $S_2([\mathcal{H}]_X^{\gamma}, \mathcal{Y})$ this yields
$$
\left\|[C]\right\|_{S_2([\mathcal{H}]_X^{\gamma}, \mathcal{Y})}^{2}=\left\|\sum_{i \in I} \sum_{j \in J} a_{ij}\mu_{i}^{\frac{1-\gamma}{2}} d_j \otimes \mu_{i}^{\gamma / 2} [e_{i}]\right\|_{S_2([\mathcal{H}]_X^{\gamma}, \mathcal{Y})}^{2}=\sum_{i \in I} \sum_{j \in J} a_{ij}^2\mu_{i}^{1-\gamma}.
$$
For $\gamma<1$, the spectral decomposition of $C_{XX}$ Eq.~(\ref{eq:SVD}) together with the fact that $\left(d_j \otimes \mu_{i}^{1 / 2} e_{i}\right)_{i \in I, j \in J}$ is an ONS in $S_2(\mathcal{H}_X, \mathcal{Y})$ yields 
\begin{equation*}
    \begin{aligned}
    \left\|CC_{XX}^{\frac{1-\gamma}{2}}\right\|_{S_2(\mathcal{H}_X, \mathcal{Y})}^{2} &= \left\|C\sum_{l \in I} \mu_l^{\frac{1-\gamma}{2}} \langle\cdot, \mu_l^{\frac{1}{2}}e_l \rangle_{\mathcal{H}_{X}} \mu_l^{\frac{1}{2}}e_l\right\|_{S_2(\mathcal{H}_X, \mathcal{Y})}^{2} \\ &= \sum_{i \in I}\left\|\sum_{l \in I} \mu_l^{\frac{1-\gamma}{2}} \langle \mu_i^{\frac{1}{2}}e_i, \mu_l^{\frac{1}{2}}e_l \rangle_{\mathcal{H}_{X}} \mu_l^{\frac{1}{2}}Ce_l\right\|_{\mathcal{Y}}^{2} + \sum_{i \in I'}\left\|\sum_{l \in I} \mu_l^{\frac{1-\gamma}{2}} \langle \bar{e}_i, \mu_l^{\frac{1}{2}}e_l \rangle_{\mathcal{H}_{X}} \mu_l^{\frac{1}{2}}Ce_l\right\|_{\mathcal{Y}}^{2} \\ &= \sum_{i \in I}\left\|\mu_i^{\frac{1-\gamma}{2}} \mu_i^{\frac{1}{2}}Ce_i\right\|_{\mathcal{Y}}^{2} \\ &= \sum_{i \in I} \sum_{j \in J}\mu_i^{1-\gamma}\left\langle C\left(\mu_i^{\frac{1}{2}}e_i\right), d_j\right\rangle_{\mathcal{Y}}^{2} \\ &= \sum_{i \in I} \sum_{j \in J} a_{ij}^2 \mu_i^{1-\gamma} .
    \end{aligned}
\end{equation*}
This proves the claimed equality in the case of $\gamma<1$. For $\gamma=1$, we have $C_{XX}^{\frac{1-\gamma}{2}}=\operatorname{Id}_{\mathcal{H}_X}$ and the Pythagorean theorem together with Parseval's identity yields
\begin{equation*}
    \begin{aligned}
    \left\|CC_{XX}^{\frac{1-\gamma}{2}}\right\|_{S_2(\mathcal{H}_X, \mathcal{Y})}^{2} &=\left\|\sum_{i \in I} \sum_{j \in J} a_{ij} d_j \otimes \mu_{i}^{1 / 2} e_{i} + \sum_{i \in I'} \sum_{j \in J} a_{ij} d_j \otimes \bar{e}_{i}\right\|_{S_2(\mathcal{H}_X, \mathcal{Y})}^{2} \\ &=\left\|\sum_{i \in I} \sum_{j \in J} a_{ij} d_j \otimes \mu_{i}^{1 / 2} e_{i}\right\|_{S_2(\mathcal{H}_X, \mathcal{Y})}^{2}+\left\| \sum_{i \in I'} \sum_{j \in J} a_{ij} d_j \otimes \bar{e}_{i}\right\|_{S_2(\mathcal{H}_X, \mathcal{Y})}^{2} \\ &=\sum_{i \in I} \sum_{j \in J} a_{ij}^{2}+ \left\| \sum_{i \in I'} \sum_{j \in J} a_{ij} d_j \otimes \bar{e}_{i}\right\|_{S_2(\mathcal{H}_X, \mathcal{Y})}^{2}
    \end{aligned}
\end{equation*}
This gives the claimed equality if $C \perp \mathcal{Y} \otimes \operatorname{ker} I_{\pi}$, as well as the claimed inequality for general $C \in S_2(\mathcal{H}_X, \mathcal{Y})$. We conclude with $\|[F]\|_{\gamma} = \|[C]\|_{S_2([\mathcal{H}]_X^{\gamma}, \mathcal{Y})}$ by definition.
\end{proof}

\begin{lma}\label{lma:bias_bound_gamma}
If $F_{*} \in [\mathcal{G}]^{\beta}$ is satisfied for some $0 \leq \beta \leq 2$, then the following bounds is satisfied, for all $\lambda>0$ and $\gamma \geq 0$:
\begin{equation} \label{eq:lemma3_2}
    \left\|\left[F_{\lambda}\right]\right\|_{\gamma}^{2} \leq \left\|F_{*}\right\|_{\min \{\gamma, \beta\}}^{2} \lambda^{-(\gamma-\beta)_{+}}. 
\end{equation}
\end{lma}

\begin{proof}
By Parseval's identity
$$
\left\|\left[F_{\lambda}\right]\right\|_{\gamma}^{2}=\sum_{i \in I} \sum_{j \in J}\left(\frac{\mu_{i}}{\mu_{i}+\lambda}\right)^{2} \mu_{i}^{-\gamma} \check{a}_{ij}^{2} .
$$
where $\check{a}_{ij} = \left\langle C_{*} [e_i],d_j\right\rangle_{\mathcal{Y}}$ for all $i \in I, j \in J$ as in the proof of Lemma~\ref{lma:new_bias}. In the case of $\gamma \leq \beta$ we bound the fraction by 1 and then Parseval's identity gives us
$$
\left\|\left[F_{\lambda}\right]\right\|_{\gamma}^{2} \leq \sum_{i \in I} \sum_{j \in J} \mu_{i}^{-\gamma} \check{a}_{ij}^{2}=\left\|F_{*}\right\|_{\gamma}^{2} .
$$
In the case of $\gamma>\beta$,
$$
\left\|\left[F_{\lambda}\right]\right\|_{\gamma}^{2}=\sum_{i \in I} \sum_{j \in J}\left(\frac{\mu_{i}^{1-\frac{\gamma-\beta}{2}}}{\mu_{i}+\lambda}\right)^{2} \mu_{i}^{-\beta} \check{a}_{ij}^{2} \leq \lambda^{-(\gamma-\beta)} \sum_{i \in I} \sum_{j \in J} \mu_{i}^{-\beta} \check{a}_{ij}^{2}=\lambda^{-(\gamma-\beta)}\left\|F_{*}\right\|_{\beta}^{2} ,
$$
where  we used Parseval's identity in the equality and Lemma 25 from \cite{fischer2020sobolev}.
\end{proof}

By \eqref{asst:emb}, the inclusion map $I^{\alpha, \infty}_{\pi}: [\mathcal{H}]_{X}^{\alpha} \hookrightarrow L_{\infty}(\pi)$ has bounded norm $A > 0$ i.e. for $f \in [\mathcal{H}]_{X}^{\alpha}$, $f$ is $\pi-$a.e. bounded and $\|f\|_{\infty} \leq A\|f\|_{\alpha}$. We now show that \eqref{asst:emb} automatically implies that the inclusion operator for $[\mathcal{G}]^{\alpha}$ is bounded.

\begin{lma} \label{lma:_infinite_embedding_v_rkhs}
Under \eqref{asst:emb} the inclusion operator $\mathcal{I}_{\pi}^{\alpha, \infty}: [\mathcal{G}]^{\alpha} \hookrightarrow L_{\infty}(\pi; \mathcal{Y})$ is bounded with operator norm less than or equal to $A$.
\end{lma}
$L_{\infty}(\pi; \mathcal{Y})$ denotes the space of $\mathcal{F}_{E_X} - \mathcal{F}_{\mathcal{Y}}$ measurable $\mathcal{Y}$-valued functions (gathered by $\pi$-equivalent classes) that are essentially bounded with respect to $\pi$. $L_{\infty}(\pi; \mathcal{Y})$ is endowed with the norm $\|F\|_\infty := \inf \{c \geq 0 : \|F(x)\|_{\mathcal{Y}} \leq c \text{ for $\pi$-almost all } x \in E_X\}$.

\begin{proof}
For every $F \in [\mathcal{G}]^{\alpha}$, there is a sequence $b_{ij} \in \ell_2(I \times J)$ such that for $\pi-$almost all $x \in E_X$, \[F(x) = \sum_{i \in I, j \in J} b_{ij} d_j \mu_i^{\alpha/2}[e_i](x)\] where $(d_j)_{j \in J}$ is any orthonormal basis of $\mathcal{Y}$ and $\|F\|_{\alpha}^2  = \sum_{i \in I, j \in J} b_{ij}^2$. We consider $F \in [\mathcal{G}]^{\alpha}$ such that $\sum_{i \in I, j \in J} b_{ij}^2 \leq 1$. For $\pi-$almost all $x \in E_X$,
\begin{equation*} 
    \begin{aligned}
        \|F(x)\|_{\mathcal{Y}}^2 &= \left\|\sum_{j \in J}\left(\sum_{i \in I} b_{ij}\mu_i^{\alpha/2}[e_i](x)\right)d_j\right\|_{\mathcal{Y}}^2 \\
        &= \sum_{j \in J}\left(\sum_{i \in I} b_{ij}\mu_i^{\alpha/2}[e_i](x)\right)^2 \\
        & \leq \sum_{j \in J} \left(\sum_{i \in I}b_{ij}^2\right)\left(\sum_{i \in I}\mu_i^{\alpha}[e_i]^2(x)\right)\\
        &\leq A^2 \sum_{j \in J}\sum_{i \in I} b_{ij}^2 \\ &\leq A^2
    \end{aligned}
\end{equation*}
where we used the Cauchy-Schwarz inequality for each $j \in J$ for the first inequality and a consequence of \eqref{asst:emb} in the second inequality (see Theorem 9 in \citealp{fischer2020sobolev}). We therefore conclude $\|\mathcal{I}_{\pi}^{\alpha, \infty}\| \leq A$.
\end{proof}

Combining Lemmas \ref{lma:new_bias}, \ref{lma:bias_bound_gamma} and \ref{lma:_infinite_embedding_v_rkhs} we have the following result.
 
\begin{lma} \label{lma:F_l_bounded}
If $F_{*} \in [\mathcal{G}]^{\beta}$ and \eqref{asst:emb} are satisfied for some $0 \leq \beta \leq 2$ and $0< \alpha \leq 1$, then the following bounds are satisfied, for all $0< \lambda \leq 1$:
\begin{equation} \label{eq:lemma4_1} 
    \left\|\left[F_{\lambda}\right] - F_{*}\right\|_{L_\infty}^{2} \leq \left(\left\|F_{*}\right\|_{L_\infty} + A\|F_*\|_{\beta}\right)^2 \lambda^{\beta - \alpha},
\end{equation}
\begin{equation} \label{eq:lemma4_2} 
    \left\|\left[F_{\lambda}\right]\right\|_{L_\infty}^{2} \leq A^2\left\|F_{*}\right\|_{\min \{\alpha, \beta\}}^{2} \lambda^{-(\alpha-\beta)_{+}}.
\end{equation}
\end{lma}
\begin{proof}
For Eq.~(\ref{eq:lemma4_2}), we use Lemma~\ref{lma:_infinite_embedding_v_rkhs} and Eq.~(\ref{eq:lemma3_2}) in Lemma~\ref{lma:bias_bound_gamma}.
\begin{equation*}
    \left\|\left[F_{\lambda}\right]\right\|_{\infty}^{2} \leq A^2 \left\|\left[F_{\lambda}\right]\right\|_{\alpha}^{2} \leq A^2\left\|F_{*}\right\|_{\min \{\alpha, \beta\}}^{2} \lambda^{-(\alpha-\beta)_{+}}
\end{equation*}

To show Eq.~(\ref{eq:lemma4_1}), in the case $\beta \leq \alpha$ we use the triangle inequality, Eq.~(\ref{eq:lemma4_2}) and $\lambda \leq 1$ to obtain
\begin{equation*}
    \begin{aligned}
        \left\|\left[F_{\lambda}\right] - F_{*}\right\|_{\infty} &\leq \left\|F_{*}\right\|_{\infty} + \left\|\left[F_{\lambda}\right]\right\|_{\infty} \\ &\leq \left(\left\|F_{*}\right\|_{\infty} + A\left\|F_{*}\right\|_{\beta} \right)\lambda^{-\frac{\alpha-\beta}{2}}
    \end{aligned}
\end{equation*}
In the case $\beta > \alpha$, Eq.~(\ref{eq:lemma4_1}) is a consequence of Lemma~\ref{lma:_infinite_embedding_v_rkhs} and Eq.~(\ref{eq:bias_bound}) in Lemma~\ref{lma:new_bias} with $\gamma = \alpha$,
\begin{equation*}
    \left\|\left[F_{\lambda}\right] - F_{*}\right\|_{\infty}^{2} \leq A^2 \left\|\left[F_{\lambda}\right] - F_{*}\right\|_{\alpha}^{2} \leq A^2 \left\|F_{*}\right\|_{\beta}^{2} \lambda^{\beta-\alpha} \leq \left(\left\|F_{*}\right\|_{\infty} + A\|F_*\|_{\beta}\right)^2 \lambda^{\beta - \alpha}.
\end{equation*}
\end{proof}

\begin{theo}\label{thm:variance_bound}

Let $\mathcal{H}_X$ be a RKHS on $E_X$ with respect to a kernel $k_X$ such that assumptions \ref{assump:separable} to \ref{assump:bounded} hold. Let $P$ be a probability distribution on $E_X \times \mathcal{Y}$ with $\pi:=P_{E_X}$ (the marginal distribution on $E_X$). Furthermore, let $\|F_*\|_{\infty} < \infty$, \eqref{asst:emb} and \eqref{asst:mom} be satisfied. We define
$$
\begin{aligned}
M(\lambda) &= \left\|[F_{\lambda}]-F_*\right\|_{\infty},\\ 
\mathcal{N}(\lambda) &=\operatorname{tr}\left(C_{XX}\left(C_{XX}+\lambda \operatorname{Id}_{\cH_X}\right)^{-1}\right), \\
Q_{\lambda} &=\max \{M(\lambda), R \}, \\
g_{\lambda}& = \log \left( 2e\mathcal{N}(\lambda) \frac{\|C_{XX}\|+\lambda}{\|C_{XX}\|}\right).
\end{aligned}
$$
Then, for $0 \leq \gamma \leq 1$, $\tau \geq 1$, $\lambda > 0$ and $n \geq 8A^{2} \tau g_{\lambda} \lambda^{-\alpha}$, with probability $1-4e^{-\tau}$:
\begin{IEEEeqnarray}{rCl}
\left\|\left[\hat{C}_{\lambda} - C_{\lambda}\right]\right\|_{S_2\left([\mathcal{H}]_X^{\gamma},\mathcal{Y}\right)}^2 \leq \frac{576\tau^2}{n\lambda^{\gamma}}\left(\sigma^{2} \mathcal{N}(\lambda)+\frac{\left\|F_{*}-\left[F_{\lambda}\right]\right\|_{L_{2}(\pi; \mathcal{Y})}^{2}A^{2}}{\lambda^{\alpha}} + \frac{2Q_{\lambda}^2A^2}{n\lambda^{\alpha}}\right) \nonumber
\end{IEEEeqnarray}
\end{theo}

\begin{proof}
We first decompose the variance term as
\begin{IEEEeqnarray}{rCl}
  &&\hspace{-0.5cm} \left\|\left[\hat{C}_{\lambda} - C_{\lambda}\right]\right\|_{S_2\left([\mathcal{H}]_X^{\gamma},\mathcal{Y} \right)} \nonumber\\
  & =& \left\|\left[\hat{C}_{Y X}\left(\hat{C}_{X X}+\lambda \operatorname{Id}_{\cH_X}\right)^{-1}-C_{Y X}\left(C_{X X}+\lambda \operatorname{Id}_{\cH_X}\right)^{-1}\right]\right\|_{S_2\left([\mathcal{H}]_X^{\gamma},\mathcal{Y} \right)} \nonumber \\
&\leq& \left\|\left(\hat{C}_{Y X}\left(\hat{C}_{X X}+\lambda \operatorname{Id}_{\cH_X} \right)^{-1}-C_{Y X}\left(C_{X X}+\lambda \operatorname{Id}_{\cH_X}\right)^{-1}\right) C_{X X}^{\frac{1-\gamma}{2}}\right\|_{S_2(\mathcal{H}_X, \mathcal{Y})} \nonumber\\
& \leq& \left\|\left(\hat{C}_{Y X}-C_{Y X}\left(C_{X X}+\lambda \operatorname{Id}_{\cH_X}\right)^{-1}\left(\hat{C}_{X X}+\lambda \operatorname{Id}_{\cH_X}\right)\right)\left(C_{X X}+\lambda \operatorname{Id}_{\cH_X}\right)^{-\frac{1}{2}}\right\|_{S_2(\mathcal{H}_X, \mathcal{Y})} \label{eqn:var1} \\
&&\cdot \left\|\left(C_{X X}+\lambda \operatorname{Id}_{\cH_X} \right)^{\frac{1}{2}}\left(\hat{C}_{X X}+\lambda \operatorname{Id}_{\cH_X} \right)^{-1}\left(C_{X X}+\lambda \operatorname{Id}_{\cH_X} \right)^{\frac{1}{2}}\right\|_{\mathcal{H}_X \to \mathcal{H}_X} \label{eqn:var2} \\
&&\cdot \left\|\left(C_{X X}+\lambda \operatorname{Id}_{\cH_X} \right)^{-\frac{1}{2}}C_{X X}^{\frac{1-\gamma}{2}}\right\|_{\mathcal{H}_X \to \mathcal{H}_X}\label{eqn:var3}
\end{IEEEeqnarray}
where we used Lemma~\ref{theo:gamma_norm_transfer} in the first inequality. Eq.~(\ref{eqn:var2}) is bounded as 
in Lemma $17$ and Theorem $16$ in \cite{fischer2020sobolev},
$$
\left\|\left(C_{X X}+\lambda \operatorname{Id}_{\cH_X} \right)^{\frac{1}{2}}\left(\hat{C}_{X X}+\lambda \operatorname{Id}_{\cH_X} \right)^{-1}\left(C_{X X}+\lambda \operatorname{Id}_{\cH_X} \right)^{\frac{1}{2}}\right\|_{\mathcal{H}_X \to \mathcal{H}_X}\ \leq 3
$$
for $n \geq 8A^{2} \tau g_{\lambda} \lambda^{-\alpha}$ with probability $1-2e^{-\tau}$ for all $\tau \geq 1$. For Eq.~(\ref{eqn:var3}) we have, using Lemma 25 from \cite{fischer2020sobolev}
$$
\left\|\left(C_{X X}+\lambda \operatorname{Id}_{\cH_X} \right)^{-\frac{1}{2}}C_{X X}^{\frac{1-\gamma}{2}}\right\|_{\mathcal{H}_X \to \mathcal{H}_X}\ \leq \sqrt{\sup _{i} \frac{\mu_{i}^{1-\gamma}}{\mu_{i}+\lambda}} \leq \lambda^{-\frac{\gamma}{2}}.
$$

Finally for the bound of Eq.~(\ref{eqn:var1}) Lemma~\ref{lma:empy_concen} show that for $\tau \geq 1$, $\lambda > 0$ and $n \geq 1$ with probability $1-2e^{-\tau}$:
\begin{equation}
    \begin{aligned}
&\bigg\|\left(\hat{C}_{Y X}-C_{Y X}\left(C_{X X}+\lambda \operatorname{Id}_{\cH_X} \right)^{-1}(\hat{C}_{X X}+\lambda \operatorname{Id}_{\cH_X} )\right)\left(C_{X X}+\lambda \operatorname{Id}_{\cH_X} \right)^{-\frac{1}{2}}\bigg\|_{S_2(\mathcal{H}_X, \mathcal{Y})}^2 \\ &\leq \frac{64\tau^2}{n}\left(\sigma^{2} \mathcal{N}(\lambda)+\frac{\left\|F_{*}-\left[F_{\lambda}\right]\right\|_{L_{2}(\pi; \mathcal{Y})}^{2}A^{2}}{\lambda^{\alpha}} + \frac{2Q_{\lambda}^2A^2}{n\lambda^{\alpha}}\right).\nonumber
\end{aligned}
\end{equation}

\end{proof}

\begin{lma}\label{lma:empy_concen}
Assume the conditions in Theorem~\ref{thm:variance_bound} hold. Then for $\tau \geq 1$, $\lambda > 0$ and $n \geq 1$ with probability $1-2e^{-\tau}$:
\begin{equation} 
    \begin{aligned}
&\bigg\|\left(\hat{C}_{Y X}-C_{Y X}\left(C_{X X}+\lambda \operatorname{Id}_{\cH_X} \right)^{-1}(\hat{C}_{X X}+\lambda \operatorname{Id}_{\cH_X} )\right)\left(C_{X X}+\lambda \operatorname{Id}_{\cH_X} \right)^{-\frac{1}{2}}\bigg\|_{S_2(\mathcal{H}_X, \mathcal{Y})}^2 \\ &\leq \frac{64\tau^2}{n}\left(\sigma^{2} \mathcal{N}(\lambda)+\frac{\left\|F_{*}-\left[F_{\lambda}\right]\right\|_{L_{2}(\pi; \mathcal{Y})}^{2}A^{2}}{\lambda^{\alpha}} + \frac{2Q_{\lambda}^2A^2}{n\lambda^{\alpha}}\right). \nonumber
    \end{aligned}
\end{equation}   
\end{lma}

\begin{proof}
We begin with the decomposition
\begin{IEEEeqnarray*}{rCl}
&&\hat{C}_{Y X}-C_{Y X}\left(C_{X X}+\lambda \operatorname{Id}_{\cH_X} \right)^{-1}\left(\hat{C}_{X X}+\lambda \operatorname{Id}_{\cH_X} \right)\\
& =& \hat{C}_{Y X}-C_{Y X}\left(C_{X X}+\lambda \operatorname{Id}_{\cH_X} \right)^{-1}\left(C_{X X}+\lambda \operatorname{Id}_{\cH_X} + \hat{C}_{X X}-C_{X X}\right) \\
& =& \hat{C}_{Y X}-C_{Y X}+C_{Y X}\left(C_{X X}+\lambda \operatorname{Id}_{\cH_X} \right)^{-1}\left(C_{X X}-\hat{C}_{X X}\right) \\
& =& \hat{C}_{Y X}-C_{Y X}\left(C_{X X}+\lambda \operatorname{Id}_{\cH_X} \right)^{-1} \hat{C}_{X X}-\left(C_{Y X}-C_{Y X}\left(C_{X X}+\lambda \operatorname{Id}_{\cH_X} \right)^{-1} C_{X X}\right) \\
&=& \hat{C}_{Y X}-C_{Y X}\left(C_{X X}+\lambda \operatorname{Id}_{\cH_X} \right)^{-1} \hat{\mathbb{E}}[\phi_X(X) \otimes \phi_X(X)]-\left(C_{YX} - C_{Y X}\left(C_{X X}+\lambda \operatorname{Id}_{\cH_X} \right)^{-1} \mathbb{E}[\phi_X(X) \otimes \phi_X(X)]\right) \\
&=& \hat{\mathbb{E}}\left[\left(Y-F_{\lambda}(X)\right) \otimes \phi_X(X)\right]-\mathbb{E}\left[\left(Y-F_{\lambda}(X)\right) \otimes \phi_X(X)\right]
\end{IEEEeqnarray*}

where we denote $\hat{\mathbb{E}}[\phi_X(X) \otimes \phi_X(X)] = \frac{1}{n}\sum_{i=1}^n \phi_X(x_i) \otimes \phi_X(x_i)$. We wish to apply Theorem~\ref{theo:ope_con_steinwart} with $H = S_2(\mathcal{H}_X, \mathcal{Y})$. Consider the random variables $\xi_{0}, \xi_{2}: E_X \times \mathcal{Y} \rightarrow \mathcal{Y} \otimes \mathcal{H}_X$ defined by
\begin{IEEEeqnarray}{rCl}
&\xi_{0}(x, y)&:=\left(y-F_{\lambda}(x)\right) \otimes \phi_X(x), \label{eq:xi0}\\
&\xi_{2}(x, y)&:=\xi_{0}(x, y)\left(C_{XX}+\lambda \operatorname{Id}_{\cH_X} \right)^{-1 / 2}.  \label{eq:xi2}  
\end{IEEEeqnarray}

Since our kernel $k_X$ is bounded,
\begin{equation*}
    \begin{aligned}
        \left\|\xi_{0}(x, y)\right\|_{S_2(\mathcal{H}_X, \mathcal{Y})} &= \left\|y-F_{\lambda}(x)\right\|_{\mathcal{Y}}\|\phi_X(x)\|_{\mathcal{H}_X} \\
        &\leq \left\|y-F_{\lambda}(x)\right\|_{\cY}\kappa_X \\
        &\leq \left(\left\|y\right\|_{\cY} + \left\|F_{\lambda}\right\|_{L_{\infty}(\pi; \mathcal{Y})}\right)\kappa_X, \\
    \end{aligned}
\end{equation*}

is satisfied for $\pi-$almost all $x \in E_X$ and $F_{\lambda}$ is $\pi$-almost surely bounded by Lemma \ref{lma:F_l_bounded}. Since $y \in L_2(\pi;\cal Y)$, we have that $y \in L_1(\pi;\cal Y)$. 
This yields
\begin{equation}
\frac{1}{n} \sum_{i=1}^{n}\left(\xi_{2}\left(x_{i}, y_{i}\right)-\mathbb{E} \xi_{2}\right) = \hat{\mathbb{E}} \xi_{2}-\mathbb{E} \xi_{2} = \left(\hat{C}_{Y X}-C_{Y X}\left(C_{X X}+\lambda \operatorname{Id}_{\cH_X} \right)^{-1}\left(\hat{C}_{X X}+\lambda \operatorname{Id}_{\cH_X} \right)\right)\left(C_{X X}+\lambda \operatorname{Id}_{\cH_X} \right)^{-\frac{1}{2}}, \label{eq:noise_decom}    
\end{equation}

and therefore Eq.~(\ref{eqn:var1}) coincides with the left hand side of a Bernstein's inequality for $H$-valued random variables (Theorem~\ref{theo:ope_con_steinwart}). Consequently, it remains to bound the $m$-th moment of $\xi_{2}$, for $m \geq 2$,
\begin{equation}
\mathbb{E}\left\|\xi_{2}\right\|_{S_2(\mathcal{H}_X, \mathcal{Y})}^{m}=\int_{E_X}\left\|\left(C_{XX}+\lambda \operatorname{Id}_{\cH_X}\right)^{-1 / 2} \phi_X(x)\right\|_{\mathcal{H}_X}^{m} \int_{\mathcal{Y}}\left\|y-F_{\lambda}(x)\right\|_{\mathcal{Y}}^{m} p(x, \mathrm{~d} y) \mathrm{d} \pi(x). \label{eq:xi_decom} 
\end{equation}

First, we consider the inner integral. Using the triangle inequality and \eqref{asst:mom}, 
\begin{IEEEeqnarray}{rCl}
\int_{\mathcal{Y}}\left\|y-F_{\lambda}(x)\right\|_{\mathcal{Y}}^{m} p(x, \mathrm{~d} y) &\leq& 2^{m-1}\left(\left\|\operatorname{Id}_{\mathcal{Y}}-F_{*}(x)\right\|_{L_{m}(p(x, \cdot); \mathcal{Y})}^{m}+\left\|F_{*}(x)-F_{\lambda}(x)\right\|_{\mathcal{Y}}^{m}\right) \nonumber \\
& \leq &\frac{1}{2} m !(2 R)^{m-2} 2 \sigma^{2}+2^{m-1}\left\|F_{*}(x)-F_{\lambda}(x)\right\|_{\mathcal{Y}}^{m}. \label{eq:xi2_decom} 
\end{IEEEeqnarray}
for $\pi$-almost all $x \in E_X$. If we plug this bound into the outer integral and use the abbreviation $h_{x}:=\left(C_{XX}+\lambda\operatorname{Id}_{\cH_X}\right)^{-1 / 2} \phi_X(x)$ we get
\begin{equation} \label{eq:m_moment_1}
\begin{aligned}
\mathbb{E}\left\|\xi_{2}\right\|_{S_2(\mathcal{H}_X, \mathcal{Y})}^{m} \leq \frac{1}{2} m !(2 R)^{m-2} 2 \sigma^{2} \int_{E_X}\left\|h_{x}\right\|_{\mathcal{H}_X}^{m} \mathrm{~d} \pi(x) +2^{m-1} \int_{E_X}\left\|h_{x}\right\|_{\mathcal{H}_X}^{m}\left\|F_{*}(x)-F_{\lambda}(x)\right\|_{\mathcal{Y}}^{m} \mathrm{~d} \pi(x).
\end{aligned}
\end{equation}
Using Lemma~\ref{theo:h_bound}, we can bound the first term in Eq.~(\ref{eq:m_moment_1}) above by 
$$
\begin{aligned}
\frac{1}{2} m !(2 R)^{m-2} 2 \sigma^{2}\int_{E_X}\left\|h_{x}\right\|_{\mathcal{H}_X}^{m} \mathrm{~d} \pi(x) & \leq \frac{1}{2} m !(2 R)^{m-2} 2 \sigma^{2}\left(\frac{A}{\lambda^{\alpha / 2}}\right)^{m-2} \int_{E_X}\left\|h_{x}\right\|_{\mathcal{H}_X}^{2} \mathrm{~d} \pi(x) \\
&=\frac{1}{2}m !\left(\frac{2RA}{\lambda^{\alpha / 2}}\right)^{m-2} 2\sigma^{2} \mathcal{N}(\lambda) \\
&\leq \frac{1}{2} m !\left(\frac{2Q_{\lambda} A}{\lambda^{\alpha / 2}}\right)^{m-2} 2 \sigma^{2} \mathcal{N}(\lambda)
\end{aligned}
$$
where we only used $R \leq Q_{\lambda}$ in the last step. Again, using Lemma~\ref{theo:h_bound}, the second term in Eq.~(\ref{eq:m_moment_1}) can be bounded by
$$
\begin{aligned}
& 2^{m-1} \int_{E_X}\left\|h_{x}\right\|_{\mathcal{H}_X}^{m}\left\|F_{*}(x)-F_{\lambda}(x)\right\|_{\mathcal{Y}}^{m} \mathrm{~d} \pi(x) \\
\leq & \frac{1}{2}\left(\frac{2A}{\lambda^{\alpha / 2}}\right)^{m}M(\lambda)^{m-2} \int_{E_X}\left\|F_{*}(x) - F_{\lambda}(x)\right\|_{\mathcal{Y}}^{2} \mathrm{~d} \pi(x) \\
=& \frac{1}{2}\left(\frac{2AM(\lambda)}{\lambda^{\alpha / 2}}\right)^{m-2}\left\|F_{*}-\left[F_{\lambda}\right]\right\|_{L_{2}(\pi; \mathcal{Y})}^{2} \frac{4A^2}{\lambda^{\alpha}} \\
\leq & \frac{1}{2} m !\left(\frac{2Q_{\lambda}A}{\lambda^{\alpha / 2}}\right)^{m-2}\left\|F_{*}-\left[F_{\lambda}\right]\right\|_{L_{2}(\pi; \mathcal{Y})}^{2} \frac{2A^2}{\lambda^{\alpha}},
\end{aligned}
$$
where we only used $M(\lambda) \leq Q_{\lambda}$ and $2 \leq m!$ in the last step. Finally, we get
\begin{equation}\label{eq:xi2_m}
\mathbb{E}\left\|\xi_{2}\right\|_{S_2(\mathcal{H}_X, \cY)}^{m} \leq \frac{1}{2} m !\left(\frac{2Q_{\lambda}A}{\lambda^{\alpha / 2}}\right)^{m-2} 2\left(\sigma^{2} \mathcal{N}(\lambda)+\left\|F_{*}-\left[F_{\lambda}\right]\right\|_{L_{2}(\pi; \mathcal{Y})}^{2} \frac{A^{2}}{\lambda^{\alpha}}\right)    
\end{equation}

and an application of Bernstein's inequality from Theorem~\ref{theo:ope_con_steinwart} with $L=2Q_{\lambda}A\lambda^{-\alpha / 2}$ and  \\ $\sigma^{2}=2\left(\sigma^{2} \mathcal{N}(\lambda)+\left\|F_{*}-\left[F_{\lambda}\right]\right\|_{L_{2}(\pi; \mathcal{Y})}^{2}A^{2}\lambda^{-\alpha}\right)$ yields the bound. Putting all the terms together, we obtain our result.    
\end{proof}

\subsection{Learning Rates - Proof of Theorem \ref{theo:upper_rate} }\label{sec:lr}
In this section, we aim to establish our upper bound on the learning rate for vector-valued regression by combining the learning rates obtained for the bias and variance.

Let us fix some $\tau \geq 1$ and a lower bound $0 < c \leq 1$ with $c \leq \|C_{XX}\|$. We first show that Theorem~\ref{thm:variance_bound} is applicable. To this end, we prove that there is an index bound $n_0 \geq 1$ such that $n \geq 8A^{2} \tau g_{\lambda_n} \lambda_n^{-\alpha}$ is satisfied for all $n \geq n_0$. Since $\lambda_n \rightarrow 0$ we choose $n_0' \geq 1$ such that $\lambda_n \leq c \leq \min\{1, \|C_{XX}\|\}$ for all $n \geq n_0'$. We get for $n \geq n_0'$, 
\begin{equation}
\begin{aligned}
\frac{8A^{2} \tau g_{\lambda_{n}} \lambda_{n}^{-\alpha}}{n} &=\frac{8A^{2} \tau \lambda_{n}^{-\alpha}}{n} \cdot \log \left(2 e \mathcal{N}\left(\lambda_{n}\right) \frac{\left\|C_{XX}\right\|+\lambda_{n}}{\left\|C_{XX}\right\|}\right) \nonumber\\
& \leq \frac{8A^{2} \tau \lambda_{n}^{-\alpha}}{n} \cdot \log \left(4 D e \lambda_{n}^{-p}\right) \nonumber \\
&=8A^{2}\tau \left(\frac{\log\left(4 D e\right) \lambda_{n}^{-\alpha}}{n}+\frac{p \lambda_{n}^{-\alpha} \log \lambda_{n}^{-1}}{n}\right) \label{eq:a_l}
\end{aligned}
\end{equation}
where the second step uses Lemma~\ref{theo:h_bound}. Hence, it is enough to show $\frac{\log(\lambda_n^{-1})}{n\lambda_n^{\alpha}} \rightarrow 0$. We consider the cases $\beta + p \leq \alpha$ and $\beta + p > \alpha$.
\begin{itemize} 
    \item $\beta + p \leq \alpha.$ By substituting that $\lambda_n = \Theta \left(\left(\frac{n}{\log^{\theta} n}\right)^{-\frac{1}{\alpha}}\right)$ for some $\theta > 1$ we have \[\frac{\lambda_{n}^{-\alpha} \log \lambda_{n}^{-1}}{n}=\Theta\left(\frac{\log(n)}{n}\frac{n}{\log^{\theta}(n)}\right)  = \Theta \left( \frac{1}{\log^{\theta-1}(n)}\right) \rightarrow 0, \text{ as }~n \rightarrow \infty.\]
    \item $\beta + p > \alpha.$ By substituting that $\lambda_n = \Theta \left(n^{-\frac{1}{\beta + p}}\right)$ and using $1 - \frac{\alpha}{\beta + p} > 0$ we have \[\frac{\lambda_{n}^{-\alpha} \log \lambda_{n}^{-1}}{n}=\Theta\left(\frac{\log(n)}{n}n^{\frac{\alpha}{\beta + p}}\right)  = \Theta \left( \frac{\log(n)}{n^{1-\frac{\alpha}{\beta + p}}}\right) \rightarrow 0, \text{ as }~n \rightarrow \infty.\]
\end{itemize}

Consequently, there is a $n_0 \geq n_0'$ with $n \geq 8A^{2} \log \tau g_{\lambda_n} \lambda_n^{-\alpha}$ for all $n \geq n_0$. Moreover, $n_0$ just depends on $\lambda_n, c, D, \tau, A$, and on the parameters $\alpha,p$. 

Let $n \geq n_0$ be fixed. By Theorem~\ref{thm:variance_bound}, we have 
\begin{IEEEeqnarray*}{rCl}
\left\|\left[\hat{C}_{\lambda_n} - C_{\lambda_n}\right]\right\|_{S_2([\mathcal{H}]_X^{\gamma}, \mathcal{Y})}^2 \leq \frac{576\tau^2}{n\lambda_n^{\gamma}}\left(\sigma^{2} \mathcal{N}(\lambda_n)+\frac{\left\|F_{*}-\left[F_{\lambda}\right]\right\|_{L_{2}(\pi; \mathcal{Y})}^{2}A^{2}}{\lambda_n^{\alpha}} + \frac{2Q_{\lambda_n}^2A^2}{n\lambda_n^{\alpha}}\right).
\end{IEEEeqnarray*}

Using Lemma~\ref{theo:h_bound} and Lemma~\ref{lma:bias_bound_gamma} with $\gamma=0$, we have
\begin{IEEEeqnarray}{rCl}
\left\|[\hat{C}_{\lambda_n}-C_{\lambda_n}]\right\|_{S_2([\mathcal{H}]_X^{\gamma}, \mathcal{Y})}^2 \leq \frac{576\tau^2}{n\lambda_n^{\gamma}}\left(\sigma^{2} D\lambda_n^{-p}+A^2\|F_*\|_{\beta}^2\lambda_n^{\beta-\alpha} + \frac{2Q_{\lambda_n}^2A^2}{n\lambda_n^{\alpha}}\right)\nonumber
\end{IEEEeqnarray}
For the last term, using the definition of $Q_{\lambda}$ in Theorem~\ref{thm:variance_bound} with Lemma~\ref{lma:F_l_bounded} and $\lambda_n \leq 1,$ we get 
\begin{IEEEeqnarray}{rCl}
Q_{\lambda_n}^2 &=& \max\{R^2, \left\|[F_{\lambda}]-F_*\right\|_{\infty}^2\} \nonumber \\ &\leq& \max\left\{R^2, \left(\left\|F_{*}\right\|_{\infty} + A \|F_*\|_{\beta}\right)^2 \lambda_n^{-(\alpha-\beta)}\right\} \nonumber \\
&\leq& K_0 \lambda_n^{-(\alpha-\beta)_+}, \nonumber
\end{IEEEeqnarray}
where $K_0 := \max\left\{R^2, \left(B_{\infty} + A\|F_*\|_{\beta}\right)^2\right\}$. Thus, 
\begin{IEEEeqnarray}{rCl}
\left\|[\hat{C}_{\lambda_n}-C_{\lambda_n}]\right\|_{S_2([\mathcal{H}]_X^{\gamma}, \mathcal{Y})}^2 \leq \frac{576\tau^2}{n\lambda_n^{\gamma}}\left(\sigma^2 D\lambda_n^{-p}+A^2\|F_*\|_{\beta}^2\lambda_n^{\beta-\alpha} + 2A^2K_0\frac{1}{n\lambda_n^{\alpha +(\alpha-\beta)_{+}}}\right) \label{var_B_3}.
\end{IEEEeqnarray}

For the first and second terms in the bracket, we use again the fact that $\lambda_n \leq 1,$ and get
\[D\sigma^2\lambda_n^{-p} + A^2\|F_*\|_{\beta}^2 \lambda_n^{-(\alpha-\beta)} \leq \left(D\sigma^2 + A^2\|F_*\|_{\beta}^2\right)\max\{\lambda_n^{-p}, \lambda_n^{-(\alpha-\beta)}\} \leq K_1 \lambda_n^{-\max\{p,\alpha-\beta\}}\]
with $K_1 := D\sigma^2 + A^2\|F_*\|_{\beta}^2$. We now have 
\begin{IEEEeqnarray}{rCl}
\left\|[\hat{C}_{\lambda_n}-C_{\lambda_n}]\right\|_{S_2([\mathcal{H}]_X^{\gamma}, \mathcal{Y})}^2 &\leq & \frac{576\tau^2}{n\lambda_n^{\gamma}}\left(K_1 \lambda_n^{-\max\{p,\alpha-\beta\}} + 2A^2K_0\frac{1}{n\lambda_n^{\alpha +(\alpha-\beta)_{+}}}\right)\nonumber\\
& = & \frac{576\tau^2}{n\lambda_n^{\gamma+\max\{p,\alpha-\beta\}}}\left(K_1 + 2A^2K_0\frac{1}{n\lambda_n^{\alpha +(\alpha-\beta)_{+}-\max\{p,\alpha-\beta\}}}\right).\nonumber
\end{IEEEeqnarray}
Again, we treat the cases $\beta + p \leq \alpha$ and $\beta + p > \alpha$ separately.
\begin{itemize}
    \item $\beta + p \leq \alpha$. In this case we have \[\alpha +(\alpha-\beta)_{+}-\max\{p,\alpha-\beta\} = \alpha.\] Since \(\lambda_n = \Theta \left(\left(\frac{n}{\log^{\theta} n}\right)^{-\frac{1}{\alpha}}\right),\) for some $\theta > 1$ we therefore have \[\frac{1}{n\lambda_n^{\alpha +(\alpha-\beta)_{+}-\max\{p,\alpha-\beta\}}} = \frac{1}{n\lambda^{\alpha}} = \Theta\left( \frac{1}{\log^{\theta}n}\right).\]
    \item $\beta + p > \alpha$. We have $p > \alpha - \beta$ and \(\lambda_n = \Theta \left(n^{-\frac{1}{\beta + p}}\right),\) and hence \[\frac{1}{n\lambda_n^{\alpha +(\alpha-\beta)_{+}-\max\{p,\alpha-\beta\}}} = \frac{1}{n\lambda_n^{\alpha+(\alpha-\beta)_{+} -p}} = \Theta\left(\left(\frac{1}{n}\right)^{1-\frac{\alpha+(\alpha-\beta)_{+}-p}{\beta+p}}\right).\] Using $p > \alpha - \beta$ again gives us $$1-\frac{\alpha+(\alpha-\beta)_{+}-p}{\beta+p} = \frac{2p - (\alpha-\beta)_{+} - (\alpha - \beta)}{\beta+p} > 0.$$
\end{itemize}
As such, there is a constant $K_2 > 0$ with
\[\left\|[\hat{F}_{\lambda_n} - F_{\lambda_n}]\right\|_{\gamma}^2  =\left\|[\hat{C}_{\lambda_n}-C_{\lambda_n}]\right\|_{S_2([\mathcal{H}]_X^{\gamma}, \mathcal{Y})}^2 \leq 576 \frac{\tau^2}{n\lambda_n^{\gamma+\max\{p,\alpha-\beta\}}}\left(K_1 + 2A^2K_0K_2\right)\] for all $n \geq n_0.$ Defining $K_3 := 576(K_1 + 2A^2K_0K_2)$, and using the bias-variance splitting from Eq.~(\ref{eqn:risk_decom}) and Lemma~\ref{lma:new_bias}, we have

\begin{IEEEeqnarray}{rCl}
\left\|[\hat{F}_{\lambda_n}]-F_*\right\|_{\gamma}^2 &\leq & 2\|C_{*}\|_{S_2([\mathcal{H}]_X^{\beta}, \mathcal{Y})}^2 \lambda_n^{\beta-\gamma} + 2K_3\frac{\tau^2}{n\lambda_n^{\gamma+\max\{p,\alpha-\beta\}}}\nonumber\\
&\leq & \tau^2 \lambda_n^{\beta-\gamma}\left(2\|C_{*}\|_{S_2([\mathcal{H}]_X^{\beta}, \mathcal{Y})}^2+2K_3 \frac{1}{n\lambda_n^{\max\{\beta+p,\alpha\}}}\right), \nonumber
\end{IEEEeqnarray}
where we used $\tau \geq 1$. Since in both cases $\beta + p \leq \alpha$ and $\beta + p > \alpha$, $\lambda_n \succcurlyeq n^{-1/\max\{\alpha, \beta + p\}}$ there is some constant $J>0$ such that \[\left\|[\hat{F}_{\lambda_n}]-F_*\right\|_{\gamma}^2 \leq  \tau^2 J \lambda_n^{\beta - \gamma}\] for all $n \geq n_0$.


\section{Proof of Theorem~\ref{theo:upper_rate_bis}}\label{sec:pro_2}
In this section, we prove our main results, Theorem~\ref{theo:upper_rate_bis}. The case where $\beta \geq \alpha$ is identical to Theorem~\ref{theo:upper_rate}. Hence we will only focus on the rate when $\beta < \alpha$ without assuming the boundedness of $F_*$. We adopt the same risk decomposition as in Eq.~\eqref{eqn:risk_decom}. The bias is upper bounded in the same way as in Section~\ref{sec:bias}. For the variance, we have the following results.

\begin{theo}\label{thm:variance_bound_q}
Let $\mathcal{H}_X$ be a RKHS on $E_X$ with respect to a kernel $k_X$ such that assumptions \ref{assump:separable} to \ref{assump:bounded} hold. Let $P$ be a probability distribution on $E_X \times \mathcal{Y}$ with $\pi:=P_{E_X}$ (the marginal distribution on $E_X$). Furthermore, let \eqref{asst:emb}, \eqref{asst:src} and \eqref{asst:mom} be satisfied. Suppose that $F_* \in L_q(\pi;\cal Y)$ and $\|F_*\|_{L_q} \leq C_q < \infty$ for some $q \geq 2$. Denote $\Omega_0 = \left\{x \in E_X: \|F_*(x) \|_{\cY} \leq t \right\}$ and \[Q(t,\lambda) := \max \{t+ 2A\left\|F_{*}\right\|_{\beta} \lambda^{\frac{\be-\al}{2}},R\}.\] Then, for $0 \leq \gamma \leq 1$, $\tau \geq 1$, $\lambda > 0$ and $n \geq 8A^{2} \tau g_{\lambda} \lambda^{-\alpha}$, with probability $1 - 4e^{-\tau}-\tau_n$ with $\tau_n = 1-\left(1-\frac{C_q^q}{t^q}\right)^n$:
\begin{IEEEeqnarray}{rCl}
\left\|\left[\hat{C}_{\lambda} - C_{\lambda}\right]\right\|_{S_2\left([\mathcal{H}]_X^{\gamma},\mathcal{Y}\right)}^2 &\leq &\frac{1152\tau^2}{n\lambda^{\gamma}}\left(\sigma^{2} \mathcal{N}(\lambda)+\frac{\left\|F_{*}-\left[F_{\lambda}\right]\right\|_{L_{2}(\pi; \mathcal{Y})}^{2}A^{2}}{\lambda^{\alpha}} + \frac{2Q(t,\lambda)^2A^2}{n\lambda^{\alpha}}\right)\nonumber \\
&&+\frac{18}{\lambda^{\gamma}}\left(\sigma^{2} \mathcal{N}(\lambda)+\left\|F_{*}-\left[F_{\lambda}\right]\right\|_{L_{2}(\pi; \mathcal{Y})}^{2} \frac{A^{2}}{\lambda^{\alpha}}\right) \frac{C_q^q}{t^{q}},\nonumber
\end{IEEEeqnarray}
where $g_{\la}$ is defined in Theorem~\ref{thm:variance_bound}.
\end{theo}
\begin{proof}
The proof adopts the same strategy as in Theorem~\ref{thm:variance_bound}. The difference is that for Eq.~(\ref{eqn:var1}), instead of using Lemma~\ref{lma:empy_concen}, we use Lemma~\ref{lma:empy_concen_q} below.
\end{proof}

Lemma~\ref{lma:empy_concen} provides the upper bound under the assumption that $\|F_*\|_{L_{\infty}(\pi; \cY)} < \infty$. For $\beta \geq \alpha$ this assumption is automatically satisfied by the condition \eqref{asst:emb}, but for $\beta < \alpha$, this assumption remains crucial. However, as indicated in \cite{zhang2023optimality}, in the scalar-valued setting, when $\alpha - p < \be < \alpha$, this assumption can be replaced by an assumption of the form $\|F_*\|_{L_{q}(\pi; \cY)} < \infty$ for some $q\geq 2$. Below, we generalise their technique to the vector-valued setting.


\begin{lma}\label{lma:empy_concen_q}
Assume the conditions in Theorem~\ref{thm:variance_bound_q} holds. Then for $\tau \geq 1$, $\lambda > 0$ and $n \geq 1$ with probability over $1 - 2e^{-\tau}-\tau_n$ with $\tau_n = 1-\left(1-\frac{C_q^q}{t^q}\right)^n$,
\begin{equation} 
    \begin{aligned}
&\bigg\|\left(\hat{C}_{Y X}-C_{Y X}\left(C_{X X}+\lambda \operatorname{Id}_{\cH_X} \right)^{-1}(\hat{C}_{X X}+\lambda \operatorname{Id}_{\cH_X} )\right)\left(C_{X X}+\lambda \operatorname{Id}_{\cH_X} \right)^{-\frac{1}{2}}\bigg\|_{S_2(\mathcal{H}_X, \mathcal{Y})}^2 \\ &\leq \frac{384\tau^2}{n}\left(\sigma^{2} \mathcal{N}(\lambda)+\frac{\left\|F_{*}-\left[F_{\lambda}\right]\right\|_{L_{2}(\pi; \mathcal{Y})}^{2}A^{2}}{\lambda^{\alpha}} + \frac{2Q(t,\lambda)^2A^2}{n\lambda^{\alpha}}\right) +6\left(\sigma^{2} \mathcal{N}(\lambda)+\left\|F_{*}-\left[F_{\lambda}\right]\right\|_{L_{2}(\pi; \mathcal{Y})}^{2} \frac{A^{2}}{\lambda^{\alpha}}\right) \frac{C_q^q}{t^{q}}. \label{eq:var_q}
    \end{aligned}
\end{equation}  
\end{lma}

\begin{proof}
Denote $\Omega_0 := \left\{x \in E_X: \|F_*(x) \|_{\cY} \leq t \right\}$ and $\Omega_1 = E_X\backslash\Omega_0$. Since $\|F_*\|_{L_q} \leq C_q$, we have 
\begin{equation} 
\P(X \in \Omega_1) = \P(\|F_*(X)\|_{\cY} >t)\leq \frac{\mathbb{E}\left[\|F_*(X)\|^q_{\cY}\right]}{t^q} \leq \frac{C_q^q}{t^q}.
\end{equation}

Let $\xi_2(x,y)$ be defined as in Eq.~\eqref{eq:xi2}. As shown before in Eq.~\eqref{eq:noise_decom}, we have
$$
\left(\hat{C}_{Y X}-C_{Y X}\left(C_{X X}+\lambda \operatorname{Id}_{\cH_X} \right)^{-1}\left(\hat{C}_{X X}+\lambda \operatorname{Id}_{\cH_X} \right)\right)\left(C_{X X}+\lambda \operatorname{Id}_{\cH_X} \right)^{-\frac{1}{2}}=\frac{1}{n} \sum_{i=1}^{n}\left(\xi_{2}\left(x_{i}, y_{i}\right)-\mathbb{E} \xi_{2}\right).
$$
Decomposing $\xi_2 = \xi_2\mathbbm{1}_{x\in \Omega_0} + \xi_2\mathbbm{1}_{x\in \Omega_1} $, we have

\begin{IEEEeqnarray*}{rCl}
\left\| \frac{1}{n}\sum_{i=1}^{n}\xi_{2}\left(x_{i}, y_{i}\right)-\mathbb{E} \xi_{2} \right\|^2_{S_2(\mathcal{H}_X, \mathcal{Y})} \leq&& 3\underbrace{\left\| \frac{1}{n}\sum_{i=1}^{n}\xi_{2}\left(x_{i}, y_{i}\right)\mathbbm{1}_{x_i\in \Omega_0}-\mathbb{E}\left[ \xi_{2}\mathbbm{1}_{X \in \Omega_0}\right] \right\|^2_{S_2(\mathcal{H}_X, \mathcal{Y})} }_{I}  \\
&&+ 3\underbrace{\left\| \frac{1}{n}\sum_{i=1}^{n}\xi_{2}\left(x_{i}, y_{i}\right)\mathbbm{1}_{x_i\in \Omega_1} \right\|^2_{S_2(\mathcal{H}_X, \mathcal{Y})}}_{II} + 3\underbrace{\left\| \mathbb{E}\left[ \xi_{2}\mathbbm{1}_{X\in \Omega_1} \right]\right\|^2_{S_2(\mathcal{H}_X, \mathcal{Y})}
}_{III}.     
\end{IEEEeqnarray*}
For term I, we employ Proposition~\ref{prop:bounded_concen} and obtain that for any $\tau \geq 1$, with probability over $1 - 2e^{-\tau}$, 
\begin{equation} 
    \begin{aligned}
\bigg\|\frac{1}{n} \sum_{i=1}^{n}\left(\xi_{2}\left(x_{i}, y_{i}\right)\mathbbm{1}_{x_i\in \Omega_0}-\mathbb{E} \xi_{2}\mathbbm{1}_{x\in \Omega_0}\right)\bigg\|_{S_2(\mathcal{H}_X, \mathcal{Y})}^2 \leq \frac{64\tau^2}{n}\left(\sigma^{2} \mathcal{N}(\lambda)+\frac{\left\|F_{*}-\left[F_{\lambda}\right]\right\|_{L_{2}(\pi; \mathcal{Y})}^{2}A^{2}}{\lambda^{\alpha}} + \frac{2Q(t,\lambda)^2A^2}{n\lambda^{\alpha}}\right):=C_{I}. \nonumber
    \end{aligned}
\end{equation}   
For term II, we have
\begin{IEEEeqnarray*}{rCl}
\tau_n :=\P(II > C_I) &\leq & \P(\exists x_i,~\text{s.t.}~ x_i\in \Omega_1) = 1-\P(x_i\in\Omega_0, \forall i =[n])  \\ 
&=& 1-\P(X\in \Omega_0)^n\\
&=& 1-\P(\|F_*(X)\|_{\cY} \leq t)^n\\
&\leq & 1-\left(1-\frac{C_q^q}{t^q}\right)^n.
\end{IEEEeqnarray*}

For term III, we have 
\begin{IEEEeqnarray*}{rCl}
\left\| \mathbb{E} \xi_{2}\mathbbm{1}_{X\in \Omega_1} \right\|^2_{S_2(\mathcal{H}_X, \mathcal{Y})} &\leq & \mathbb{E}\left[\left\| \xi_{2} \right\|_{S_2(\mathcal{H}_X \mathcal{Y})}\mathbbm{1}_{X\in \Omega_1}\right]^2\\
&\leq &  \mathbb{E}\left[\left\|  \xi_{2}\right\|^2_{S_2(\mathcal{H}_X \mathcal{Y})}\right] \P(X \in \Omega_1)\\
&\leq & 2\left(\sigma^{2} \mathcal{N}(\lambda)+\left\|F_{*}-\left[F_{\lambda}\right]\right\|_{L_{2}(\pi; \mathcal{Y})}^{2} \frac{A^{2}}{\lambda^{\alpha}}\right) C_q^qt^{-q}.
\end{IEEEeqnarray*}
where for the last inequality, we used Eq.~\eqref{eq:xi2_m}.
\end{proof}



\begin{prop}\label{prop:bounded_concen}
Suppose that conditions~\eqref{asst:emb},~\eqref{asst:src} and~\eqref{asst:mom} are satisfied with $0 \leq \alpha \leq 1$ and $\beta \in(0,2]$. Let $\xi_2(x,y)$ be defined as in Eq.~\eqref{eq:xi2} and $\Omega_0 = \left\{x \in E_X: \|F_*(x) \|_{\cY} \leq t \right\}$. We have for any $\tau \geq 1$, with probability over $1 - 2e^{-\tau}$, 
\begin{equation} 
    \begin{aligned}
\bigg\|\frac{1}{n} \sum_{i=1}^{n}\left(\xi_{2}\left(x_{i}, y_{i}\right)\mathbbm{1}_{x_i\in \Omega_0}-\mathbb{E} \xi_{2}\mathbbm{1}_{X\in \Omega_0}\right)\bigg\|_{S_2(\mathcal{H}_X, \mathcal{Y})}^2 \leq \frac{64\tau^2}{n}\left(\sigma^{2} \mathcal{N}(\lambda)+\frac{\left\|F_{*}-\left[F_{\lambda}\right]\right\|_{L_{2}(\pi; \mathcal{Y})}^{2}A^{2}}{\lambda^{\alpha}} + \frac{2Q(t,\lambda)^2A^2}{n\lambda^{\alpha}}\right). \nonumber
    \end{aligned}
\end{equation}   
\end{prop}

\begin{proof}
We would like to apply the Bernstein's inequality from Theorem~\ref{theo:ope_con_steinwart} to obtain the desired bound. To this end, we bound the $m$-th moment of $\xi_2\mathbbm{1}_{X\in \Omega_0}$. Similar to Eq.~\eqref{eq:xi_decom}, we have  
\begin{IEEEeqnarray*}{rCl}
\mathbb{E}\left\|\xi_{2}\mathbbm{1}_{X\in \Omega_0}\right\|_{S_2(\mathcal{H}_X, \mathcal{Y})}^{m}=\int_{E_X}\left\|\left(C_{XX}+\lambda \operatorname{Id}_{\cH_X}\right)^{-1 / 2} \phi_X(x)\right\|_{\mathcal{H}_X}^{m} \mathbbm{1}_{X\in \Omega_0}\int_{\mathcal{Y}}\left\|y-F_{\lambda}(x)\right\|_{\mathcal{Y}}^{m} p(x, \mathrm{~d} y) \mathrm{d} \pi(x).     
\end{IEEEeqnarray*}
Apply Eq.~\eqref{eq:xi2_decom} to the above inner integral, we obtain 
\begin{equation} \label{eq:m_moment_2}
\begin{aligned}
\mathbb{E}\left\|\xi_{2}\mathbbm{1}_{X\in \Omega_0}\right\|_{S_2(\mathcal{H}_X, \mathcal{Y})}^{m} \leq  & \frac{1}{2} m !(2 R)^{m-2} 2 \sigma^{2} \int_{E_X}\left\|h_{x}\right\|_{\mathcal{H}_X}^{m}\mathbbm{1}_{X\in \Omega_0} \mathrm{~d} \pi(x) \\
&+2^{m-1} \int_{E_X}\left\|h_{x}\right\|_{\mathcal{H}_X}^{m}\left\|F_{*}(x)-F_{\lambda}(x)\right\|_{\mathcal{Y}}^{m}\mathbbm{1}_{X\in \Omega_0} \mathrm{~d} \pi(x).
\end{aligned}
\end{equation}
The first term in Eq.~\eqref{eq:m_moment_2} can be bounded using Lemma~\ref{theo:h_bound} as below:
$$
\begin{aligned}
\frac{1}{2} m !(2 R)^{m-2} 2 \sigma^{2}\int_{E_X}\left\|h_{x}\right\|_{\mathcal{H}_X}^{m} \mathbbm{1}_{x\in \Omega_0}\mathrm{~d} \pi(x) \leq \frac{1}{2} m !\left(\frac{2R A}{\lambda^{\alpha / 2}}\right)^{m-2} 2 \sigma^{2} \mathcal{N}(\lambda).
\end{aligned}
$$
For the second term, note that if $\beta \geq \alpha$, by Lemma~\ref{lma:F_l_bounded} Eq.~(\ref{eq:lemma4_1}),
$$
\left\|\left(F_{*}-[F_{\lambda}]\right)\mathbbm{1}_{X\in \Omega_0}\right\|_{L_{\infty}} \leq \left\|F_{*}-[F_{\lambda}]\right\|_{L_{\infty}} \leq \left(\left\|F_{*}\right\|_{L_\infty} + A\|F_*\|_{\beta}\right) \lambda^{\frac{\beta - \alpha}{2}} \leq  2A\|F_*\|_{\beta}\lambda^{\frac{\beta - \alpha}{2}},
$$
and if $\beta < \alpha$, by Lemma~\ref{lma:F_l_bounded} Eq.~(\ref{eq:lemma4_2}),
$$
\left\|\left(F_{*}-[F_{\lambda}]\right)\mathbbm{1}_{X\in \Omega_0}\right\|_{L_{\infty}} \leq \left\|F_{*}\mathbbm{1}_{X\in \Omega_0}\right\|_{L_{\infty}} + \left\|[F_{\lambda}]\right\|_{L_{\infty}} \leq t+ A\left\|F_{*}\right\|_{\beta} \lambda^{\frac{\be-\al}{2}}.
$$
Therefore, for all $\beta \in (0,2]$, 
$$
\left\|\left(F_{*}-[F_{\lambda}]\right)\mathbbm{1}_{X\in \Omega_0}\right\|_{L_{\infty}} \leq t+ 2A\left\|F_{*}\right\|_{\beta} \lambda^{\frac{\be-\al}{2}} \leq  Q(t,\la).
$$
We have, using Lemma~\ref{theo:h_bound} again,
\begin{equation}\label{eq:term2}
\begin{aligned}
& 2^{m-1} \int_{E_X}\left\|h_{x}\right\|_{\mathcal{H}_X}^{m}\left\|F_{*}(x)-F_{\lambda}(x)\right\|_{\mathcal{Y}}^{m}\mathbbm{1}_{x\in \Omega_0} \mathrm{~d} \pi(x) \\
\leq & \frac{1}{2}\left(\frac{2A}{\lambda^{\alpha / 2}}\right)^{m}Q(t,\lambda)^{m-2} \int_{E_X}\left\|F_{*}(x) - F_{\lambda}(x)\right\|_{\mathcal{Y}}^{2} \mathbbm{1}_{x\in \Omega_0}\mathrm{~d} \pi(x) \\
\leq & \frac{1}{2}m!\left(\frac{2Q(t,\lambda)A}{\lambda^{\alpha / 2}}\right)^{m-2} \left\|F_{*}-\left[F_{\lambda}\right]\right\|_{L_{2}(\pi; \mathcal{Y})}^{2}\frac{2A^2}{\lambda^{\alpha}}. \\
\end{aligned}
\end{equation}
Finally, we get
$$
\mathbb{E}\left\|\xi_{2}\right\|_{S_2(\mathcal{H}_X, \cY)}^{m} \leq \frac{1}{2} m !\left(\frac{2Q(t,\lambda)A}{\lambda^{\alpha / 2}}\right)^{m-2} 2\left(\sigma^{2} \mathcal{N}(\lambda)+\left\|F_{*}-\left[F_{\lambda}\right]\right\|_{L_{2}(\pi; \mathcal{Y})}^{2} \frac{A^{2}}{\lambda^{\alpha}}\right).
$$
An application of Bernstein's inequality from Theorem~\ref{theo:ope_con_steinwart} with $$
L=2Q(t,\lambda)A\lambda^{-\alpha / 2}, \qquad \sigma^{2}=2\left(\sigma^{2} \mathcal{N}(\lambda)+\left\|F_{*}-\left[F_{\lambda}\right]\right\|_{L_{2}(\pi; \mathcal{Y})}^{2}A^{2}\lambda^{-\alpha}\right)
$$ 
yields the bound. 
\end{proof}

\subsection{Learning Rates}
In this section, we establish the upper bound on the learning rate for Theorem~\ref{theo:upper_rate_bis}. To this end, we first look at Theorem~\ref{thm:variance_bound_q}, with probability over $1 - 4e^{-\tau}-\tau_n$ with $\tau_n = 1-\left(1-\frac{C_q^q}{t^q}\right)^n$,
\begin{equation*} 
    \begin{aligned}
\left\|\left[\hat{C}_{\lambda} - C_{\lambda}\right]\right\|_{S_2\left([\mathcal{H}]_X^{\gamma},\mathcal{Y}\right)}^2 &\leq \frac{1152\tau^2}{n\lambda^{\gamma}}\left(\sigma^{2} \mathcal{N}(\lambda)+\frac{\left\|F_{*}-\left[F_{\lambda}\right]\right\|_{L_{2}(\pi; \mathcal{Y})}^{2}A^{2}}{\lambda^{\alpha}} + \frac{2Q(t,\lambda)^2A^2}{n\lambda^{\alpha}}\right) \\
&+\frac{18}{\lambda^{\gamma}}\left(\sigma^{2} \mathcal{N}(\lambda)+\left\|F_{*}-\left[F_{\lambda}\right]\right\|_{L_{2}(\pi; \mathcal{Y})}^{2} \frac{A^{2}}{\lambda^{\alpha}}\right) \frac{C_q^q}{t^{q}}.
    \end{aligned}
\end{equation*} 
We first notice that if $t > n^{1/q}$, the second term on the r.h.s is \[\frac{18C_q^q}{n\lambda^{\gamma}}\left(\sigma^{2} \mathcal{N}(\lambda)+\left\|F_{*}-\left[F_{\lambda}\right]\right\|_{L_{2}(\pi; \mathcal{Y})}^{2} \frac{A^{2}}{\lambda^{\alpha}}\right).\] 
Moreover, by Bernouilli's inequality, $\tau_n \leq \frac{C_q^qn}{t^q}$. As such, given any $\tau \geq 1$, if $t > n^{1/q}$, there is a $n'_0 \geq 1$ such that $\tau_n \leq e^{-\tau}$ for all $n \geq n'_0$. We therefore have with probability greater than $1 - 5e^{-\tau}$, for all $n \geq n_0'$,
\begin{equation} 
    \begin{aligned}
\left\|\left[\hat{C}_{\lambda} - C_{\lambda}\right]\right\|_{S_2\left([\mathcal{H}]_X^{\gamma},\mathcal{Y}\right)}^2\leq \frac{c_0\tau^2}{n\lambda^{\gamma}}\left(\sigma^{2} \mathcal{N}(\lambda)+\frac{\left\|F_{*}-\left[F_{\lambda}\right]\right\|_{L_{2}(\pi; \mathcal{Y})}^{2}A^{2}}{\lambda^{\alpha}} + \frac{2Q(t,\lambda)^2A^2}{n\lambda^{\alpha}}\right). \label{eq:var_q_l}
    \end{aligned}
\end{equation} 
where $c_0 = \max\{1152,18C_q^q\}$. 
From now on, we denote $\lambda$ as $\lambda_n$ to indicate that $\lambda$ is a function of $n$ and we pick $n_1 \geq 1$ such that $\lambda_n \leq 1$ for all $n \geq n_1$. For $n \geq n_1$, we have,
\begin{IEEEeqnarray}{rCl}
Q(t,\lambda_n)^2 &=& \max\{R^2, \left(t+ 2A\left\|F_{*}\right\|_{\beta} \lambda_n^{\frac{\be-\al}{2}}\right)^2\} \nonumber \\ &\leq& \max\left\{R^2,2t^2+8A^2B^2 \lambda_n^{-(\alpha-\beta)}\right\} \nonumber \\ &\leq& \max\left\{R^2,8A^2B^2 \lambda_n^{-(\alpha-\beta)}\right\} + 2t^2 \nonumber \\
&\leq& C_0 (\lambda_n^{-(\alpha-\beta)_+} +t^2), \nonumber
\end{IEEEeqnarray}
where $C_0 := \max\left\{R^2, 8A^2B^2, 2\right\}$. As a result, Eq.~\eqref{eq:var_q_l} become
\begin{equation} 
    \begin{aligned}
\left\|\left[\hat{C}_{\lambda_n} - C_{\lambda_n}\right]\right\|_{S_2\left([\mathcal{H}]_X^{\gamma},\mathcal{Y}\right)}^2 &\leq \frac{c_0\tau^2}{n\lambda_n^{\gamma}}\left(\sigma^{2} \mathcal{N}(\lambda_n)+\frac{\left\|F_{*}-\left[F_{\lambda_n}\right]\right\|_{L_{2}(\pi; \mathcal{Y})}^{2}A^{2}}{\lambda_n^{\alpha}} + 2C_0A^2\frac{\lambda_n^{-(\alpha-\beta)_+}}{n\lambda_n^{\alpha}}\right)\\
& +2C_0A^2c_0\tau^2\frac{t^2}{n^2\lambda_n^{\alpha+\gamma}}. \label{eq:var_q_l1}
\end{aligned}
\end{equation}
Using Lemma~\ref{theo:h_bound} and Lemma~\ref{lma:bias_bound_gamma} with $\gamma=0$, we can see that the first term on the r.h.s coincides with Eq.~\eqref{var_B_3} up to some constants. Hence, the analysis in Section~\ref{sec:lr} carries on. In particular, the choice of $n_0 \geq 1$, such that $n \geq 8A^2 \tau g_{\lambda_n}\lambda_n^{-\alpha}$ for all $n \geq n_0$ remains the same as in Section~\ref{sec:lr} (see Eq.~\eqref{eq:a_l}). For $n \geq n_1$, using $\lambda_n \leq 1$ and $C_1 := D\sigma^2 + A^2B^2$,
$$
\begin{aligned}
\left\|[\hat{F}_{\lambda_n} - F_{\lambda_n}]\right\|_{\gamma}^2 &=\left\|[\hat{C}_{\lambda_n}-C_{\lambda_n}]\right\|_{S_2([\mathcal{H}]_X^{\gamma}, \mathcal{Y})}^2 \leq \frac{c_0\tau^2}{n\lambda_n^{\gamma}}\left(C_1\lambda_n^{-\max\{p,\alpha-\beta\}} + 2C_0A^2\frac{\lambda_n^{-(\alpha-\beta)_+}}{n\lambda_n^{\alpha}}\right) +2C_0A^2c_0\tau^2\frac{t^2}{n^2\lambda_n^{\alpha+\gamma}} \\ &\leq
\frac{c_0\tau^2}{n\lambda_n^{\gamma+\max\{p,\alpha-\beta\}}}\left(C_1 + 2C_0A^2\frac{1}{n\lambda_n^{\alpha+(\alpha-\beta)_+-\max\{p,\alpha-\beta\}}}\right) + 2C_0A^2c_0\tau^2\frac{t^2}{n^2\lambda_n^{\alpha+\gamma}}
\end{aligned}
$$
There is a constant $C_2 > 0$ such that for all $n \geq n_2$, for $n_2 := \max\{n'_0,n_0,n_1\}$,
\[\left\|[\hat{F}_{\lambda_n} - F_{\lambda_n}]\right\|_{\gamma}^2 \leq  \frac{c_0\tau^2}{n\lambda_n^{\gamma+\max\{p,\alpha-\beta\}}}\left(C_1 + 2C_0A^2C_2\right) + 2C_0A^2c_0\tau^2\frac{t^2}{n^2\lambda_n^{\alpha+\gamma}}.\] 
 Using the bias-variance splitting from Eq.~(\ref{eqn:risk_decom}) and Lemma~\ref{lma:new_bias}, let $C_3 := \max\{c_0\left(C_1 + 2A^2C_0C_2\right), 2C_0A^2c_0\}$, we have
\begin{IEEEeqnarray}{rCl}
\left\|[\hat{F}_{\lambda_n}]-F_*\right\|_{\gamma}^2 &\leq & 2\|C_{*}\|_{S_2([\mathcal{H}]_X^{\beta}, \mathcal{Y})}^2 \lambda_n^{\beta-\gamma} + 2C_3\tau^2\left(\frac{1}{n\lambda_n^{\gamma+\max\{p,\alpha-\beta\}}}+\frac{t^2}{n^2\lambda_n^{\alpha+\gamma}}\right)\nonumber\\
&\leq & \tau^2 \lambda_n^{\beta-\gamma}\left(2\|C_{*}\|_{S_2([\mathcal{H}]_X^{\beta}, \mathcal{Y})}^2+2C_3 \frac{1}{n\lambda_n^{\max\{\beta+p,\alpha\}}}\right)+ 2C_3\tau^2\frac{t^2}{n^2\lambda_n^{\alpha+\gamma}}, \nonumber
\end{IEEEeqnarray}
where we used $\tau \geq 1$. 

We now have two scenarios. We first assume that $\beta + p > \alpha$. If we choose $\lambda_n = n^{-1/(\beta + p)}$ there is some constant $J>0$ such that \[\left\|[\hat{F}_{\lambda_n}]-F_*\right\|_{\gamma}^2 \leq  \tau^2 J \left(\lambda_n^{\beta - \gamma} + \frac{t^2}{n^2\lambda_n^{\alpha+\gamma}}\right)\] for all $n \geq n_2$. In order for the second term to match the first term, we will need $\frac{t^2}{n^2\lambda_n^{\alpha}} \leq \lambda_n^{\beta}$ when $\lambda_n = n^{-\frac{1}{\beta+p}}$. This amounts to require that \[t \leq n^{\frac{\beta+2p-\alpha}{2(\beta+p)}}.\] Recall that we require $t > n^{1/q}$, combining with the above, we need \[q > \frac{2(\beta+p)}{\beta +2p-\alpha } =: q_0.\] 
\changebis{Therefore we need $F_* \in  L_{q}(\pi;\cY)$ with $q > q_0$. By Theorem~\ref{th:Lq_embedding}, $F_* \in L_{q_1}(\pi;\cY)$ with $q_1 := \frac{2\alpha}{\alpha-\beta}$, and since $q_1 > q_0$, the requirement is automatically satisfied.}

Secondly, we assume that $\beta +p \leq \alpha$. If we choose $\lambda_n = \left(n/\log^{\theta} n \right)^{-\frac{1}{\alpha}}$ with $\theta >1$, we obtain, for some constant $J >0$ that  \[\left\|[\hat{F}_{\lambda_n}]-F_*\right\|_{\gamma}^2 \leq  \tau^2 J \left(\lambda_n^{\beta - \gamma} + \frac{t^2}{n^2\lambda_n^{\alpha+\gamma}}\right)\] for all $n \geq n_2$. In order for the second term to match the first term, we will need $\frac{t^2}{n^2\lambda_n^{\alpha}} \leq \lambda_n^{\beta}$ when $\lambda_n = \left(n/\log^{\theta} n \right)^{-\frac{1}{\alpha}}$. This amounts to require that \[t \leq n \left(\frac{\log^{\theta} n}{n} \right)^{\frac{\alpha+\beta}{2\alpha}}.\] Merging with the constraint $t > n^{1/q}$, we require \[q \geq \frac{2\alpha}{\alpha -\beta } = q_1.\]
\changebis{Therefore we need $F_* \in  L_{q_1}(\pi;\cY)$ which is automatically satisfied by Theorem~\ref{th:Lq_embedding}. This concludes the proof of Theorem~\ref{theo:upper_rate_bis}.}

\subsection{Proof of Theorem~\ref{th:Lq_embedding}}
\changebis{
    Let us drop the subscript $\alpha,\beta$ and use $q:= q_{\alpha,\beta}$. Theorem 5 in \cite{zhang2023optimalityspectral} proves the scalar-valued $L_q-$embedding property: the inclusion $[\cH]_X^{\beta} \hookrightarrow L_q(\pi,\R)$ is bounded with operator norm denoted as $M$.

    Let $F \in[\mathcal{G}]^\beta$ be given by the representation $F(\cdot)=\sum_{i=1}^N f_i(\cdot) y_i \simeq \sum_{i=1}^N y_i \otimes f_i$ with $f_i \in[\mathcal{H}]^\beta$ and $y_i \in \mathcal{Y}$. We may assume without loss of generality that the $y_i$ are orthogonal in $\cY$ (to obtain such representation we can apply the Gram-Schmidt process to any finite-rank expansion). Note in particular that functions of this form are dense in $[\mathcal{G}]^\beta$ by construction. We prove the claim by manually bounding the norm of $F$ in $L_q(\pi ; \mathbb{R})$ by a constant multiple of its norm in $[\mathcal{G}]^\beta$. We have
    
    $$
        \begin{aligned}
    \|F\|_{L_q(\pi ; \mathcal{Y})}^2 & =\left(\int\left\|\sum_{i=1}^N f_i(x) y_i\right\|_{\mathcal{Y}}^q \mathrm{~d} \pi(x)\right)^{2 / q} & & \text { (definition of $\|\cdot\|_{L_q(\pi ; \mathcal{Y})}$ ) } \\
    &=\left(\int\left(\sum_{i,j=1}^N f_i(x)f_j(x)\langle y_i, y_j \rangle_{\mathcal{Y}}\right)^{q/2} \mathrm{~d} \pi(x)\right)^{2 / q} & &  \\
    &= \left(\int\left(\sum_{i=1}^Nf_i(x)^2\left\|y_i\right\|_{\mathcal{Y}}^2\right)^{q/2} \mathrm{~d} \pi(x)\right)^{2 / q} & & \text { ($y_i$ orthogonal in $\mathcal{Y}$) } \\
    & =\left\|\sum_{i=1}^N f_i(\cdot)^2\left\| y_i\right\|_{\mathcal{Y}}^2\right\|_{L_{q/2}(\pi ; \mathbb{R})} & & \text { (definition of $\|\cdot\|_{L_{q/2}(\pi ; \R)})$} \\
    &\leq \sum_{i=1}^N \left\|y_i\right\|_{\mathcal{Y}}^2 \left\| f_i^2 \right\|_{L_{q/2}(\pi ; \mathbb{R})} & & \text { (triangle inequality) } \\
    &= \sum_{i=1}^N \left\|y_i\right\|_{\mathcal{Y}}^2 \left\| f_i \right\|_{L_q(\pi ; \mathbb{R})}^2  \\
    & \leq M^2\sum_{i=1}^N \left\| y_i\right\|_{\mathcal{Y}}^2\left\|f_i\right\|_{[\cH]_X^\beta}^2  & & \text { (scalar-valued $L_q$-embedding property) } \\
    & = M^2\|F\|_{[\cG]^\beta}^2 & &  \text { (definition of $[\cG]^\beta$ and $y_i$ orthogonal in $\mathcal{Y}$). }
    \end{aligned}
    $$
    The standard denseness argument proves this property for all $F \in [\mathcal{G}]^\beta$.
    }

\section{Proof of Theorem \ref{theo:lower_bound}}
\label{sec:proof_lower_bound}

The proof of the lower bound is performed
by projecting the infinite-dimensional response variable $Y$ onto a one-dimensional subspace of $\mathcal{Y}$ and applying arguments from the real-valued learning scenario.

We start by noticing that for any $F \in L_2(\pi; \mathcal{Y})$ and $a \in \mathcal{Y}$,
\begin{equation} \label{eq:bound_scalar}  
\begin{aligned}
    \int_{E_X} \left(\langle F(x), a \rangle_{\mathcal{Y}} - \langle F_*(x), a \rangle_{\mathcal{Y}} \right)^2 d\pi(x) &\leq \int_{E_X} \|F(x) - F_{*}(x)\|_{\mathcal{Y}}^2 \|a\|^2_{\mathcal{Y}} d\pi(x) \\ &=  \|a\|^2_{\mathcal{Y}}\|F - F_{*}\|^2_{L_2(\pi; \mathcal{Y})}.
\end{aligned}
\end{equation}
Moreover, by Lemma~\ref{lma:gamma_norm_reduction}, the inequality holds for general $\gamma$-norm (which implies the previous equation, setting $\gamma=0$), 
\begin{IEEEeqnarray}{rCl}
\|\langle F(.), a\rangle_{\mathcal{Y}}-\langle F_*(.), a \rangle_{\mathcal{Y}}\|_{\gamma} \leq \|a\|_{\mathcal{Y}}\|F-F_*\|_{\gamma}. \label{eq:bound_scalar_gamma}
\end{IEEEeqnarray}

\begin{lma}\label{lma:gamma_norm_reduction}
Let $\gamma \geq 0$, for any $F \in [\mathcal{G}]^{\gamma}$ and $a \in \cY$, we have \[\|\langle F(.), a\rangle_{\mathcal{Y}}\|_{\gamma} \leq \|a\|_{\mathcal{Y}} \|F\|_{\gamma}.\]
\end{lma}

\begin{proof}
The case where $\gamma = 0$ is already proved in Eq.~(\ref{eq:bound_scalar}). We now let $\gamma > 0$. Recall $\{d_j\}_{j\in J}$ and $\{\mu_i^{\gamma/2}[e_i]\}_{i \in I}$ are the orthonormal basis of $\mathcal{Y}$ and $[\mathcal{H}]_{X}^{\gamma}$, since $F \in [\mathcal{G}]^{\gamma}$, we can write $F$ as \[F = \sum_{i\in I,j \in J} a_{ij} d_j \mu_i^{\gamma/2}[e_i].\] Therefore, we have \[\langle F(\cdot), a\rangle_{\mathcal{Y}} = \sum_{i\in I,j \in J} a_{ij}\langle d_j, a\rangle_{\mathcal{Y}}\mu_i^{\gamma/2}[e_i](\cdot).\] $\langle F(\cdot), a \rangle_{\mathcal{Y}}$ is an element of $[\mathcal{H}]_X^{\gamma}$ as 
\begin{IEEEeqnarray*}{rCl}
\|\langle F(\cdot), a \rangle_{\mathcal{Y}}\|_{\gamma}^2 &=& \sum_{i \in I} \left(\sum_{j \in J} a_{ij} \langle d_j, a\rangle_{\mathcal{Y}}\right)^2\\
& \leq & \sum_{i\in I} \sum_{j \in J} a_{ij}^2 \sum_{j \in J} \langle d_j, a\rangle_{\mathcal{Y}}^2 \\
&=& \|a\|_{\mathcal{Y}}^2 \sum_{i\in I,j \in J}a_{ij}^2 \\
&=& \|a\|_{\mathcal{Y}}^2 \|F\|_{\gamma}^2 < +\infty,
\end{IEEEeqnarray*}
where for the second step, we used Cauchy-Schwartz inequality and for the third step Parseval's identity.
\end{proof}

We now express the l.h.s as the risk of a scalar-valued regression. Consider a distribution $P$ on $E_X \times \mathcal{Y}$ that factorizes as $P(x,y) = p(y \mid x)\pi(x)$ for all $(x,y) \in E_X \times \mathcal{Y}$. For all $x \in E_X$, $p(\cdot \mid x)$ defines a probability distribution on $\mathcal{Y}$. We fix an element $a \in \mathcal{Y}$, $a \ne 0$ and define $\mathcal{Y}^a := \{y_a \in \mathbb{R} \mid y_a = \langle y,a \rangle_{\mathcal{Y}}, y \in \mathcal{Y}\}$. Since $\mathcal{Y}$ is a Hilbert space hence a vector space, we have $\mathcal{Y}^a = \mathbb{R}$. Consider the joint distribution $P_a$ on $E_X \times \mathbb{R}$ such that

\begin{equation}
\begin{aligned}
    p_a(. \mid x) &:= \left(\langle \cdot, a \rangle_{\mathcal{Y}}\right)_{\#}p(\cdot \mid x) \\
    P_a(x, y_a) &:= p_a(y_a \mid x)\pi(x), \quad (x,y_a) \in E_X \times \mathbb{R}
\end{aligned} \label{eq:reduction_distrib}
\end{equation}

where $\#$ denotes the push-forward operation. For a dataset $D = \{(x_i, y_i)\}_{i=1}^n \in \left(E_X \times \mathcal{Y} \right)^n$ where the data are i.i.d from $P$, the dataset $D_a = \{(x_i, y_{a.i})\}_{i=1}^n \in \left(E_X \times \mathbb{R} \right)^n$ where $y_{a.i} := \langle y_i,a \rangle_{\mathcal{Y}}$ for all $i=1,\ldots,n$ is i.i.d from $P_a$. Note that $p_a(\cdot \mid x)$ is a probability distribution on $\mathbb{R}$ for all $x$ supported by $\pi$. By definition of the push-forward operator, the Bayes predictor associated to the joint distribution $P_a$ is 
\begin{equation} \label{eq:bayes_scalar}
\begin{aligned}
    f_{a,*}(x) = \int_{\mathbb{R}} y_a dp_a(y_a \mid x) &= \int_{\mathcal{Y}} \langle y, a \rangle_{\mathcal{Y}} dp(y \mid x) \\ &= \left\langle \int_{\mathcal{Y}} y dp(y \mid x), a  \right\rangle_{\mathcal{Y}} \\ &= \langle F_*(x), a \rangle_{\mathcal{Y}}
\end{aligned}
\end{equation}
where $F_*$ is the $\mathcal{Y}$-valued Bayes predictor associated to $P$. Therefore plugging Eq.~(\ref{eq:bayes_scalar}) in Eq.~(\ref{eq:bound_scalar_gamma}) we obtain that for any learning method $D \rightarrow \hat{F}_{D} \in \mathcal{Y}^{E_X}$

\begin{equation}
\|\hat{F}_D - F_{*}\|_{\gamma} \geq \|a\|_{\mathcal{Y}}^{-1}\|\hat{f}_{D_a} - f_{a.*}\|_{\gamma} \label{eq:bound_scalar_bis}
\end{equation}
where $\hat{f}_{D_a}(\cdot) := \langle \hat{F}_D(\cdot), a \rangle_{\mathcal{Y}}$. The r.h.s is the error measured in (scalar) $\gamma$-norm of the learning method $D_a \rightarrow \hat{f}_{D_a} \in \mathbb{R}^{E_X}$ on the scalar-regression learning problem associated to $D_a$. 

In what follows we fix $\{d_n\}_{n \geq 1}$ an orthonormal basis of $\mathcal{Y}$ and take \textbf{$a=d_1$}.

To derive a lower bound on the r.h.s in Eq.~(\ref{eq:bound_scalar_bis}), the strategy is to define a conditional distribution $p_a(. \mid x)$ on $\mathbb{R}$, $x\in E_X$, that is difficult to learn. We offer to use Gaussian conditional distributions as in \cite{fischer2020sobolev}. The additional difficulty in our setting is to show the existence of conditional distributions $p(. \mid x)$ on $\mathcal{Y}$, $x\in E_X$ such that Eq.~(\ref{eq:reduction_distrib}) holds and such that $F_* \in [\mathcal{G}]^{\beta}$ and \eqref{asst:mom} are satisfied. We make use of the following Lemma that corresponds to Lemma 19, Lemma 23 and Equation (55) in \cite{fischer2020sobolev}.

\begin{lma} \label{theo:function_construction} Let $k_X$ be a kernel on $E_X$ such that assumptions \ref{assump:separable} to \ref{assump:bounded} hold and $\pi$ be a probability distribution on $E_X$ such that \eqref{asst:evd+} is satisfied for some $0<p \leq 1$. Then, for all parameters $0< \beta \leq 2, 0 \leq \gamma \leq 1$ with $\gamma < \beta$ and all constant $B>0$, there exist constants $0<\epsilon_{0} \leq 1$ and $L_0, L>0$ such that the following statement is satisfied: for all $0<\epsilon \leq \epsilon_{0}$ there is an $M_{\epsilon} \geq 1$ with
\begin{IEEEeqnarray}{rCl}
2^{L \epsilon^{-u}} \leq M_{\epsilon} \leq 2^{3 L \epsilon^{-u}} \nonumber 
\end{IEEEeqnarray}
where $u:=\frac{p}{\beta - \gamma}$, and functions $f_0, \ldots, f_{M_{\epsilon}}$ such that $f_{i} \in [\mathcal{H}]_X^{\beta}$, $\|f_{i}\|_{\beta} \leq B$, and
\begin{IEEEeqnarray}{rCl}
&& \left\|f_{i}-f_{j}\right\|^{2}_{\gamma} \geq 4\epsilon \nonumber\\ &&\left\|f_{i}-f_{j}\right\|^{2}_{L_2(\pi)} \leq 32L_0^{\gamma}\epsilon m^{-\gamma/p},\nonumber
\end{IEEEeqnarray}
for all $i, j \in\left\{0, \ldots, M_{\varepsilon}\right\}$ with $i \neq j$ where $m \leq U\epsilon^{-u}$ for some constant $U > 0$.

\end{lma}

Recall that the Kullback-Leibler divergence of two probability measures $P_1, P_2$ on some measurable space $(\Omega, \mathcal{A})$ is given by
$$
KL\left(P_1, P_2\right):=\int_{\Omega} \log \left(\frac{\mathrm{d} P_1}{\mathrm{~d} P_2}\right) \mathrm{d} P_1
$$
if $P_1 \ll P_2$ and otherwise $KL\left(P_1, P_2\right):=\infty$.

Exploiting Lemma~\ref{theo:function_construction}, given $0<\beta\leq 2$ and $\sigma, R, B>0$, we build distributions $P_{d_1,i}$ on $E_X \times \mathbb{R}$ and $P_{i}$ on $E_X \times \mathcal{Y}$ for $i \in\left\{1, \ldots, M_{\varepsilon}\right\}$ such that, 
\begin{itemize}
    \item Step 1. $f_i = f^{d_1}_{*,i}$ where $f^{d_1}_{*,i}$ is the Bayes predictor associated to $P_{d_1,i}$
    \item Step 2. $KL(P_{d_1,i},P_{d_1,j})$ is upper-bounded by a linear function of $\left\|f_{i}-f_{j}\right\|^{2}_{L_2(\pi)}$
    \item Step 3. Eq.~(\ref{eq:reduction_distrib}) holds, i.e. $P_{d_1,i}$ is related to $P_i$ through the projection along the line with direction $d_1$
    \item Step 4. The Bayes predictor $F_{*,i}$ associated to $P_{i}$ is in $[\mathcal{G}]^{\beta}$ and $\|F_*\|_{\beta} \leq B$
    \item Step 5. $Y \in L_2(\pi_i,\Omega,\mathbb{P},\cal Y)$ 
    \item Step 6. \eqref{asst:mom} is satisfied with $P_i$ with parameters $\sigma, R$.
\end{itemize}

\textbf{Gaussian conditional distributions as in \cite{blanchard2018optimal, fischer2020sobolev}.} Take $\bar{\sigma} = \min(\sigma, R)$. For all $i=1,\ldots,M_{\epsilon}$ we define the joint distribution $P_{d_1,i}(x,y) = p_{d_1,i}(y \mid x) \pi(x)$ where

\begin{equation*} 
p_{d_1,i}(\cdot \mid x) = \mathcal{N}\left(f_{i}(x), \bar\sigma^2\right) 
\end{equation*} 
is a univariate Gaussian distribution. 

\paragraph{Step 1.} We automatically get $f_i = f^{d_1}_{*,i}$.

\paragraph{Step 2.} The Kullback-Leibler divergence satisfies
\begin{IEEEeqnarray*}{rCl}
KL(P_{a,i},P_{a,j})  &=& \int_{E_X} KL(p_{a,i}(\cdot \mid x),p_{a,j}(\cdot \mid x)) d\pi(x)\\
& =& \frac{1}{2\bar\sigma^2}\int_{E_X} (f_i(x) - f_j(x))^2 d\pi(x)\\
&=& \frac{1}{2\bar\sigma^2}\|f_{i} - f_{j}\|_{L_2(\pi)}^2
\end{IEEEeqnarray*}

\paragraph{Step 3.} We consider the map $y_i:E_X \to \mathcal{Y}, x \mapsto f_i(x)d_1$ such that for all $x \in E_X$, $\langle y_i(x), d_1 \rangle_{\mathcal{Y}} = f_i(x)$, then we build a Gaussian measure on $\mathcal{Y}$ as follows, we fix $x \in E_X$ and consider $Z$ a univariate Gaussian random variable, then
$$X(\omega) = y_i(x) + \bar{\sigma} Z(\omega)d_1, \qquad \omega \in \Omega$$
is such that (according to Definition \ref{def:inf_gau})  $X \sim \mathcal{N}_{\mathcal{Y}}(y_i(x), \bar{\sigma}^2 d_1 \otimes d_1)$. Therefore, we pick
\begin{equation*} 
p_i(\cdot \mid x) = \mathcal{N}_{\mathcal{Y}}\left(y_{i}(x), \bar\sigma^2d_1 \otimes d_1 \right)
\end{equation*}
and clearly $\langle X, a \rangle = f_i(x) + \bar\sigma Z $ so Eq.~(\ref{eq:reduction_distrib}) holds.

\paragraph{Step 4.} Note that $F_{*,i}(\cdot) = y_i(\cdot) = f_i(\cdot)d_1$, therefore we have 
$$
\|F_{*,i}\|_{\beta} = \|f_i\|_{\beta}\|d_1\|_{\mathcal{Y}} = \|f_i\|_{\beta} \leq B 
$$

\paragraph{Step 5.} Write $\Sigma := \bar\sigma^2d_1 \otimes d_1,$
\begin{IEEEeqnarray}{rCl}
&&\int_{E_X \times \mathcal{Y}} \|y\|_{\mathcal{Y}}^2 p_i(x,dy)\pi(dx) \nonumber \\
&=& \int_{E_X}\mathbb{V}(Y \sim \mathcal{N}\left(0, \Sigma \right)) + \|y_i(x)\|_{\mathcal{Y}}^2\pi(dx) \nonumber \\ &=& \mathbb{V}(Y \sim \mathcal{N}\left(0, \Sigma \right)) + \|f_i\|_{L_2(\pi)}^2 \nonumber \\ &=& \bar\sigma^2 + \|f_i\|_{L_2(\pi)}^2 < \infty, \nonumber
\end{IEEEeqnarray} 
since 
\begin{IEEEeqnarray}{rCl}
\mathbb{V}(Y \sim \mathcal{N}\left(0, \Sigma \right)) = \sum_{n \geq 1} \langle \Sigma d_n, d_n \rangle = \bar\sigma^2. \nonumber
\end{IEEEeqnarray}

\paragraph{Step 6.} We now show that $P_i$ satisfies \eqref{asst:mom} with parameters $\sigma = R = \bar{\sigma}$. We recall that for all $x \in E_X$, we picked $p_i(\cdot \mid x) = \mathcal{N}_{\mathcal{Y}}\left(y_{i}(x), \bar\sigma^2d_1 \otimes d_1 \right)$. We consider the centered random variable $X(\omega) = \bar{\sigma} Z(\omega)d_1$ and \eqref{asst:mom} amounts to show $\mathbb{E}\left[\|X\|^q_{\mathcal{Y}} \right] \leq \frac{1}{2}q!(\bar\sigma)^q$. But almost surely we have 
\begin{equation*}
    \|X\|^q_{\mathcal{Y}} = \left(\sum_{n \geq 1} |\langle Y, d_n \rangle_{\mathcal{Y}}|^2\right)^{q/2} = \bar{\sigma}^q |Z|^q
\end{equation*}
which brings us back to the univariate case that it easy to demonstrate (see for example Lemma 21 by \citealp{fine2002efficient}).

\paragraph{Proof of Theorem~\ref{theo:lower_bound}}Combining those 6 points with the proof of Lemma 19 and Theorem 2 \cite{fischer2020sobolev} gives us Theorem~\ref{theo:lower_bound}.

\section{Interpolation theory and proof of Theorem~\ref{th:interpolation_vector}} 
\label{sec:interpolation_vector}

In the next two results, we consider $H_1, H_2$, two infinite dimensional separable Hilbert spaces such that $H_2 \hookrightarrow H_1$ and such that the inclusion operator $\cI: H_2 \to H_1$ performing the change of norms $\cI x = x$ for $x \in H_2$, is compact. By the spectral theorem (see e.g., \cite{steinwart2012mercer} Theorem A.3), $\cI$ admits a decomposition 
$$
    \cI = \sum_{i \in I} \rho_i \cI(h_i) \langle \sqrt{\rho_i} h_i, \cdot \rangle_{H_1},
$$
where $\{\sqrt{\rho_i}h_i\}_{i \in I}$ is an ONB of $\operatorname{ker}(\cI)^{\perp}$, $\{\cI(h_i)\}_{i \in I}$ is an ONB of $\overline{\operatorname{ran}(\cI)}$, and $\rho_1 \geq \rho_2 \geq \cdots > 0$ is a positive sequence of singular values, converging to $0$. For any $f \in H_1$, we define $a_i := \langle f, \cI(h_i) \rangle_{H_1}$, $i \in I$.

We recall that the symbol $\cong$ means that sets coincide and the corresponding norms are equivalent. For details on the theory of interpolation spaces of the real method, see \cite{triebel1995interpolation}.


\begin{prop}\label{prop:mercer_scovel}
    Let $H_1, H_2$ be two separable Hilbert spaces such that $H_2 \hookrightarrow H_1$ and such that the inclusion operator $\cI: H_2 \to H_1$ is compact.  Then, for all $\theta \in (0,1)$,
    $$
    f \in [H_1, \cI(H_2)]_{\theta,2} \iff (a_i)_{i \in I} \in  \ell_2\left(\rho^{-\theta}\right) := \left\{ (a_i)_{i \in I} \in \ell_2(I) \bigg\mid  \|(a_i)\|_{\ell_2\left(\rho^{-\theta}\right)}^2 := \sum_{i \in I} \frac{a_{i}^2}{\rho_i^{\theta}} < +\infty \right\}.
    $$
    Furthermore, there exist constants $c,C > 0$, such that for all $f \in H_1$,
    $$
    c\|(a_i)\|_{\ell_2\left(\rho^{-\theta}\right)} \leq \|f\|_{[H_1, \cI(H_2)]_{\theta,2}} \leq C \|(a_i)\|_{\ell_2\left(\rho^{-\theta}\right)}.
    $$
\end{prop}

\begin{proof}
    The result is obtained as a direct consequence of the proof of \cite{steinwart2012mercer} Theorem 4.6, which uses the machinery of the so-called $K$\textit{-functional} and focuses on a RKHS $H_2$ compactly embedded into $H_1 = L_2(\pi)$. 
\end{proof}

\begin{theo}\label{th:tensor_interpolation}
    Let $H_1, H_2$ and $E$ be three separable Hilbert spaces such that $H_2 \hookrightarrow H_1$ and such that the inclusion operator $\cI: H_2 \to H_1$ is compact. 
    Then, for all $\theta \in (0,1)$,
    $$
    [E \otimes H_1, E \otimes \cI(H_2)]_{\theta,2} \cong E \otimes [H_1, \cI(H_2)]_{\theta,2}
    $$
\end{theo}

\begin{proof} Using the notations of Proposition~\ref{prop:mercer_scovel}, we have $[H_1, \cI(H_2)]_{\theta,2} \cong \ell_2\left(\rho^{-\theta}\right)$. Let $\{h_\ell\}_{\ell \in L}$ be an orthonormal basis for $E$, then 
$$
E \otimes [H_1, \cI(H_2)]_{\theta,2} \cong \ell_2(L \times \rho^{-\theta}) := \left\{(a_{i,\ell})_{i \in I, \ell \in L} \in \ell_2(I \times L) \bigg\mid \sum_{i \in I, \ell \in L} \frac{a_{i,\ell}^2}{\rho_i^{\theta}} < +\infty \right\}.
$$
To conclude, we require $[E \otimes H_1, E \otimes \cI(H_2)]_{\theta,2} \cong \ell_2(L \times \rho^{-\theta})$. This can be obtained again by following the steps of \cite{steinwart2012mercer} Theorem 4.6 and introducing the orthonormal basis $\{h_\ell\}_{\ell \in L}$ for $E$. 
\end{proof}

\begin{proof}[Proof of Theorem~\ref{th:interpolation_vector}] Recall that $I_{\pi}: \cH_X \to L_2(\pi)$ is a compact inclusion. Therefore, using that $L_2(\pi; \cY) \simeq \cY \otimes L_2(\pi)$ and $[\cG]^1 \simeq \cY \otimes I_{\pi}\left(\cH_X\right)$, by Theorem~\ref{th:tensor_interpolation} and Remark~\ref{rem:interpolation_scalar}, we have
$$
 \left[L_2(\pi; \cY),[\cG]^1\right]_{\alpha, 2} \simeq \left[\cY \otimes L_2(\pi), \cY \otimes I_{\pi}\left(\cH_X\right)\right]_{\alpha, 2} \cong \cY \otimes \left[L_2(\pi), I_{\pi}\left(\cH_X\right)\right]_{\alpha, 2} \cong \cY \otimes [\cH]_X^{\alpha} \simeq [\cG]^{\alpha}.
$$
\end{proof}

\section{Proofs of Section~\ref{sec:sobolev}}
\label{sec:proofs_sobolev}


We gather technical results related to Sobolev spaces. To this end, in the rest of this section we assume that $E_{X} \subseteq \Rd$ is a bounded domain with smooth boundary equipped with the Lebesgue measure $\mu$ and denote $L_2(E_{X}):= L_2(E_{X},\mu)$ as the corresponding $L_2$ space. For $s>0$, we denote $W^{s,2}(E_{X})$ as the (fractional) Sobolev space with smoothness $s$ \citep[Definition 7.32]{adams2003sobolev}. Note that $W^{s,2}(E_{X})$ is a subset of $L_2(E_{X})$ and therefore not a space of functions, however we have the following well-known Sobolev embedding theorem. Let $C_0(E_{X})$ be the space of bounded and continuous functions equipped with the norm $\|f\|_{\infty}:= \sup_{x \in E_{X}}|f(x)|$, $f \in C_0(E_{X})$.

\begin{theo}[Sobolev embedding theorem, \cite{adams2003sobolev} Theorem 7.34 (c)] \label{th:embedding} If $s > d/2$, for each $f \in W^{s,2}(E_{X})$, there exists a unique function in $C_0(E_{X})$, denoted $\bar{f}$ such that $f=\bar{f}$ $\mu-$almost everywhere. Furthermore, there is a constant $C_{\infty} \geq 0$ such that for all $f \in W^{s,2}(E_{X})$,
$$
\|\bar{f}\|_{\infty} \leq C_{\infty}\|f\|_{W^{s,2}(E_{X})}.
$$  
\end{theo}

In short, if $s > d/2$, $W^{s,2}(E_{X}) \hookrightarrow C_0(E_{X})$. As a consequence for $s>d/2$, we can build a RKHS from $W^{s,2}(E_{X})$.

\begin{theo} \label{theo:sobolev_rkhs}
    For $s > d/2$, define 
    $$
    \bar{W}^{s,2}(E_{X}) := \{\bar{f} \in C_0(E_{X}) : [\bar{f}]_{\mu} \in W^{s,2}(E_{X})\}
    $$
    equipped with the norm $\|\bar{f}\|_{\bar{W}^{s,2}(E_{X})}:= \|[\bar{f}]_{\mu}\|_{W^{s,2}(E_{X})}$. $\bar{W}^{s,2}(E_{X})$ is a separable RKHS (with respect to a kernel $k_s$) that we call the \textit{Sobolev RKHS}. Furthermore, $k_s$ is bounded and measurable. Therefore assumptions \ref{assump:separable} to \ref{assump:bounded} are satisfied for $k_s$ and $\mu$. 
\end{theo}

\begin{proof}
    For any $x \in E_{X}$ and $\bar{f} \in \bar{W}^{s,2}(E_{X})$, by the Sobolev embedding theorem, 
    $$
    |\bar{f}(x)| \leq \|\bar{f}\|_{\infty} \leq C_{\infty}\|[\bar{f}]_{\mu}\|_{W^{s,2}(E_{X})} = \|\bar{f}\|_{\bar{W}^{s,2}(E_{X})}.
    $$
    Therefore, the evaluation functional is continuous, proving that $\bar{W}^{s,2}(E_{X})$ is a RKHS. $k_s$ is bounded by \citep[Lemma 4.23]{steinwart2008support} and measurable by \citep[Lemma 4.24]{steinwart2008support}.
\end{proof}

We now characterise \eqref{asst:emb} and  \eqref{asst:evd} for $\bar{W}^{s,2}(E_{X})$.

\begin{prop} \label{prop:sobolev_alpha_p}
    For $s > d/2$ and $\cH_X = \bar{W}^{s,2}(E_{X})$, \eqref{asst:evd} is satisfied with $p = d/(2s)$ and \eqref{asst:emb} is satisfied for any $\alpha \in (p,1].$
\end{prop}

\begin{proof}
    For \eqref{asst:evd} see \cite{edmunds1996function}. For $\alpha \in (0,1]$, it is shown in \cite{fischer2020sobolev} Eq.~(14) that $[\bar{W}^{s,2}(E_{X})]^{\alpha}_{\mu} \simeq W^{\alpha s,2}(E_{X})$. Hence by Theorem~\ref{th:embedding}, if $\alpha s > d/2$,
    $$
    [\bar{W}^{s,2}(E_{X})]^{\alpha}_{\mu} \simeq W^{\alpha s,2}(E_{X}) \hookrightarrow C_0(E_{X}) \hookrightarrow L_{\infty}(\mu).
    $$
\end{proof}

We thank Haobo Zhang for bringing to our attention a sketch of proof for the next result. Up to our knowledge this result was only proved with $\theta \in (0,1]$.

\begin{prop} \label{prop:sobolev_beta}
    For $s > d/2$ and $\cH_X = \bar{W}^{s,2}(E_{X})$, for all $\theta > 0$, $[\bar{W}^{s,2}(E_{X})]^{\theta}_{\mu} \simeq W^{\theta s,2}(E_{X})$.
\end{prop}

\begin{proof}
    For $\theta \in (0,1]$, see \cite{fischer2020sobolev} Eq.~(14). For $\theta > 1$, since $\theta^{-1} \in (0,1)$, we have
    $$
    [\bar{W}^{s\theta,2}(E_{X})]^{\theta^{-1}}_{\mu} = W^{s,2}(E_{X}).
    $$
    Since $s\theta > d/2$, by Theorem~\ref{theo:sobolev_rkhs}, $\bar{W}^{s\theta,2}(E_{X})$ is a RKHS with a kernel satisfying assumptions \ref{assump:separable} to \ref{assump:bounded} for $\mu$. Therefore it is compactly embedded into $L_2(E_X)$ and the singular value decomposition of the inclusion $\bar{W}^{s\theta,2}(E_{X}) \hookrightarrow L_2(E_X)$ leads to a characterization of $\bar{W}^{s\theta,2}(E_{X})$ as 
    $$
    \bar{W}^{s\theta,2}(E_{X}) = \left\{\sum_{i \in I} a_{i} \sqrt{\la_{i}}f_{i}:\left(a_{i}\right)_{i \in I} \in \ell_{2}(I)\right\}
    $$
    with $\la_i>0$ for all $i \in I$ and $\{\sqrt{\la_{i}}f_i\}_{i \in I}$ forming an ONB in $\bar{W}^{s\theta,2}(E_{X})$ \citep[Lemma 2.6]{steinwart2012mercer}. We therefore have 
    $$
    W^{s,2}(E_{X}) = [\bar{W}^{s\theta,2}(E_{X})]^{\theta^{-1}}_{\mu} = \left\{\sum_{i \in I} a_{i} \la_{i}^{\frac{1}{2\theta}}[f_{i}]_{\nu}:\left(a_{i}\right)_{i \in I} \in \ell_{2}(I)\right\}
    $$
    which further proves 
    $$
    \bar{W}^{s,2}(E_{X}) = \left\{\sum_{i \in I} a_{i} \la_{i}^{\frac{1}{2\theta}}f_{i}:\left(a_{i}\right)_{i \in I} \in \ell_{2}(I)\right\}.
    $$
    Since $s > d/2$, $\bar{W}^{s,2}(E_{X})$ is a RKHS with bounded kernel $k_s$ and by \citep[Lemma 2.6]{steinwart2012mercer}, we must have for all $x \in E_X$ 
    $$
    k_s(x,x) = \sum_{i \in I}\la_i^{\frac{1}{\theta}}f_i(x)^2 < +\infty.
    $$
    Therefore, using Proposition~\ref{prop:zhang} with $\cH_X = \bar{W}^{s\theta,2}(E_{X})$, $\beta = \theta^{-1}$ and $\alpha=\theta$, we get
    $$
    [\bar{W}^{s,2}(E_{X})]^{\theta}_{\mu} = [\bar{W}^{s\theta,2}(E_{X})]_{\mu}^{\theta^{-1}\cdot \theta} = [\bar{W}^{s\theta,2}(E_{X})]_{\mu}^{1} = W^{s\theta,2}(E_{X}).
    $$
\end{proof}

\begin{prop} \label{prop:zhang}
    Let $\pi$ be a probability measure on $E_X$ and $\cH_X$ be a RKHS on $E_X$ with kernel $k_X$ such that assumptions \ref{assump:separable} to \ref{assump:bounded} are satisfied. Let $(\mu_i)_{i\in I} > 0$ be a non-increasing sequence, and let $(e_i)_{i \in I} \in \mathcal{H}_{X}$ be a family such that $\left([e_i]\right)_{i \in I}$ is an orthonormal basis (ONB) of $\overline{\text{ran}~I_{\pi}} \subseteq L_2(\pi)$ and such that Eq.~(\ref{eq:SVD}) holds. We denote the integral operator by $L_{k_X}:= L_X$ to highlight the dependence on $k_X$. If $\beta > 0$ is such that 
    $$
    \sum_{i \in I}\mu_i^{\beta}e_i(x)^2 < +\infty, \qquad x \in E_X,
    $$
    then 
    \[\mathcal{H}_{X}^{\be}:=\left\{\sum_{i \in I} a_{i} \mu_{i}^{\beta / 2}e_{i}:\left(a_{i}\right)_{i \in I} \in \ell_{2}(I)\right\} \]
    with the norm 
    $$
    \left\|\sum_{i \in I} a_{i} \mu_{i}^{\be / 2}\left[e_{i}\right]\right\|_{\mathcal{H}_{X}^{\be}}:=\left\|\left(a_{i}\right)_{i \in I}\right\|_{\ell_{2}(I)}=\left(\sum_{i \in I} a_{i}^{2}\right)^{1 / 2}
    $$
    is a RKHS compactly embedded into $L_2(\pi)$ and its kernel $k_X^{\beta}$ is given by
    $$
    k_X^{\be}(x,z) = \sum_{i \in I}\mu_i^{\beta}e_i(x)e_i(z), \qquad x,z \in E_X.
    $$
    Furthermore, we have $L_{k_X}^{\beta} = L_{k_X^{\be}}$ which implies that for all $\alpha > 0$,
    $$
    [\mathcal{H}_{X}^{\be}]^{\alpha}_{\pi} = [\mathcal{H}_{X}]^{\be \cdot \al}_{\pi}.
    $$
\end{prop}

\begin{proof}
    For the proof that $\mathcal{H}_{X}^{\be}$ is a RKHS with kernel $k_X^{\be}$, see \citep[Definition 4.1 and Proposition 4.2]{steinwart2012mercer} and for the proof that $L_{k_X}^{\beta} = L_{k_X^{\be}}$ see \citep[Lemma 4.3]{steinwart2012mercer}. For the last point recall that interpolation spaces are defined such that $[\mathcal{H}_{X}^{\be}]^{\alpha}_{\pi} = \operatorname{ran} L_{k_X^{\be}}^{\alpha / 2}$, and since $L_{k_X}^{\beta} = L_{k_X^{\be}}$ we have
    $$
    [\mathcal{H}_{X}^{\be}]^{\alpha}_{\pi} = \operatorname{ran} L_{k_X^{\be}}^{\alpha / 2} = L_{k_X}^{(\be \cdot \al)/2} = [\mathcal{H}_{X}]^{\be \cdot \al}_{\pi}.
    $$
\end{proof}

\begin{proof}[Proof of Proposition~\ref{cor:inter_sobolev}] Using Definition~\ref{def:frac_sobolev} combined with Theorem~\ref{th:vec_sobolev} and Theorem~\ref{th:tensor_interpolation}, we have for $ r > 0$ and $m:=\min \{s \in \mathbb{N}: s>r\}$, 
\begin{align*}
W^{r,2}(E_X; \cY) &=\left[L_2(E_X;\cY), W^{m,2}(E_X; \cY)\right]_{r / m, 2} \\ &\cong \left[\cY \otimes L_2(E_X), \cY \otimes W^{m,2}(E_X)\right]_{r / m, 2} \\ &\cong \cY \otimes \left[L_2(E_X), W^{m,2}(E_X)\right]_{r / m, 2} \\ &= \cY \otimes W^{r,2}(E_X).
\end{align*}
Then, by Proposition~\ref{prop:sobolev_beta}, if $k_X$ is a kernel on $E_X$ such that $\cH_X = \bar{W}^{m, 2}(E_X)$ with $m > d/2$, then for all $r  \geq 0$,
$$
[\cG]^{r/m} \simeq \cY \otimes [\cH]_X^{r/m} \simeq \cY \otimes W^{r,2}(E_X) \cong W^{r,2}(E_X; \cY).
$$
\end{proof}

\section{Auxiliary Results}\label{sec:auxiliary}
The following lemma is from Lemma 11 and 13 \cite{fischer2020sobolev}.

\begin{lma}\label{theo:h_bound}
Under \eqref{asst:emb}, we have
$$
\left\|\left(C_{X X}+\lambda Id_{\mathcal{H}_X} \right)^{-\frac{1}{2}} k_X(X, \cdot)\right\|_{\mathcal{H}_X} \leq A\lambda^{-\frac{\alpha}{2}}.
$$
Under \eqref{asst:evd}, there exists a constant $D>0$ such that
$$
\mathcal{N}(\lambda)=\text{Tr}\left(C_{XX}\left(C_{XX}+\lambda Id_{\mathcal{H}_X}\right)^{-1}\right) \leq D \lambda^{-p}.
$$
\end{lma}

The following Theorem is from \citet[Theorem $26$]{fischer2020sobolev}.
\begin{theo}[Bernstein's inequality]\label{theo:ope_con_steinwart}
. Let $(\Omega, \mathcal{B}, P)$ be a probability space, $H$ be a separable Hilbert space, and $\xi: \Omega \rightarrow H$ be a random variable with
$$
\mathbb{E}[\|\xi\|_{H}^{m}] \leq \frac{1}{2} m ! \sigma^{2} L^{m-2}
$$
for all $m \geq 2$. Then, for $\tau \geq 1$ and $n \geq 1$, the following concentration inequality is satisfied
$$
P^{n}\left(\left(\omega_{1}, \ldots, \omega_{n}\right) \in \Omega^{n}:\left\|\frac{1}{n} \sum_{i=1}^{n} \xi\left(\omega_{i}\right)-\mathbb{E}_{P} \xi\right\|_{H}^{2} \geq 32 \frac{\tau^{2}}{n}\left(\sigma^{2}+\frac{L^{2}}{n}\right)\right) \leq 2 e^{-\tau}.
$$
\end{theo}

\begin{defn}[Gaussian Measures on separable Hilbert spaces \citealp{bogachev1998gaussian}] \label{def:inf_gau}
Let $H$ be a separable Hilbert space, $\mathcal{F}_H$ be the Borel $\sigma$-algebra defined on $H$ and let $H^*$, the topological dual space, be identified with $H$ by means of the Riesz representation. A probability measure $\gamma$ on $(H,\mathcal{F}_H)$ is said to be Gaussian if $\gamma \circ f^{-1}$ is a Gaussian measure in $\mathbb{R}$ for every $f \in H^*$.
For the Gaussian measure $\gamma$ and any $f \in H^*$, we define the mean and covariance as 
\begin{IEEEeqnarray*}{rCl}
\mu_{\gamma}(f) &:=& \int_H f(x)\gamma(dx)\\
Q_{\gamma}(f,g) &:=& \int_H \left[f(x) - \mu_{\gamma}(f)\right]\left[g(x) - \mu_{\gamma}(g)\right]\gamma(dx).
\end{IEEEeqnarray*}
By the Riesz representation theorem, for any $f \in H^*$ there is a unique $v_f \in H$ such that $f(x) = \langle x, v_f \rangle_{H}$ for all $x \in H$. It is then straightforward to see that $\mu_{\gamma}(f) = \langle \mu, v_f \rangle_H$ for all $f \in H^*$ where $\mu := \int_H x\gamma(dx)$ and $Q_{\gamma}(f,g) = \langle Qv_f, v_g \rangle_H$ for all $f,g \in H^*$ where $Q := \int_H \left(x - \mu \right) \otimes \left(x - \mu \right) \gamma(dx).$ This justifies the notation $\mathcal{N}_{H}(\mu,Q)$ for the Gaussian measure $\gamma$ and we we write the variance as $\mathbb{V}(\gamma) := Q$.
\end{defn}


\end{document}